\documentclass{article}
\PassOptionsToPackage{numbers, compress}{natbib}
\usepackage{microtype}
\usepackage{graphicx}
\usepackage{subcaption}
\usepackage{booktabs} 

\PassOptionsToPackage{numbers,sort&compress}{natbib}
\usepackage[main, final]{neurips_2026}

\usepackage{xcolor}         
\definecolor{mydarkred}{rgb}{0.6,0,0}
\definecolor{mydarkgreen}{rgb}{0,0.6,0}
\usepackage[colorlinks,
linkcolor=mydarkred,
citecolor=mydarkgreen]{hyperref}

\usepackage{amsmath}
\usepackage{amssymb}
\usepackage{mathtools}
\usepackage{amsthm}

\usepackage{multirow, bm}
\usepackage{float}
\usepackage{make cell}
\usepackage[capitalize,noabbrev]{cleveref}

\usepackage[table]{xcolor}
\definecolor{gr}{HTML}{00AA00}
\definecolor{go}{HTML}{FFD700}   
\definecolor{or}{HTML}{FF8C00} 

\colorlet{grL}{gr!20}
\colorlet{goL}{go!30}
\colorlet{orL}{or!30}

\usepackage[capitalize,noabbrev]{cleveref}

\theoremstyle{plain}
\newtheorem{theorem}{Theorem}

\DeclareMathOperator*{\argmax}{argmax}

\newtheorem{lemma}[theorem]{Lemma}

\newtheorem{corollary}[theorem]{Corollary}
\newtheorem{definition}[theorem]{Definition}

\newtheorem{remark}{Remark}
\newtheorem{factor}{Factor}
\usepackage{bbding}
\usepackage{bm}
\usepackage{epstopdf}
\usepackage[font=small,labelfont=bf]{caption}

\def \x {\bm{x}}

\def \y {\bm{y}}

\def \z {\bm{z}}

\def \H {\bm{H}}

\def \mmu {\bm{\mu}}
\def \mdelta {\bm{\delta}}
\def \mzeta {\bm{\zeta}}
\def \mpi {\bm{\pi}}
\def \mPi {\bm{\Pi}}
\def \bR {\mathbb{R}}
\def \bP {\mathbb{P}}

\def \bQ {\mathbb{Q}}
\def \bI {\mathbb{I}}

\def \bW {\mathbb{W}}

\def \cF {\mathcal{F}}

\def \cS {\mathcal{S}}
\def \cT {\mathcal{T}}

\def \cH {\mathcal{H}}

\def \cX {\mathcal{X}}
\def \cN {\mathcal{N}}

\usepackage[textsize=tiny]{todonotes}
\usepackage{silence}
\title{USAD: Uncertainty-aware Statistical Adversarial Detection}

\author{
Zhijian Zhou\thanks{Equal contribution. $^\dag$Corresponding authors.}\qquad Xunye Tian$^*$\qquad Jiacheng Zhang$^*$\qquad Zesheng Ye$^\dag$ \\
\textbf{Yiyi Guo \qquad Donghao Zhang \qquad Liuhua Peng \qquad Feng Liu$^\dag$} \\
The University of Melbourne \\
\texttt{fengliu.ml@gmail.com} \qquad \texttt{zesheng.ye@unimelb.edu.au}
}

\begin{document}
\maketitle
\addtocontents{toc}{\protect\setcounter{tocdepth}{-1}}
\raggedbottom

\begin{abstract}
\emph{Statistical adversarial detection} (SAD) treats adversarial detection as a two-sample test. Given a reference set of clean examples (CEs) and a batch of queries, potentially containing an unknown mixture of CEs and adversarial examples (AEs), SAD decides whether the query distribution drifts away from the CE distribution while controlling the false-alarm rate.
Existing SAD-based methods mainly use \emph{maximum mean discrepancy} (MMD) to measure the distributional discrepancy.
However, MMD's distributional properties limit its ability to capture characteristic uncertainty patterns of AEs that are crucial for detection: AEs typically exhibit abnormal feature spread (i.e., global uncertainty) and instability under perturbations (i.e., local uncertainty).
To close the gap, we propose \emph{Uncertainty-aware Statistical Adversarial Detection} (USAD), which explicitly captures these uncertainty patterns with two new statistics: (1) \underline{\textbf{V}ariance \textbf{D}iscrepancy} (VD), which measures the difference in feature spread between AEs and CEs to capture global uncertainty differences, and (2) \underline{\textbf{P}erturbation-based \textbf{C}ovariance \textbf{D}iscrepancy} (PCD), which compares feature covariance under Gaussian perturbations to capture local uncertainty differences.
By aggregating VD and PCD, USAD achieves superior detection performances over baseline methods against various adversarial attacks, highlighting the importance of considering characteristic behaviors of AEs for effective SAD.
Our code is available at: \url{https://github.com/yeager20001118/USAD}.
\end{abstract}    
\section{Introduction}\label{sec:intro}
Defending against \emph{adversarial examples} (AEs) remains a long-standing challenge in deep learning \citep{szegedy2014intriguing, goodfellow2015explaining, madry2018towards, carlini2017towards, zhang2019theoretically, croce2020reliable, Gao:Liu:Zhang:Han:Liu:Niu:Sugiyama2021, cao2021invisible, jing2021too, zhang2024improving, sun2025sample, bo2025trustworthy}.
The most lightweight defense against adversarial attacks is adversarial detection, which identifies AEs in input data before they reach target~systems~\citep{ma2018characterizing, deng2021libre, zhang2023detecting, Gao:Liu:Zhang:Han:Liu:Niu:Sugiyama2021}.
Existing methods often consider adversarial detection as a binary classification task \citep{ma2018characterizing, deng2021libre, zhang2023detecting}. They typically train a binary detector tailored to specific classifiers or attack types, which limits the effectiveness against unknown attacks \citep{tramer2022detecting, bryniarski2022evading}.
Although a study leverages diffusion models~\citep{song2019generative, song2021score, huang2021a} to remove such dependency \citep{zhang2023detecting}, it is vulnerable to diffusion-based adaptive attacks~\citep{xue2023diffusion, kang2023diffattack}.
These limitations motivate us to pursue a model- and attack-agnostic detection framework.

\emph{Statistical adversarial detection} (SAD) offers a complementary direction by formulating adversarial detection as a two-sample hypothesis test: given a reference pool of CEs and a set of queries, SAD asks whether the query {\em distribution} has drifted away from the clean one (see Section~\ref{sec:sadd_revisit} for formal setup).
Concretely, the defender needs to maintain a reference set $X$ of CEs and, for a window $Y$ of incoming queries, compute a test statistic $\cT(X, Y)$ that measures their distributional discrepancy~\citep{Gao:Liu:Zhang:Han:Liu:Niu:Sugiyama2021}.
A permutation test then determines a threshold $t_\alpha$ such that $\Pr (\cT(X, Y) > t_\alpha \mid Y \text{ is clean} ) \leq \alpha$, providing guarantees on statistically-controlled Type-I error, which is often desired in high-stakes contexts.
Operationally, this corresponds to maintaining a queue of recent queries and running tests whenever a window has accumulated, or in a sliding-window over the stream; each such window may contain an {\em unknown mixture} of CEs and AEs~\citep{zhang2025one}.
This way SAD acts as a statistical monitoring layer: once a query set is flagged as suspicious, the defending system can trigger other defenses or route the traffic for human inspection, all while keeping the rate of spurious alarms under control~\citep{Gao:Liu:Zhang:Han:Liu:Niu:Sugiyama2021}.
See Appendix~\ref{app:discussion} for further discussions on the practicality of SAD and Appendix~\ref{app:related} for related works.

\begin{figure*}[t]
    \centering
    \begin{subfigure}{0.495\linewidth}
        \centering
        \includegraphics[width=\textwidth]{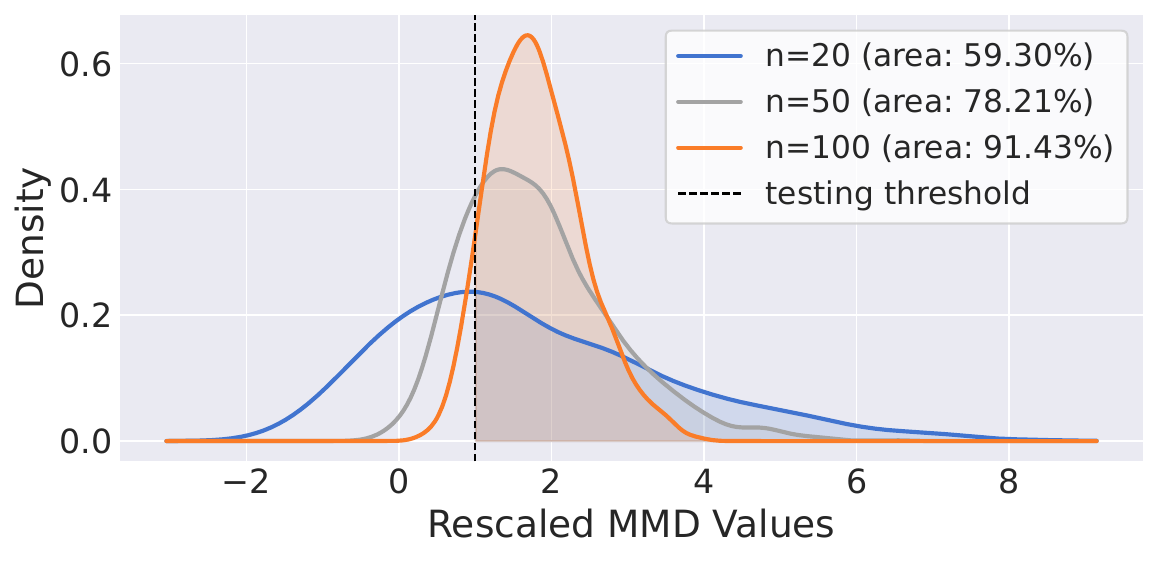}
        \caption{Distribution of MMD values for AE batch sizes.}
        \label{fig:mmd_test_power}
    \end{subfigure}
    \hfill
    \begin{subfigure}{0.495\linewidth}
        \centering
        \includegraphics[width=\textwidth]{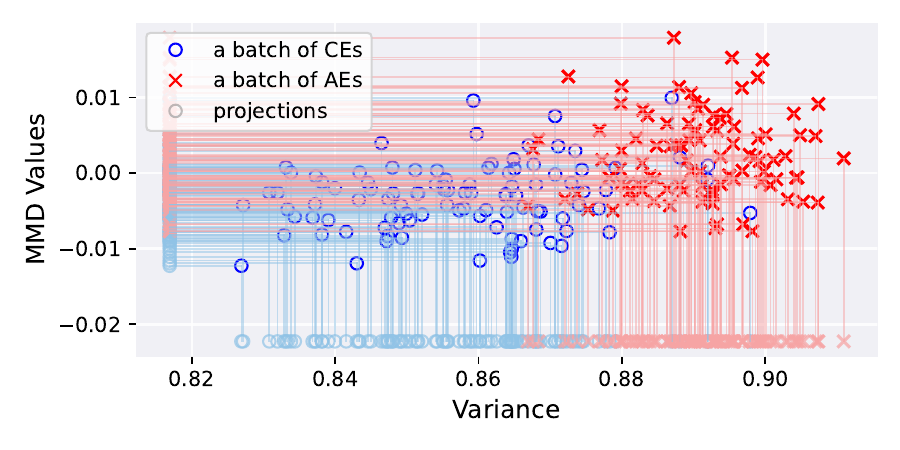}
        \caption{MMD–variance relationship for CE and AE batches.}
        \label{fig:mmd_variance}
    \end{subfigure}
    \caption{Statistical characteristics of MMD-based SAD using semantic features \citep{Gao:Liu:Zhang:Han:Liu:Niu:Sugiyama2021}.
    \textbf{(a):} Kernel density estimates of MMD values for three AE batch sizes ($n=\{20,50,100\}$). For visualization, the MMD values of each batch size are rescaled by its own permutation threshold, so that the three batch sizes share a common threshold of $1$, denoted by the vertical dashed line. The shaded regions to the right of the threshold indicate detection power, quantified by the area percentages shown in the legend (91.43\% for $n=100$, 78.21\% for $n=50$, and 59.30\% for $n=20$). It clearly illustrates that MMD-based SAD exhibits diminishing detection power as the batch size decreases. 
    \textbf{(b):} Sensitivity of MMD to uncertainty differences between AEs (red crosses) and CEs (blue circles), with projections onto the x- and y-axes for comparison, where the variance on the x-axis refers to the empirical RKHS variance of semantic features computed with the kernel (i.e., $\widehat{V}(\cdot;\ell)$ in Section~\ref{sec:vd}). Despite the markedly higher variance of AEs and their clear separation from CEs, the MMD values exhibit substantial overlap along the y-axis, indicating the insensitivity of MMD-based SAD to variance differences.}
    \vspace{-1.1em}
    \label{fig: motivation}
\end{figure*}

Current SAD methods mainly use \emph{maximum mean discrepancy} (MMD) \citep{Gretton:Borgwardt:Rasch:Scholkopf:Smola2012} as the test statistic $\cT(X, Y)$ \citep{Gao:Liu:Zhang:Han:Liu:Niu:Sugiyama2021, zhang2025one}.
They have shown strong performance against {\em unknown} and even {\em adaptive} attacks, while inheriting rigorous false-alarm control.
However, existing MMD-based SAD methods are {\em sample-inefficient}: 
they often require query windows with a sufficiently large effective sample size (and a non-trivial fraction of AEs therein) for the {\em test statistic} to reliably identify the distributional discrepancy with high test power.
As Figure \ref{fig:mmd_test_power} shows, we confirm that the test power of MMD-based SAD degrades rapidly as the window size of queries shrinks.
This limits practical deployment, as waiting for large query sets is often unrealistic and attack volume is inherently uncontrollable.


To understand this limitation, we examine which distributional properties the {\em test statistic} captures.
Since MMD captures mean-embedding shifts (see Eq.~\eqref{eq: mmd}), we hypothesize that it is insensitive to geometric properties, such as feature spread~\citep{ma2018characterizing} and local instability~\citep{huang2019model,roth2019odds,zhang2023detecting}, that distinguish AEs from CEs.
These geometric cues are closely related to \emph{uncertainty}: AEs often reside near decision boundaries and exhibit higher feature variance and predictive uncertainty than CEs \citep{Feinman:Curtin:Shintre:Gardner2017}.
To validate this hypothesis, we test whether MMD captures AE--CE uncertainty differences, using \emph{variance} as the uncertainty measure.
As shown in Figure~\ref{fig:mmd_variance}, CE and AE MMD values largely overlap despite substantial variance differences, indicating that MMD is insensitive to this uncertainty measurement.
This finding reveals a fundamental mismatch: \emph{MMD is not tailored for SAD, and therefore, MMD-based SAD methods are insensitive to distinguish AEs from CEs when AEs display their characteristic uncertainty patterns rather than mean embedding shifts.}
This motivates new \emph{test statistics} directly sensitive to AE uncertainty patterns.

To close the gap, this paper proposes \emph{\textbf{U}ncertainty-aware \textbf{S}tatistical \textbf{A}dversarial \textbf{D}etection} (USAD), which explicitly accounts for two well-established uncertainty-induced patterns, namely
(1) AEs are displaced off the clean data manifold, causing abnormal spread in their semantic feature distributions \citep{Feinman:Curtin:Shintre:Gardner2017, ma2018characterizing} (i.e., \emph{global uncertainty}), and
(2) norm-bounded AEs are generated by small perturbations that cross local decision boundaries~\cite{Moosavi:Fawzi:Frossard2016,fawzi2016robustness} and often exploit brittle non-robust features~\cite{Ilyas:Santurkar:Tsipras:Engstrom:Tran:Madry2019}, leading to atypical local response patterns under added perturbations~\cite{huang2019model,roth2019odds}  (i.e., \emph{local uncertainty}), with two new {\em test statistics}:
(1) \underline{\emph{\textbf{V}ariance \textbf{D}iscrepancy}} (VD), which measures the difference in feature spread between AEs and CEs (see Section \ref{sec:vd}), and (2) \underline{\emph{\textbf{P}erturbation-based \textbf{C}ovariance \textbf{D}iscrepancy}} (PCD), which compares feature covariance under Gaussian perturbations to capture local stability differences (see Section \ref{sec:pcd}).
VD is designed to capture \emph{global uncertainty differences} arising from manifold displacement, whereas PCD seeks to capture \emph{local uncertainty differences} induced by perturbation sensitivity.
To capture both global and local uncertainty differences, we follow~\citet{zhou2025dual} and aggregate VD and PCD while accounting for their mutual dependencies (see Section \ref{sec: dual}).
The false alarm rate of USAD is theoretically controlled (see Section~\ref{sec: theo}).

Section \ref{sec: experiment} evaluates the effectiveness of USAD on benchmark image datasets such as CIFAR-10 \citep{cifar} and ImageNet-1K \citep{deng2009imagenet}. 
Specifically, USAD \emph{consistently} outperforms baseline SAD methods by a notable margin against various unseen adversarial attacks under different threat models, including different norms (e.g., $\ell_\infty$ and $\ell_2$), perturbation budgets (e.g., $\epsilon$ from $8/255$ to $1/255$), batch sizes (e.g., from 80 to 10), and ratio of AEs in the query batch (e.g., from $100 \%$ to $20\%$).
Notably, even with a query batch size as small as $|Y|=10$, our method achieves test power essentially equal to 1 at $\epsilon = 4/255$, outperforming baselines by at least 24.8\%. This marks a substantial step toward overcoming the sample-size limitations that constrain detection power in hypothesis-testing-based methods. 
More importantly, our method remains robust against a well-designed adaptive attack.

Our contributions are: (1) We identify that MMD-based SADs fail to capture the distributional uncertainty differences crucial for distinguishing AEs from CEs; (2) To address this, we propose \emph{\textbf{U}ncertainty-aware \textbf{S}tatistical \textbf{A}dversarial \textbf{D}etection} (USAD) with new test statistics, i.e., \emph{\textbf{V}ariance \textbf{D}iscrepancy} (VD) and \emph{\textbf{P}erturbation-based \textbf{C}ovariance \textbf{D}iscrepancy} (PCD), to explicitly capture \emph{global} and \emph{local} uncertainty differences between AEs and CEs;
(3) Theoretically, we prove that USAD guarantees valid false-alarm rate control. Empirically, we show that USAD consistently outperforms strong baseline methods across diverse unseen, transfer, and adaptive attacks, under small query window sizes and when AEs are sparsely {\em mixed} with CEs within each query batch.
\section{Statistical Adversarial Detection}\label{sec:pre}
We begin by formalizing {\em statistical adversarial detection}~(SAD) as a hypothesis testing task for distinguishing {\em adversarial examples}~(AEs) from {\em clean examples}~(CEs) by measuring their distributional discrepancy (\cref{sec:sadd_revisit}).
We then identify limitations of existing SAD practices and motivate the need for new measures (\cref{sec:motivation}).
\subsection{Detecting AEs from a Statistical Perspective}\label{sec:sadd_revisit}
\noindent\textbf{Adversarial Attacks.} The algorithms for generating AEs are commonly referred to as adversarial attacks \citep{szegedy2014intriguing, goodfellow2015explaining}. 
Given a well-trained classifier $f: \mathcal{X} \to \mathcal{C}$ on a dataset $D = \{(\x_i, c_i)\}^N_{i=1}$ with $\x_i$ being a sample from the input space $\cX\subseteq\bR^d$ and $c_i$ being its ground-truth label defined in a label space $\mathcal{C}$, an AE $\y_i$ regarding $\x_i$ with perturbation $\mzeta_i$ can be generated by solving the optimization:
\vskip -0.2in
\begin{equation}
    \y_i = \x_i +\argmax_{\mzeta_i}\mathcal{L}(f(\x_i + \mzeta_i), c_i), \text{~s.t.~} ||\mzeta_i||_p \leq \epsilon,
\label{eq: ae}
\end{equation}
\vskip -0.08in
where $\epsilon$ is the perturbation budget, $||\cdot||_p$ is the $\ell_p$-norm (e.g., $\ell_\infty$ or $\ell_2$) and $\mathcal{L}$ is an objective function.

Eq.~\eqref{eq: ae} can be solved by many methods to generate AEs. For example, the \emph{fast gradient sign method} (FGSM) is a one-step attack that perturbs clean data in the direction of the loss gradient \citep{goodfellow2015explaining}.
Building on FGSM, the \emph{basic iterative method} (BIM) applies iterative perturbations along the gradient direction, updating the input at each step to generate stronger AEs \citep{kurakin2017adversarial}.
\emph{Projected gradient descent} (PGD) further enhances BIM by introducing random initialization prior to iterative gradient-based updates \citep{madry2018towards}.
Beyond non-targeted attacks, the \emph{Carlini \& Wagner} (C\&W) attack generates targeted AEs by optimizing a carefully designed objective function that explicitly encourages misclassification toward a chosen label \citep{carlini2017towards}.
\emph{AutoAttack} (AA) \citep{croce2020reliable} aggregates multiple attacks, which include three non-targeted white-box attacks \citep{croce2020reliable, croce2020minimally} and one targeted black-box attack \citep{andriushchenko2020square}, which makes AA widely adopted for evaluating adversarial robustness.

\smallskip
\noindent
\textbf{SAD Problem Setup.}
Given a reference set of CEs $X = \{\x_i \}_{i=1}^{n} \sim \bP$ and a query set $Y = \{\y_j \}_{j=1}^{m}$ drawn from some distribution $\bQ$, we aim to decide whether the queries $Y$ follow the same distribution as CEs.
Under the null hypothesis $\H_0$, {\em all} queries are clean and $\bP = \bQ$.
Under the alternative $\H_1$, at least a {\em non-zero fraction} of queries in $Y$ are AEs generated by {\em unseen} adversarial attacks against $f$.
SAD therefore tests
\begin{equation}\label{eq:sadd_as_ht}
    \H_0: \bP = \bQ\ \ \ \text{and}\ \ \ \H_1: \bP \neq \bQ,
\end{equation}
and reject $\H_0$ at significance level $\alpha$ when the chosen \underline{{\em test statistic}} exceeds a threshold, providing Type-I control. Here, Type-I error means falsely declaring a clean query set as adversarial.
Note that Eq.~\eqref{eq:sadd_as_ht} tests distributional equality rather than adversariality: $\H_0$ is rejected whenever $\bQ$ differs from $\bP$ in the distributional properties captured by the chosen test statistic, so benign shifts that alter the same properties may also be detected (see Appendix~\ref{app:discussion} and Appendix~\ref{app:benign_shift}).

\smallskip
\noindent
\textbf{SAD as Hypothesis Test.}
The SAD starts from designing a \underline{{\em test statistic}} $\cT(X, Y)$ to measure the discrepancy between $\bP$ and $\bQ$ (by observing $X$ and $Y$).
At a chosen level $\alpha$, $\cT(X,Y)$ is then compared against a \underline{{\em testing threshold}} $t_\alpha$ defined as the $(1 - \alpha)$-quantile of the {\em null distribution}, which characterizes the sampling behavior of $\cT(X, Y)$ under $\H_0$, satisfying $\Pr(\cT(X, Y) > t_\alpha | \H_0) = \alpha$.
It means that $\H_0$ will be rejected (i.e., declare $Y$ adversarial) if $\cT(X, Y) > t_\alpha$.

\smallskip
As the null distribution is rarely known analytically, we typically estimate $t_\alpha$ using a \underline{{\em permutation test}}, which builds an empirical null distribution by randomly re-partitioning the pooled examples $X \cup Y$ into new sets $(X_{(r)}, Y_{(r)})$, and computing $\cT(X_{(r)}, Y_{(r)})$ for each permutation $r = 1, \dots, R$.
The empirical threshold $t_\alpha$ is then set to the $(1-\alpha)$-quantile of these permuted statistic values, ensuring the probability of falsely rejecting $\H_0$ is controlled at the chosen level $\alpha$ (see Appendix~\ref{app:perm} for details).

\smallskip
\noindent
\textbf{MMD as Test Statistic.}
Critically, the effectiveness of SAD relies on whether the chosen test statistic $\cT(X, Y)$ can capture the distributional discrepancy if $X$ indeed comes from a distribution different from $Y$ (i.e., $\bP \neq \bQ$).
Existing SAD practices~\citep{Grosse:Manoharan:Papernot:Backes:McDaniel2017,Carlini:Wagner2017,Gao:Liu:Zhang:Han:Liu:Niu:Sugiyama2021,zhang2025one} have predominantly employed the MMD~\citep{Gretton:Borgwardt:Rasch:Scholkopf:Smola2012}, which measures the distance between $\bP$ and $\bQ$ by mapping their elements to a \emph{reproducing kernel Hilbert space} (RKHS) and comparing the {\em kernel mean embeddings}~\citep{Muandet:Fukumizu:Sriperumbudur:Scholkopf2017} therein.
Formally, for a characteristic kernel $\kappa: \mathcal{X} \times \mathcal{X} \to \mathbb{R}$ with RKHS $\cH_\kappa$ and its corresponding map $\kappa(\cdot,\x)\in\cH_\kappa$, we have the kernel mean embeddings of $\bP$ and $\bQ$ given by $\mmu_\bP=E_{\x\sim\bP}[\kappa(\cdot,\x)]$ and $\mmu_\bQ=E_{\y\sim\bQ}[\kappa(\cdot,\y)]$, respectively.
The MMD-based SAD~\citep{Gao:Liu:Zhang:Han:Liu:Niu:Sugiyama2021, zhang2025one} then instantiates Eqn.\eqref{eq:sadd_as_ht} on testing the equality of such embeddings, namely $\H_0^{\text M}: \mmu_\bP = \mmu_\bQ$ versus $\H_1^{\text M}: \mmu_\bP \neq \mmu_\bQ$ via squared distance between $\mmu_\bP$ and $\mmu_\bQ$:
\begin{eqnarray}\label{eq: mmd}
\textnormal{MMD}^2(\bP,\bQ,\kappa)
&\triangleq& \|\mmu_\bP-\mmu_\bQ\|_{\cH_\kappa}^2\\
&=&E[\kappa(\x,\x')+\kappa(\y,\y')-2\kappa(\x,\y)],\nonumber
\end{eqnarray}
where $\x, \x' \stackrel{\text{i.i.d.}}{\sim} \bP$, $\y, \y' \stackrel{\text{i.i.d.}}{\sim} \bQ$, and $\|\cdot\|_{\cH_\kappa}^2 = \langle \cdot, \cdot \rangle_{\cH_\kappa}$ defines the inner product in RKHS $\cH_\kappa$.
In practice, we empirically estimate $\textnormal{MMD}^2(\bP,\bQ,\kappa)$ with samples $X$ and $Y$, using unbiased U-statistics~\cite{Gretton:Borgwardt:Rasch:Scholkopf:Smola2012} (Appendix~\ref{app:mmd}).
Recently, SAMMD~\cite{Gao:Liu:Zhang:Han:Liu:Niu:Sugiyama2021} showed that applying MMD to semantic features from the target classifier $f$, combined with proper {\em testing threshold}, can elevate the detection power of MMD-based SAD. 

\subsection{SAD Needs Uncertainty-Aware Statistics}\label{sec:motivation}
\textbf{When MMD-based SAD Fails.}
Despite maintaining Type-I error control, MMD-based SAD methods exhibit diminishing power as query size $|Y|$ decreases (Fig.~\ref{fig:mmd_test_power}).
Such small-sample weakness is problematic for SAD: practical defense requires early detection when only a few query examples have been observed.
It is often unrealistic to wait for large query sets to accumulate~\cite{zhang2025one}.
Moreover, this degradation persists even with carefully tuned kernel embeddings and thresholds~\cite{Gao:Liu:Zhang:Han:Liu:Niu:Sugiyama2021}, implying that the bottleneck is not merely kernel choice or threshold but in which distributional properties the {\em test statistic} itself targets.

MMD (and its variant SAMMD), by construction, is most sensitive to mean shifts $\|\mmu_\bP-\mmu_\bQ\|_{\cH_\kappa}^2$, capturing where distributions are centered (i.e., first-order moments) in the RKHS $\cH_\kappa$, yet providing less information about their internal structures, such as spread and local geometry~\citep{Harchaoui:Bach:Moulines2008,Zhou:Ni:Yao:Gao2023}.
These second-order properties, however, are precisely where adversarial attacks leave strong artifacts that can be valuable in distinguishing AEs.

\smallskip
\noindent
\textbf{Adversarial Artifacts as Distributional Uncertainty.}
AEs result from solving \eqref{eq: ae}, where adversarial perturbations are crafted to enforce AEs to cross target classifier $f$'s decision boundaries.
This exploits two geometric properties, leading to two well-established characteristic behaviors of AEs:
\begin{factor}[Global Uncertainty from Manifold Displacement]\label{factor:global_uncertainty}
    AEs are systematically pushed off clean data manifold~\cite{Feinman:Curtin:Shintre:Gardner2017}.
    This displacement alters the spread of AEs' semantic feature distribution~\cite{ma2018characterizing}. 
    Consequently, AEs may become more or less dispersed than CEs, without necessarily shifting the mean embedding, reflecting a shift in the \underline{global uncertainty} over the distribution.
\end{factor}
\begin{factor}[Local Uncertainty from Perturbation Sensitivity]\label{factor:local_uncertainty}
    AEs tend to reside near classifier $f$'s decision boundaries and exploit non-robust features~\cite{Moosavi:Fawzi:Frossard2016, Ilyas:Santurkar:Tsipras:Engstrom:Tran:Madry2019}, making them sensitive to local isotropic perturbations.
    Motivated by prior random-perturbation-based detection studies~\cite{huang2019model,roth2019odds,zhang2023detecting} and the connection between adversarial robustness and stability under additive Gaussian noise~\cite{Li:Chen:Wang:Carin2019}, we examine differences in the feature-response patterns of AEs and CEs under Gaussian perturbations through perturbation-induced covariance, where AEs are expected to exhibit less stable feature responses than CEs.
    This instability marks elevated \underline{local uncertainty} at individual examples.
\end{factor}
Appendix~\ref{app:factor_analysis} summarizes the arguments from prior work for these two factors and analyzes them directly on synthetic data and real AEs.

\smallskip
\noindent
\textbf{Bridging the Gap.}
These findings reveal a fundamental mismatch: 
existing MMD-based SAD practices compare first-order distributional differences (i.e., mean shifts in the RKHS).
Thus, they may struggle to distinguish $\mathbb{P}$ (i.e., clean distribution) from $\mathbb{Q}$ (i.e., adversarial distribution) when adversarial artifacts appear as second-order structural changes (i.e., uncertainty pattern shifts) without substantially moving mean embeddings of AEs.
This calls for new discrepancy measures that are directly sensitive to these changes in distributional uncertainty caused by adversarial attacks.
We introduce two such {\em test statistics} to close the gap.
\vspace{-0.5em}
\section{Uncertainty-Aware Test Statistics}
\vspace{-0.5em}
We now develop {\em test statistics} that target how adversarial perturbations affect uncertainty in the classifier $f$'s semantic feature space identified in \cref{sec:pre}.
For clarity, we slightly abuse notation and denote $f(\cdot) \in \mathcal{F} \subseteq \mathbb{R}^q$ as semantic features of an input $\cdot \in \mathcal{X}$ extracted by $f$, where $\mathcal{F}$ is a $q$-dimensional semantic feature space (often from the penultimate layer of $f$).
\cref{sec:vd} introduces {\em variance discrepancy}~(VD) that quantifies {\em global uncertainty} by comparing within-distribution variances, detecting altered concentration in query distribution $\bQ$ relative to clean distribution $\bP$.
\cref{sec:pcd} introduces {\em perturbation-based covariance discrepancy}~(PCD) that measures whether query examples (suspected to be AEs) exhibit higher {\em local uncertainty} to small, random isotropic noise compared to CEs, encoded by Gaussian-noise-induced covariance at each example.

\begin{figure*}[t]
    \begin{center}
    \centerline{\includegraphics[width=\columnwidth]{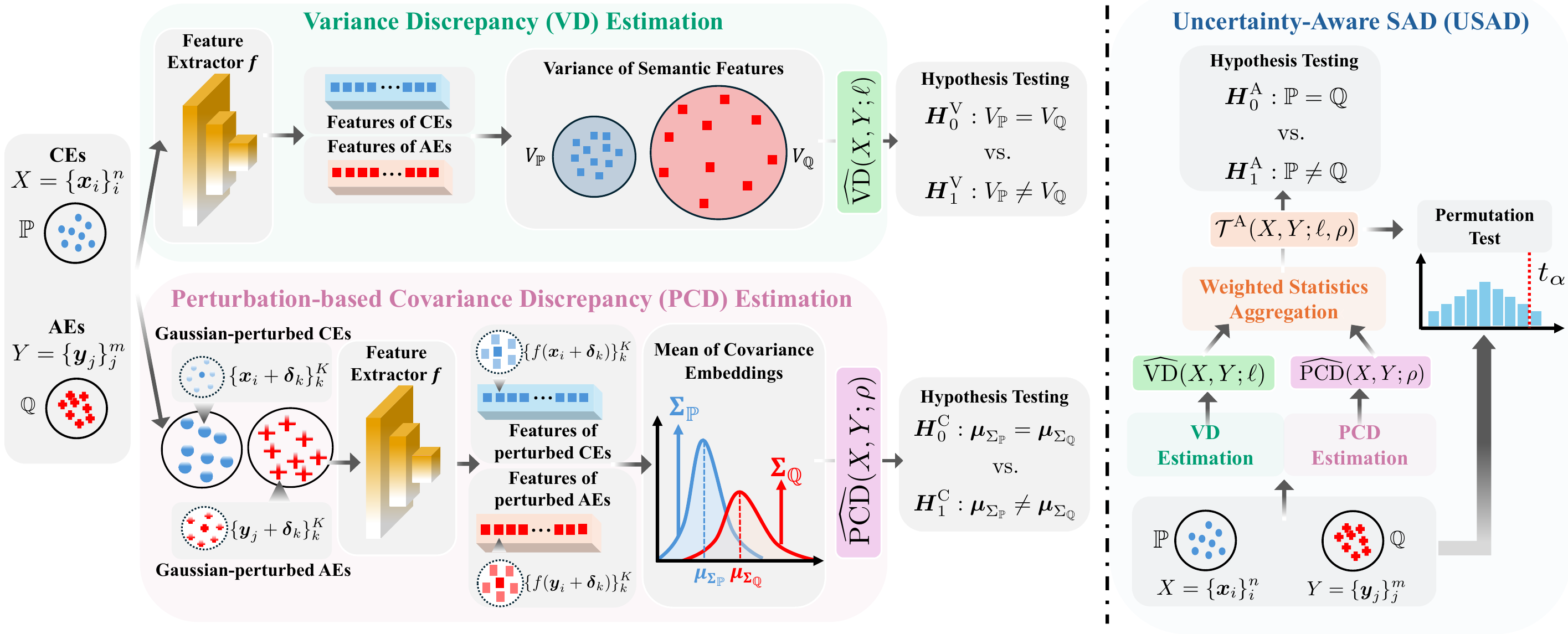}}
    \caption{\textit{\textbf{U}}ncertainty-aware \textit{\textbf{S}tatistical \textbf{A}dversarial \textbf{D}etection}~(USAD). The semantic features of CEs $X \sim \bP$ and queries $Y \sim \bQ$ (suspected AEs) are extracted by using the penultimate layer of classifier $f$. From these features, USAD estimates (i) {\em variance discrepancy}~(VD) measuring shifts in feature-spread between $X$ and $Y$ and (ii) {\em perturbation-based covariance discrepancy}~(PCD) comparing their covariance mean-embeddings under Gaussian perturbations.
    These VD and PCD statistics are adaptively aggregated via {\em correlation-aware weighting} into a single {\em test statistic}, and a {\em permutation test} determines, at level $\alpha$, whether $Y$ is adversarial.
    }
    \vspace{-3em}
    \end{center}
    \label{fig: pipeline}
\end{figure*}
\subsection{Variance Discrepancy}\label{sec:vd}
\textbf{Global Uncertainty as Variance.}
We quantify the variance of semantic features within-distribution to detect whether $Y$'s features spread differently from those of $X$ in the RKHS.
Let $\ell: \cF\times\cF\rightarrow\bR$ be a kernel with RKHS $\cH_\ell$ (which also defines map $\ell(\cdot,f(\x))\in\cH_\ell$); in our experiments, $\ell$ is a Gaussian kernel on semantic features (see Appendix~\ref{appendix:exp}).
For distributions $\bP$ and $\bQ$, the semantic feature variances are
\begin{equation*}
    \begin{aligned}
        V_\bP & \triangleq E_{\x\sim\bP}\left[\ell(f(\x),f(\x))\right]-E_{\x,\x'\sim\bP^2}[\ell(f(\x),f(\x'))], \\    
        V_\bQ & \triangleq E_{\y\sim\bQ}\left[\ell(f(\y),f(\y))\right]-E_{\y,\y'\sim\bQ^2}[\ell(f(\y),f(\y'))].
    \end{aligned}
\end{equation*}
The first term of $V_\bP$ establishes the self-similarity of examples within $\bP$, while the second captures their pairwise similarity.
Thus, small $V_\bP$ indicates tight clustering, while large $V_\bP$ reflects greater dispersion (likewise for $V_\bQ$).
VD is related to the \emph{maximum variance discrepancy}~(MVD)~\cite{makigusa2024}: VD compares the traces of the RKHS covariance operators of $\bP$ and $\bQ$, whereas MVD compares the full operators.
VD thus directly targets changes in the total feature spread, which correspond to the global uncertainty in~\cref{factor:global_uncertainty}, and requires estimating only a scalar quantity from limited samples; richer local covariance changes are instead captured by PCD (\cref{sec:pcd}), making VD and PCD complementary (see Appendix~\ref{app:mvd_pfd} for an empirical comparison with MVD).



\smallskip
\noindent
\textbf{Test.}
We formalize the comparison of variances $V_\bP$ and $V_\bQ$ by instantiating Eq.\eqref{eq:sadd_as_ht} as the following hypothesis test:
\begin{equation}\label{eq:hype_vd}
\H^{\textnormal V}_0: V_\bP=V_\bQ \quad \text{vs.} \quad \H^{\textnormal V}_1:V_\bP\neq V_\bQ\ ,
\end{equation}
with {\em test statistic} defined as the {\em variance discrepancy} $\textnormal{VD}(\bP,\bQ;\ell) = \left(V_{\bP}-V_{\bQ}\right)^2$ between $\bP$ and $\bQ$.


\smallskip
\noindent
\textbf{Estimator.}
Since $\bP$ and $\bQ$ are analytically unknown, we empirically estimate $\textnormal{VD}(\bP,\bQ;\ell)$ using $X$ and $Y$, such that
\begin{equation}\label{eq:vd_hat}
    \widehat{\textnormal{VD}}(X,Y;\ell) = \left(\widehat{V}(X;\ell)-\widehat{V}(Y;\ell)\right)^2\ ,
\end{equation}
where for $Z\in\{X,Y\}$, we define
\begin{equation*}
\widehat{V}(Z;\ell) = \sum_{\bm{z}\in Z}\frac{\ell(f(\bm{z}),f(\bm{z}))}{|Z|}-\sum_{\bm{z}_i\neq\bm{z}_j}^{Z}\frac{\ell(f(\bm{z}_i),f(\bm{z}_j))}{|Z|(|Z|-1)}\ .\\
\end{equation*}
If $\widehat{\textnormal{VD}}(X,Y;\ell)$ exceeds the {\em threshold} estimated via {\em permutation test}, we consider $\bP$ and $\bQ$ to be significantly different.

This way VD completes our response to~\cref{factor:global_uncertainty}: it detects uncertainty shifts at the {\em global} level, comparing how tightly $\bP$ and $\bQ$ concentrate across their support in the RKHS.
Next, we will examine {\em local} uncertainty shifts---how features fluctuate around individual examples under small, isotropic perturbations, which differ between CEs and AEs as noted in~\cref{factor:local_uncertainty}.
This motivates a perturbation-based covariance comparison that complements VD.

\subsection{Perturbation-based Covariance Discrepancy}\label{sec:pcd}
Recall that CEs often show consistent responses to isotropic perturbations~\cite{Cohen:Rosenfeld:Kolter2019}, while AEs are constructed to be near the decision boundaries of $f$ (to cause mis-classification)~\cite{Moosavi:Fawzi:Frossard2016,Moosavi:Fawzi:Fawzi:Frossard:Soatto2018, Ilyas:Santurkar:Tsipras:Engstrom:Tran:Madry2019}. 
Accordingly, AEs are expected to exhibit less stable feature responses to even small perturbations, especially in directions that matter for $f$'s decisions, indicating higher local sensitivity than CEs (\cref{factor:local_uncertainty}).

\smallskip
\noindent
\textbf{Local Uncertainty as Perturbation-Induced-Covariance.}
To probe local sensitivity, we specifically compare the distributions of per‑example feature covariances under isotropic Gaussian noises.
For each example $\x \in \cX$, we draw $K$ random perturbations $\{\mdelta_k\}_{k=1}^K\stackrel{\text{i.i.d.}}{\sim}\cN(\bm{0},\sigma^2I)$, extract features $\{f(\x+\mdelta_k)\}_{k=1}^K$, and compute the covariance matrix of $\x$:
\begin{equation}\label{eq:per_sample_cov}
    \Sigma_{\x} = \frac{1}{K-1} \sum_{k=1}^{K} (f(\x + \mdelta_k) - \overline{f_{\x}})(f(\x + \mdelta_k) - \overline{f_{\x}})^{\top},
\end{equation}
where $\overline{f_{\x}} \triangleq K^{-1} \sum_{k=1}^{K} f(\x + \mdelta_k)$.
The covariance matrix $\Sigma_{\x}$ encodes {\em local uncertainty} via its eigenstructure: large eigenvalues indicate directions where features vary substantially under Gaussian perturbation, small eigenvalues indicate stability~\cite{zhang2023detecting,Yang:Li:Xu:Kailkhura:Xie:Li2022}.


\smallskip
\noindent
\textbf{Test.}
To compare differences in {\em local uncertainty} (as distributional discrepancy) between $\bP$ and $\bQ$, we instantiate Eq.\eqref{eq:sadd_as_ht} regarding the perturbation–induced covariance embeddings by testing the equality of the corresponding kernel mean embeddings with null and alternative hypotheses as:
\begin{equation}\label{eq:hype_cd}
\H^{\textnormal C}_0: \mmu_{\Sigma_\bP}=\mmu_{\Sigma_\bQ}\ \ \ \text{vs.}\ \ \ \H^{\textnormal C}_1: \mmu_{\Sigma_\bP}\neq\mmu_{\Sigma_\bQ}\ ,
\end{equation}
where $\mmu_{\Sigma_\bP} \triangleq E_{\x\sim\bP}\left[\rho(\Sigma_{\x},\cdot)\right]$ and $\mmu_{\Sigma_\bQ} \triangleq E_{\y\sim\bQ}\left[\rho(\Sigma_{\y},\cdot)\right]$ are mean embeddings of the covariance distributions for $\bP$ and $\bQ$ in the RKHS $\cH_\rho$\footnote{Mathematically, the covariance matrix $\Sigma_{\x}\in \cS_{+}(\cF)$ defines a positive semi-definite (PSD) covariance operator on the semantic feature space $\cF$, with $\cS_{+}(\cF)$ being the cone of PSD operators on $\cF$. Thus, the RKHS $\cH_{\rho}$ where we embed $\Sigma_{\x}$ into and define $\textnormal{PCD}(\bP,\bQ,\rho)$ is characterized by the kernel $\rho:\cS_{+}(\cF) \times \cS_{+}(\cF)\rightarrow\bR$ and corresponding map $\rho(\cdot, \Sigma_{\x}) \in \cH_{\rho}$.}.
We define {\em perturbation-based covariance discrepancy}~(PCD) as $\textnormal{PCD}(\bP,\bQ,\rho) \triangleq \|\mmu_{\Sigma_\bP}-\mmu_{\Sigma_\bQ}\|_{\cH_\rho}^2$, which serves as {\em test statistic}.
Equivalently, $\textnormal{PCD}(\bP,\bQ,\rho)$ mirrors Eq.\eqref{eq: mmd} but measures the difference between distributions of covariance matrices under kernel $\rho$.
In our experiments, $\rho$ is either the log-RBF kernel or the Gaussian kernel defined directly on covariance matrices (see Appendix~\ref{appendix:exp}).

\smallskip
\noindent
\textbf{Estimator.}
In line with Eq.\eqref{eq: mmd}, we empirically estimate $\textnormal{PCD}(\bP,\bQ,\rho)$ with finite examples $X$ and $Y$, such that
\begin{equation}\label{eq:PCD_hat}
\widehat{\textnormal{PCD}}(X,Y;\rho)=
\sum_{i\neq j}^n\frac{\rho(\Sigma_{\x_i},\Sigma_{\x_j})}{n(n-1)}-\sum_{i=1}^n\sum_{j=1}^m\frac{2\rho(\Sigma_{\x_i},\Sigma_{\y_j})}{nm}+\sum_{i\neq j}^m\frac{\rho(\Sigma_{\y_i},\Sigma_{\y_j})}{m(m-1)}.
\end{equation}
This completes our response to~\cref{factor:local_uncertainty}: PCD detects uncertainty shifts between $\bP$ and $\bQ$ at the {\em local} level by comparing how semantic features fluctuate under isotropic noise around individual examples. 
Together with VD's {\em global} spread comparison, we now possess two complementary lenses through which adversarial artifacts manifest as measurable distributional uncertainty signals.

\section{Making SAD Aware of Local and Global Uncertainty, Adaptively}
\label{sec: dual}
While VD and PCD quantify complementary uncertainty patterns in the semantic feature space, attacks vary in how strongly they activate each effect: some primarily displace AEs from the clean manifold (i.e., strong VD signal)~\citep{Feinman:Curtin:Shintre:Gardner2017}; others tend to exploit local gradients near boundaries (i.e., strong PCD signal)~\cite{Ilyas:Santurkar:Tsipras:Engstrom:Tran:Madry2019}.
For SAD detection to be effective across diverse attacks, it should adapt to the uncertainty present in the data based on each statistic's informativeness for the query data rather than treat each statistic uniformly.

In light of this, we aggregate VD and PCD into {\em a unified test} that accounts for their joint correlation structure~\cite{zhou2025dual}, yielding an adaptive detector sensitive to both global and local adversarial artifacts, termed \emph{\textbf{U}ncertainty-aware \textbf{S}tatistical \textbf{A}dversarial \textbf{D}etection}~(USAD).
\vskip -0.08in
\smallskip
\noindent
\textbf{USAD as Hypothesis Test.}
Still, USAD instantiates Eq.\eqref{eq:sadd_as_ht} by testing distributional equality:
\vskip -0.2in
\begin{equation*}
\H_0^{\mathrm{A}}:\ \bP=\bQ \quad\text{vs.}\quad \H_1^{\mathrm{A}}:\ \bP\neq\bQ,
\end{equation*}
and defining the corresponding aggregated {\em test statistic} as
\begin{equation}\label{eq:usad_statistic}
\cT^{\rm A}(X, Y; \ell, \rho) = \cT(X, Y; \ell, \rho)^\top \widehat{\Sigma}_{\rm A}^{-1} \cT(X, Y; \ell, \rho),
\end{equation}
where $\cT(X,Y;\ell,\rho) \triangleq \left(\widehat{\textnormal{VD}}(X,Y;\ell),\widehat{\textnormal{PCD}}(X,Y;\rho)\right)^{\top}$ and $\widehat{\Sigma}_{\rm A}$ denotes the correlation between the statistics VD and PCD.
\vskip -0.08in
\smallskip
\noindent
\textbf{Weighted Statistics Aggregation.}
$\cT^{\rm A}(X, Y; \ell, \rho)$ combines both global $\widehat{\textnormal{VD}}(X,Y;\ell)$ and local uncertainty shifts measure $\widehat{\textnormal{PCD}}(X,Y;\rho)$, following {\em correlation-aware} statistics ensembling~\cite{zhou2025dual}, with the key idea to weight $\widehat{\textnormal{VD}}(X,Y;\ell)$ and $\widehat{\textnormal{PCD}}(X,Y;\rho)$ according to their empirical dependence, so that redundant components are down-weighted while complementary ones are emphasized.

\smallskip
\noindent
\textbf{Statistics Correlation Estimation.}
In practice, we estimate their correlation structure under the null hypothesis $\H_0^{\rm A}$ from CEs $X \sim \bP$, which characterizes how $\cT^{\rm A}(X, Y; \ell, \rho)$ behaves under $\H_0^{\rm A}$~\cite{chatterjee2025boosting,zhou2025dual}.
Specifically, we perform $B$ bootstrap resamplings from $X$.
In each iteration $b$, we draw two disjoint subsets $X_b$ and $Y_b$ from the clean pool (with $|X_b|=n$ and $|Y_b|=m$), and compute the statistics vector $\cT_b(X,Y;\ell,\rho)$.
From $\{\cT_b\}_{b=1}^{B}$, we estimate the correlation
\vskip -0.15in
\begin{equation}\label{eq:sigma_aggregate}
    \widehat{\Sigma}_{\mathrm{A}} = \frac{1}{B-1}\sum_{b=1}^{B} \bigl(\cT_b- \widehat{\cT_{\mmu}} \bigr)\bigl(\cT_b-\widehat{\cT_{\mmu}}\bigr)^{\!\top},
\end{equation}
\vskip -0.12in
where $\widehat{\cT_{\mmu}} = B^{-1} \sum_{b=1}^{B} \cT_b$ is the mean estimated with all $B$ iterations.
The diagonal entries of $\widehat{\Sigma}_{\mathrm{A}}$ capture individual variabilities of VD and PCD, while off-diagonal entries reflect their correlation under the null.

\smallskip
\noindent
\textbf{Permutation Test.}
Following the SAD pipeline (\cref{sec:pre}), we apply a {\em permutation test} with $R$ random permutations to $\cT^{\rm A}(X, Y; \ell, \rho)$ (Eq.\eqref{eq:usad_statistic}).
For each $r=1,\dots,R$, we compute $\cT^{\rm A}_{(r)}$ while keeping $\widehat{\Sigma}_{\mathrm{A}}$ fixed at Eq.\eqref{eq:sigma_aggregate}, so the aggregation weights reflect the null covariance structure across permutations.
At significance level $\alpha$, the testing threshold is $t_\alpha\ =\ \mathrm{Quantile}_{1-\alpha}(\{\cT^{\mathrm{A}}_{(r)}\}_{r=1}^{R})$.
\smallskip
\noindent
\textbf{Detection.}
We reject $\H_0^{\rm A}$ (i.e., classifying $Y$ as adversarial) if and only if $\cT^{\rm A}(X, Y; \ell, \rho) > t_\alpha$.
\section{Theoretical Analysis}
\label{sec: theo}
In this section, we present theoretical guarantees for our test statistics, with full proofs in Appendix~\ref{app:proof}.
These results are obtained with standard tools, namely concentration results for $U$-statistics~\cite{Hoeffding1994}, permutation tests~\cite{Kim:Balakrishnan:Wasserman2022}, and the correlation-aware statistics aggregation of~\citet{zhou2025dual}; their purpose is to establish the validity and consistency of the new statistics VD and PCD and of the USAD framework, rather than to develop new proof techniques.
We first establish Type-I error control as follows.
\begin{theorem}\label{thm:perm}
Given the test batch $Y$ consists of clean data, the VD, PCD, or USAD, when combined with a permutation test at significance level $\alpha \in (0,1)$, controls the Type~I error at level $\alpha$; that is, the probability of falsely identifying the test batch $Y$ as adversarial is at most $\alpha$.
\end{theorem}
We next study the test power of VD as follows.
\begin{theorem}\label{thm:VD_test}
Let $\ell$ be a bounded kernel. Under the alternative hypothesis $\H^{\textnormal V}_1$, the power of the VD test, that is, the probability of correctly rejecting the null hypothesis $\H^{\textnormal V}_0$, converges to $1$ as $n,m\rightarrow\infty$ with $n/m\rightarrow c\in(0,\infty)$.
\end{theorem}
This guarantees that VD is a consistent statistical test; a similar result holds for PCD, as below.
\begin{theorem}\label{thm:PCD_test}
Let $\rho$ be a bounded kernel. Under the alternative hypothesis $\H^{\textnormal C}_1$, the power of the PCD test converges to $1$ as $n,m\rightarrow\infty$ with $n/m\rightarrow c\in(0,\infty)$.
\end{theorem}
We now extend these results to the USAD.
\begin{corollary}\label{cor:USAD_test}
If \textbf{either or both} the alternative hypotheses $\H^{\textnormal V}_1$ and $\H^{\textnormal C}_1$ hold, the power of the USAD test converges to $1$ as $n,m\rightarrow\infty$ with $n/m\rightarrow c\in(0,\infty)$, provided that the condition number of $\widehat{\Sigma}_{\rm A}$ remains bounded.
\end{corollary}
In Theorem~\ref{thm:perm} and Corollary~\ref{cor:USAD_test}, although our analysis focuses on USAD, the aggregation framework is general and can incorporate additional statistics such as MMD.
In particular, Corollary~\ref{cor:USAD_test} covers alternatives under which $\H^{\textnormal V}_1$ or $\H^{\textnormal C}_1$ holds, so alternatives that change the mean embedding while leaving $V_\bQ$ and $\mmu_{\Sigma_\bQ}$ unchanged are not targeted by the current USAD; such alternatives can be addressed by incorporating MMD, and Appendix~\ref{app:ablation} reports a VD+PCD+MMD variant, showing that adding more statistics does not necessarily improve finite-sample power and that the statistics should be selected according to the target alternatives.

\begin{remark}\label{rem:mixing_interp}
Our analysis concerns \textbf{distribution-level detection} and does \emph{NOT} require the query batch to be purely adversarial or purely clean: we test whether the query distribution $\bQ$ \emph{significantly differs} from a trusted clean reference $\bP$. If $Y$ contains a \textbf{non-negligible fraction} of adversarial examples, the induced mixture yields a detectable distribution shift, and the test power approaches $1$ for sufficiently large batches (Theorems~\ref{thm:VD_test}--\ref{thm:PCD_test} and Corollary~\ref{cor:USAD_test}). Conversely, if adversarial examples are diluted to a \textbf{vanishing fraction} (CE mixing level $\to 100\%$), the mixture can become effectively indistinguishable from clean data; in this regime, the rejection probability is controlled by the significance level $\alpha$ with finite-sample Type-I error guarantees (Theorem~\ref{thm:perm}).
Quantitatively, the population discrepancies underlying VD and PCD decrease with the fraction of adversarial examples in $Y$, so that a smaller fraction requires a larger query batch for detection (see Remark~\ref{rem:ae_fraction} in Appendix~\ref{app:low_fraction}).
\end{remark}

\section{Experiments}
\label{sec: experiment}
We evaluate our approaches on ResNet-18/50 \citep{resnet} and ViT-B-16 \citep{vit} trained on \textit{CIFAR-10} \citep{cifar} and \textit{ImageNet-1K} \citep{deng2009imagenet}. We report CIFAR-10 results here and defer ImageNet-1K results, Type-I error, transfer attack, $\ell_0, \ell_1, \ell_2$ attacks and adversarial video detection on UCF101 \citep{wei2019sparse} to Appendix~\ref{app:more_exp}. 
See the experimental setup in Appendix~\ref{appendix:exp}.

\begin{table*}[t]
\centering
\caption{Test power for different input example sizes $m = |Y|$ decreasing from 50 to 10 under PGD attack with $\epsilon=4/255$. The test power results are averaged over $1,000$ repetitions. The target model is ResNet-18 trained on \textit{CIFAR-10} dataset.}
\resizebox{0.9\linewidth}{!}{%
\begin{tabular}{c|ccccc}
\toprule
Method & $m = 10$ & $m = 20$ & $m = 30$ & $m = 40$ & $m = 50$ \\
\midrule
SAMMD \citep{Gao:Liu:Zhang:Han:Liu:Niu:Sugiyama2021} & $0.292 \pm 0.029$ & $0.486 \pm 0.022$ & $0.650 \pm 0.008$ & $0.756 \pm 0.019$ & $0.832 \pm 0.009$ \\
\cellcolor{orL}{PCD (ours)} & \cellcolor{orL}{$0.486 \pm 0.024$} & \cellcolor{orL}{$0.828 \pm 0.013$} & \cellcolor{orL}{$0.958 \pm 0.007$} & \cellcolor{orL}{$0.998 \pm 0.002$} & \cellcolor{orL}{$\mathbf{1.000 \pm 0.000}$} \\
\cellcolor{goL}{VD (ours)} & \cellcolor{goL}{$\mathbf{1.000 \pm 0.000}$} & \cellcolor{goL}{$\mathbf{1.000 \pm 0.000}$} & \cellcolor{goL}{$\mathbf{1.000 \pm 0.000}$} & \cellcolor{goL}{$\mathbf{1.000 \pm 0.000}$} & \cellcolor{goL}{$\mathbf{1.000 \pm 0.000}$} \\
\midrule
MMDAgg \citep{Schrab:Kim:Albert:Laurent:Guedj:Gretton2023} & $0.342 \pm 0.016$ & $0.476 \pm 0.022$ & $0.632 \pm 0.020$ & $0.816 \pm 0.017$ & $0.912 \pm 0.009$ \\
MMD-FUSE \citep{Biggs:Schrab:Gretton2023} & $0.524 \pm 0.017$ & $0.892 \pm 0.014$ & $0.994 \pm 0.004$ & $\mathbf{1.000 \pm 0.000}$ & $\mathbf{1.000 \pm 0.000}$ \\
MMD-DUAL \citep{zhou2025dual} & $0.752 \pm 0.022$ & $0.978 \pm 0.003$ & $\mathbf{1.000 \pm 0.000}$ & $\mathbf{1.000 \pm 0.000}$ & $\mathbf{1.000 \pm 0.000}$ \\
\cellcolor{grL}{USAD (ours)} & \cellcolor{grL}{$\mathbf{1.000 \pm 0.000}$} & \cellcolor{grL}{$\mathbf{1.000 \pm 0.000}$} & \cellcolor{grL}{$\mathbf{1.000 \pm 0.000}$} & \cellcolor{grL}{$\mathbf{1.000 \pm 0.000} $} & \cellcolor{grL}{$\mathbf{1.000 \pm 0.000}$} \\
\bottomrule
\end{tabular}
}
\label{tab:sample_size}
\end{table*}
\begin{table*}[t]
\centering
\caption{Test power under adaptive attack with $\epsilon=4/255$. We increase the input example sizes $m = |Y|$ from 10 to 50 to show that all the SAD methods can gain power from example size, even though the adaptive attack can effectively avoid the detection with insufficient examples. The test power results are averaged over $1,000$ repetitions. The target model is ResNet-18 trained on \textit{CIFAR-10} dataset.}
\resizebox{0.9\linewidth}{!}{%
\begin{tabular}{c|c|ccccc}
\toprule
Method & Attack & $m = 10$ & $m = 20$ & $m = 30$ & $m = 40$ & $m = 50$ \\
\midrule
\multirow{2}{*}{PCD} & PGD & $0.486 \pm 0.024$ & $0.828 \pm 0.013$ & $0.958 \pm 0.007$ & $0.998 \pm 0.002$ & $1.000 \pm 0.000$ \\
 & Adaptive & $0.388 \pm 0.027$ & $0.714 \pm 0.020$ & $0.768 \pm 0.033$ & $0.882 \pm 0.010$ & $0.914 \pm 0.014$ \\
\midrule
\multirow{2}{*}{VD} & PGD & $1.000 \pm 0.000$ & $1.000 \pm 0.000$ & $1.000 \pm 0.000$ & $1.000 \pm 0.000$ & $1.000 \pm 0.000$ \\
 & Adaptive & $0.186 \pm 0.023$ & $0.326 \pm 0.007$ & $0.416 \pm 0.022$ & $0.546 \pm 0.021$ & $0.614 \pm 0.011$ \\
\midrule
\multirow{2}{*}{USAD} & PGD & $1.000 \pm 0.000$ & $1.000 \pm 0.000$ & $1.000 \pm 0.000$ & $1.000 \pm 0.000$ & $1.000 \pm 0.000$ \\
 & Adaptive & $0.818 \pm 0.013$ & $0.998 \pm 0.002$ & $1.000 \pm 0.000$ & $1.000 \pm 0.000$ & $1.000 \pm 0.000$ \\
\midrule
\multicolumn{2}{c|}{Avg. Decrease}  & $-0.365 \pm 0.207$ & $-0.263 \pm 0.193$ & $-0.258 \pm 0.168$ & $-0.190 \pm 0.131$ & $-0.157 \pm 0.111$ \\
\bottomrule
\end{tabular}
}
\label{tab:adaptive}
\vspace{-1em}
\end{table*}

\noindent
\textbf{Baselines.}
We compare our methods against the state-of-the-art SAD method: \textbf{SAMMD}~\cite{Gao:Liu:Zhang:Han:Liu:Niu:Sugiyama2021}, an MMD-based detector that uses semantic features, which significantly outperforms all the single MMD-based two-sample testing methods (e.g., MMD-D~\cite{Liu:Xu:Lu:Zhang:Gretton:Sutherland2020}, C2ST~\cite{Cheng:Cloninger2019}), we refer readers interested in these experimental results to SAMMD.
We do not directly compare with per-example detectors (e.g., LID~\citep{ma2018characterizing}, EPS-AD~\citep{zhang2023detecting}, and PFD~\citep{wang2025detecting}), which classify each input separately and do not provide explicit Type-I error control; we view them as complementary to SAD (see Appendix~\ref{app:discussion}).
\begin{figure*}[t]
\vspace{-0.4em}
\begin{center}
\centerline{\includegraphics[width=\columnwidth]{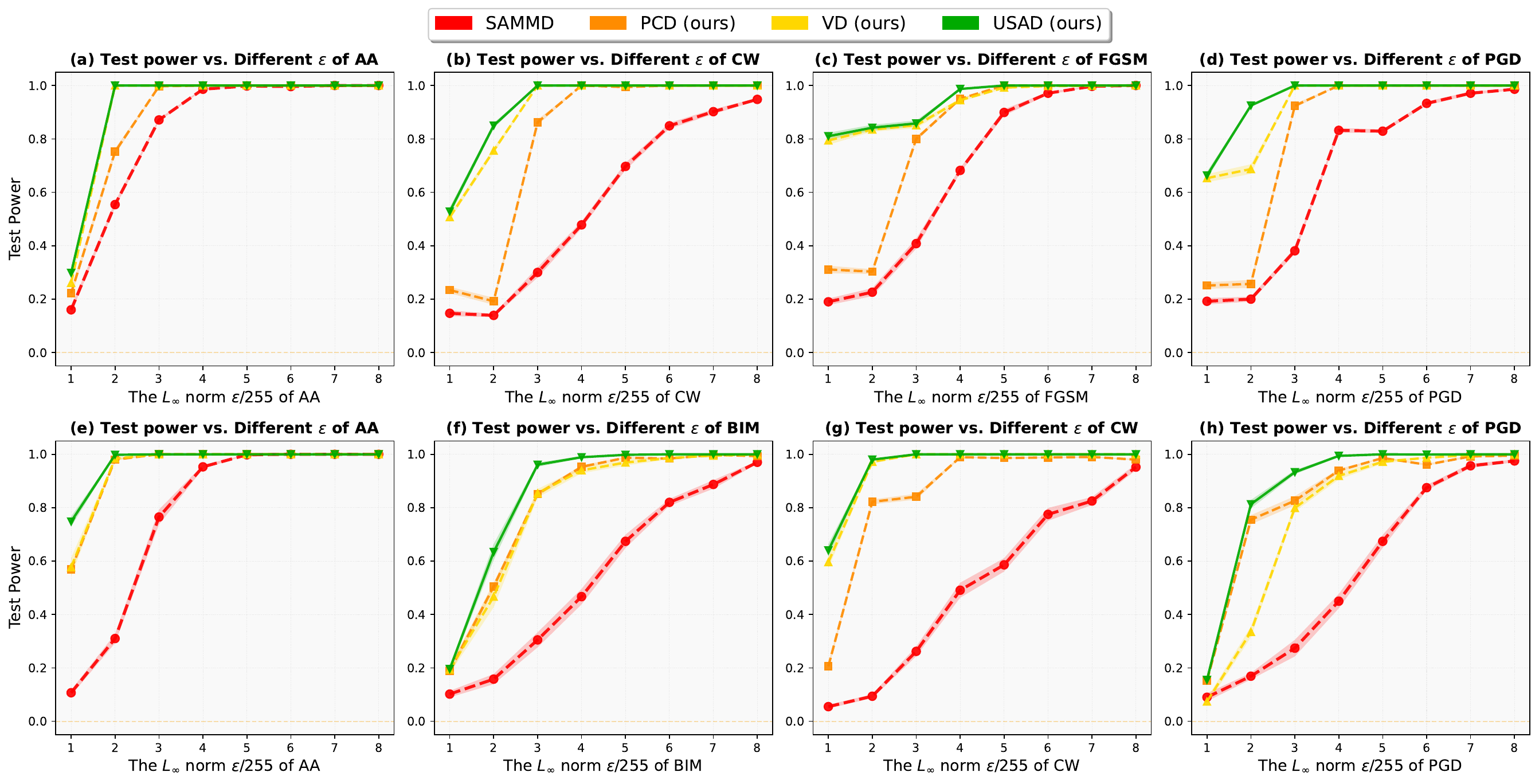}}
\vspace{-0.8em}
\caption{Results $(a\!-\!h)$ are test power (detection rate) under different adversarial attacks with different $\epsilon$, the given adversarial samples all share the same sample size $|Y| = 50$. The results are averaged over $1,000$ repetitions and the ideal test power is $1$. The target model is ResNet-18 trained on \textit{CIFAR-10} dataset in $(a\!-\!d)$, and ResNet-50 trained on \textit{ImageNet} dataset in $(e\!-\!h)$.}
\label{exp:main}
\end{center}
\vspace{-1.5em}
\end{figure*}
\begin{table*}[t]
\centering
\caption{Rejection rates for different \emph{clean examples} (CE) mixing levels from $0\%$ to $100\%$ with $|Y|=50$, under PGD attack with $\epsilon=4/255$. The results are averaged over $1,000$ repetitions. The target model is ResNet-18 trained on \textit{CIFAR-10} dataset.}
\resizebox{0.9\linewidth}{!}{%
\begin{tabular}{c|cccccc}
\toprule
Method & $0$ & $20$ & $40$ & $60$ & $80$ & $100$ \\
\midrule
SAMMD & $0.828 \pm 0.016$ & $0.620 \pm 0.013$ & $0.344 \pm 0.010$ & $0.196 \pm 0.012$ & $0.052 \pm 0.010$ & $0.040 \pm 0.009$ \\
\cellcolor{orL}{PCD (ours)} &
\cellcolor{orL}{$0.996 \pm 0.002$} & \cellcolor{orL}{$0.952 \pm 0.005$} & \cellcolor{orL}{$0.356 \pm 0.020$} &
\cellcolor{orL}{$0.142 \pm 0.010$} & \cellcolor{orL}{$0.066 \pm 0.011$} & \cellcolor{orL}{$0.056 \pm 0.004$} \\
\cellcolor{goL}{VD (ours)} &
\cellcolor{goL}{$\mathbf{1.000 \pm 0.000}$} & \cellcolor{goL}{$\mathbf{1.000 \pm 0.000}$} & \cellcolor{goL}{$\mathbf{1.000 \pm 0.000}$} &
\cellcolor{goL}{$\mathbf{0.998 \pm 0.002}$} & \cellcolor{goL}{$\mathbf{0.778 \pm 0.016}$} & \cellcolor{goL}{$0.060 \pm 0.006$} \\
\midrule
MMDAgg &
$\mathbf{1.000 \pm 0.000}$ & $\mathbf{1.000 \pm 0.000}$ & $0.992 \pm 0.003$ & $0.548 \pm 0.025$ & $0.096 \pm 0.005$ & $0.050 \pm 0.009$ \\
MMD-FUSE &
$0.908 \pm 0.010$ & $0.682 \pm 0.013$ & $0.330 \pm 0.022$ & $0.168 \pm 0.015$ & $0.050 \pm 0.006$ & $0.030 \pm 0.009$ \\
MMD-DUAL &
$\mathbf{1.000 \pm 0.000}$ & $0.990 \pm 0.003$ & $0.806 \pm 0.012$ & $0.274 \pm 0.012$ & $0.068 \pm 0.009$ & $0.040 \pm 0.006$ \\
\cellcolor{grL}{USAD (ours)} &
\cellcolor{grL}{$\mathbf{1.000 \pm 0.000}$} & \cellcolor{grL}{$\mathbf{1.000 \pm 0.000}$} & \cellcolor{grL}{$\mathbf{1.000 \pm 0.000}$} &
\cellcolor{grL}{$0.900 \pm 0.014$} & \cellcolor{grL}{$0.350 \pm 0.018$} & \cellcolor{grL}{$0.038 \pm 0.007$} \\
\bottomrule
\end{tabular}
}
\vspace{-0.8em}
\label{tab:mixing_level}
\end{table*}

\noindent
\textbf{Evaluation Metric.}
We evaluate detection performance using test power, which measures the ability to correctly identify adversarial examples while maintaining a significance level of $\alpha=0.05$. Figure~\ref{exp:main}(a--d) and Figure~\ref{exp:main}(e--h) report the test power under five attacks, where the query batches consist entirely of adversarial examples: \emph{AutoAttack} (AA) \cite{croce2020reliable}, \emph{basic iterative method} (BIM) \cite{kurakin2017adversarial}, \emph{Carlini \& Wagner} (CW) attack \cite{carlini2017towards}, \emph{fast gradient sign method} (FGSM) \cite{goodfellow2015explaining} and \emph{projected gradient descent} (PGD) \cite{madry2018towards} under different target models that are trained on either \textit{CIFAR-10} or \textit{ImageNet-1K}. 

\noindent
\textbf{Overall Result.} Our results demonstrate that USAD consistently achieves superior performance across all attacks. Notably, for all attacks, USAD reaches perfect detection (test power $\approx 1.0$) whenever the perturbation budget $\epsilon \geq 4/255$ under a relatively small example size $|Y| = 50$, and even our proposed single uncertainty-aware methods PCD and VD can reach perfect detection whenever $\epsilon \geq 6/255$, all three significantly outperforming the baseline SAMMD.

\noindent
\textbf{Impact of Sample Size.}
To study robustness to varying example sizes, we compare PCD and VD to the single-kernel MMD baseline (SAMMD), and compare USAD to representative multi-kernel MMD tests (e.g., \textbf{MMDAgg}~\cite{Schrab:Kim:Albert:Laurent:Guedj:Gretton2023}, \textbf{MMD-FUSE}~\cite{Biggs:Schrab:Gretton2023}, and \textbf{MMD-DUAL}~\cite{zhou2025dual}) that boost power in small-sample regimes via kernel aggregation. Table~\ref{tab:sample_size} reports results as the query size ranges from 10 to 50. PCD and VD consistently outperform SAMMD, and USAD consistently outperforms MMD-based kernel aggregation methods, maintaining power $1.0$ even with \emph{only 10 query examples}.
An ablation comparing the correlation-aware aggregation of USAD with a simple equal-weight sum of standardized VD and PCD is reported in Appendix~\ref{app:ablation}.


\noindent
\textbf{Performance against Adaptive Attack.}
Inspired by \citet{Gao:Liu:Zhang:Han:Liu:Niu:Sugiyama2021,kaya2022generating}, we evaluate adaptive attacks designed to evade SAD methods. Specifically, we use adaptive PGD variants that maximize the adversarial cross-entropy loss while minimizing the discrepancy measured by the target statistic (VD, PCD, or USAD; see Eqs.~\eqref{eq:vd_hat}, \eqref{eq:PCD_hat}, and \eqref{eq:usad_statistic}). Table~\ref{tab:adaptive} shows that although power drops compared to non-adaptive attacks, USAD still achieves $0.818$ power with only $10$ examples, indicating strong robustness. The component-wise results further suggest that PCD better retains power under adaptive attacks: such attacks can suppress the global spread signal used by VD, whereas reducing the local instability captured by PCD is harder while preserving adversarial effectiveness. More importantly, these results highlight a key advantage of SAD over single-example detectors (e.g., LID \citep{ma2018characterizing} and EPS-AD \citep{zhang2023detecting}): adaptive attacks can reduce power in small batches, but as batch size grows, SAD consistently separates AEs from CEs.
Appendix~\ref{app:mvd_pfd} further compares USAD with the per-example detector PFD~\citep{wang2025detecting} under an adaptive C\&W attack~\citep{Carlini:Wagner2017}.

\textbf{Impact of CE mixing levels.}
We evaluate all methods as the fraction of \emph{clean examples (CEs)} in the query batch increases from $0\%$ to $100\%$. As reported in Table~\ref{tab:mixing_level}, detection becomes harder as the CE proportion grows (the induced shift weakens), yet our methods attain the highest power across mixing levels. When the query batch is purely clean ($100\%$ CE), the problem reduces to the null setting and the rejection rate (i.e., Type-I error) stays around the significance level $\alpha=0.05$. Notably, even at a high mixing level (e.g., $80\%$ CE), VD still achieves strong power (around $0.78$), while competing methods degrade to near-chance performance. We provide a more detailed discussion of this distribution-level perspective and the mixture regimes in Remark~\ref{rem:mixing_interp}.
Lower AE fractions in the query batch, together with larger batch sizes, are studied in Appendix~\ref{app:low_fraction}.


\section{Conclusion}
This work reveals that the general-purpose metric, MMD, is insensitive to how adversarial perturbations actually manifest as distributional uncertainty changes.
We addressed this by developing two statistics to measure global and local uncertainties and combining them into an uncertainty-aware SAD detector, achieving high power in the small-sample regime with theoretically false alarm control. Our formulation offers a practical template: task‑aligned features, statistics that encode the relevant geometry, and principled aggregation to discount redundancy. 
We expect this recipe to extend beyond SAD to general distribution shift testing.

\newpage
\section*{Acknowledgements}
This research was supported by the University of Melbourne’s Research Computing Services and the Petascale Campus Initiative. ZJZ and XYT are supported by the Melbourne Research Scholarship and the ARC with grant number DE240101089. JCZ and YYG are supported by the Melbourne Research Scholarship. LHP is supported by the ARC with grant number LP240100101. FL is supported by the ARC with grant number DE240101089, LP240100101, DP230101540 and the NSF\&CSIRO Responsible AI program with grant number 2303037.
\bibliography{main}
\bibliographystyle{unsrtnat}


\newpage
\appendix
\onecolumn
\addtocontents{toc}{\protect\setcounter{tocdepth}{2}}
\begingroup
\renewcommand{\contentsname}{Appendix Contents}
\tableofcontents
\endgroup

\newpage
\section{More Discussions on the Scope, Position, and Limitations of SAD}\label{app:discussion}

This paper focuses on {\em statistical adversarial detection}~(SAD) as an active research area and a promising paradigm complementary to existing adversarial defense strategies.
We have identified the fundamental limitation of existing SAD practices and proposed a concrete variant termed {\em Uncertainty-aware Statistical Adversarial Detection}~(USAD) accordingly.

\textbf{Scope of SAD.}
By framing adversarial detection as a two-sample hypothesis test, SAD offers two practical advantages: it requires neither modification of the classifier nor attack-specific training, and it provides a \emph{lightweight monitoring primitive}.
Once a query window is flagged as distributionally suspicious, the system can throttle the source, trigger more expensive defenses, or route the traffic for human inspection, while maintaining controlled false-alarm rates.
Beyond test-time monitoring, the same batch-level tests can also support \emph{data cleaning} in large-scale training pipelines, where coordinated manipulation may contaminate entire mini-batches from distributed or third-party data sources.
In this scope, USAD strengthens SAD by substantially improving detection reliability in small and early-stage windows, where MMD-based methods often lack sufficient power.
This makes USAD particularly well-suited for timely monitoring and data vetting in realistic deployment settings~\citep{Gao:Liu:Zhang:Han:Liu:Niu:Sugiyama2021}.
Our target setting is provider-side deployment, where the model owner has access to the classifier $f$ and its intermediate representations.
Therefore, USAD assumes a white-box setting from the defender's perspective, rather than a strict label-only black-box setting; this assumption is also common in feature-based adversarial detection and prior SAD methods, such as LID~\citep{ma2018characterizing}, SAMMD~\citep{Gao:Liu:Zhang:Han:Liu:Niu:Sugiyama2021}, and \citet{zhang2025one}, which use internal model features for detection.
Extending USAD to settings where only predicted labels or confidence scores are available is an interesting direction that we leave for future work.

\textbf{Position of USAD.}
We emphasize that USAD (and other SAD variants) is \underline{not intended to replace} \underline{per-example detectors}.
Rather, it serves as a statistical monitor operating on short query windows.
Detecting individual AEs in isolation is known to be nearly as hard as robust classification and typically relies on strong attack-specific assumptions~\citep{tramer2022detecting}.
In contrast, SAD deliberately trades per-example granularity for rigorous false-alarm control and reliable detection power at the window level, making it a natural building block for monitoring and throttling in deployed systems.
In Section~\ref{sec: experiment}, we further show that USAD maintains high power even when AEs are sparsely mixed with CEs within each window.

\textbf{Towards Per-example Localization.}
We view USAD and per-example detectors as complementary: USAD first flags a suspicious batch with explicit false-alarm control, after which a per-example detector can be used to localize or filter individual AEs.
Moreover, the VD and PCD statistics can also be decomposed into per-example contributions.
For VD, define, for each query $\y_j \in Y$,
\begin{equation*}
    v_j = \frac{\ell(f(\y_j),f(\y_j))}{m}-\sum_{i\neq j}^{m}\frac{\ell(f(\y_i),f(\y_j))}{m(m-1)},
\end{equation*}
so that $\widehat{V}(Y;\ell)=\sum_{j=1}^{m}v_j$. Then, by Eq.~\eqref{eq:vd_hat},
\begin{equation*}
    \widehat{\textnormal{VD}}(X,Y;\ell)=\sum_{j=1}^{m}\left(\widehat{V}(X;\ell)-\widehat{V}(Y;\ell)\right)\left(\frac{\widehat{V}(X;\ell)}{m}-v_j\right).
\end{equation*}
Similarly, by Eq.~\eqref{eq:PCD_hat}, $\widehat{\textnormal{PCD}}(X,Y;\rho)=\sum_{j=1}^{m}c^{\textnormal{PCD}}_j$, where
\begin{equation*}
    c^{\textnormal{PCD}}_j=\frac{1}{m}\sum_{i\neq i'}^{n}\frac{\rho(\Sigma_{\x_i},\Sigma_{\x_{i'}})}{n(n-1)}-\sum_{i=1}^{n}\frac{2\rho(\Sigma_{\x_i},\Sigma_{\y_j})}{nm}+\sum_{i\neq j}^{m}\frac{\rho(\Sigma_{\y_i},\Sigma_{\y_j})}{m(m-1)},
\end{equation*}
in which the first term distributes the clean--clean term of Eq.~\eqref{eq:PCD_hat} evenly over the queries.
These contributions can be combined using the USAD aggregation weights (i.e., $\widehat{\Sigma}_{\rm A}^{-1}$ in Eq.~\eqref{eq:usad_statistic}) to rank the query examples: larger positive contributions indicate examples that contribute more strongly to the rejection and can be prioritized for filtering or further examination.
We note that the current guarantees remain at the batch level, and we leave a formal per-example analysis for future work.

\textbf{Limitations of SAD.}
That said, we acknowledge the limitations of SAD in its current forms.
Prior to USAD, SAD is instantiated almost exclusively with \emph{maximum mean discrepancy} (MMD)~\citep{Gretton:Borgwardt:Rasch:Scholkopf:Smola2012} as the test statistic $\cT(X, Y)$~\citep{Gao:Liu:Zhang:Han:Liu:Niu:Sugiyama2021, zhang2025one}.
These MMD-based SAD methods are inherently {\em sample-inefficient} (see Section~\ref{sec:intro} and Figure~\ref{fig:mmd_test_power}): reliable detection typically requires query windows with sufficiently large effective sample sizes and non-trivial adversarial fractions to achieve high test power.
This limitation stems from the fact that MMD primarily captures discrepancies in kernel mean embeddings, whose estimation variance can be substantial in small-sample regimes.
As a result, MMD-based SAD methods can fail precisely in early-stage or low-volume attack regimes where timely detection is most important.

\textbf{Limitations of USAD.}
We also acknowledge the limitations of USAD itself.
First, USAD targets changes in the global feature spread or the perturbation-induced local covariance, which are specifically motivated by the uncertainty patterns of AEs rather than arbitrary distributional differences.
However, as Eq.~\eqref{eq:sadd_as_ht} tests distributional equality, these statistics are not formally specific to adversarial attacks, and benign shifts that alter the same quantities may also be detected; the response of USAD to benign shifts thus depends on whether they change the uncertainty patterns captured by VD or PCD, and USAD relies on a clean reference set $X$ that is representative of the deployment distribution.
Appendix~\ref{app:benign_shift} examines a benign shift empirically and shows that benign-shift and adversarial batches remain highly separable according to VD and PCD.
Second, detection power depends jointly on the fraction of AEs in the query batch, the batch size, and the attack signal: when AEs constitute only a very small fraction of the query batch, reliable detection requires substantially larger batches (Appendix~\ref{app:low_fraction}), and detection may also become more challenging when AEs induce only weak changes in the quantities captured by VD or PCD.
Third, USAD is not sensitive to every possible alternative: distributional differences that leave $V_\bQ$ and $\mmu_{\Sigma_\bQ}$ unchanged (e.g., pure mean-embedding shifts) are not targeted by the current statistics, although the USAD framework can incorporate additional statistics such as MMD for this purpose (Appendix~\ref{app:ablation}).
Finally, USAD requires access to the semantic features of the classifier $f$ (see the scope above), and PCD incurs a higher computational cost than VD (Appendix~\ref{app:computation_efficiency}).

\section{Related Works}\label{app:related}

\paragraph{Variance Discrepancy.}
Prior studies have utilized second-order information in RKHS in several ways. Some analyze covariance operators to derive asymptotic null distributions for two-sample tests \citep{Harchaoui:Bach:Moulines2007,Harchaoui:Bach:Moulines2008,Kirchler:Khorasani:Kloft:Lippert2020,Kubler:Jitkrittum:Scholkopf:Muandet2020}, while others spectrally regularize MMD to incorporate covariance structure into a mean-embedding statistic \citep{Hagrass:Sriperumbudur:Li2024}. NAMMD~\citep{zhou2026closeness} further employs the trace of the covariance operator to improve comparability in distribution closeness testing using a simple finite-sample estimator, and the maximum variance discrepancy (MVD)~\cite{makigusa2024} compares the full covariance operators (see \cref{sec:vd} for its relationship to VD). In contrast, our \emph{variance discrepancy} (VD) captures a signal particularly relevant for adversarial detection by directly comparing the within-distribution variances of semantic features. This focuses on the concentration of features within each distribution, which reflects the global uncertainty in \cref{factor:global_uncertainty} and provides additional sensitivity to manifold deviations characteristic of adversarial examples \citep{Feinman:Curtin:Shintre:Gardner2017, ma2018characterizing}.

\paragraph{Perturbation-based Covariance Discrepancy.}
Previous score-based detectors typically perturb a test point and read out label uncertainty. Some methods craft gradient-guided perturbations at one or multiple layers of a deep classifier to probe local sensitivity \cite{Lee:Lee:Lee:Shin2018}; others use random perturbations to measure prediction instability \cite{huang2019model,roth2019odds,zhang2023detecting,wang2025detecting}, in line with the connection between adversarial robustness and stability under additive noise established by randomized smoothing and related certification methods \cite{Cohen:Rosenfeld:Kolter2019,Li:Chen:Wang:Carin2019,Yang:Li:Xu:Kailkhura:Xie:Li2022}. A complementary line estimates uncertainty from model randomness, such as Monte-Carlo dropout or Bayesian parameter sampling, and then perturbs predictions accordingly \cite{Carlini:Wagner2017,deng2021libre}. Inspired by these ideas, our approach moves from label space to feature space: instead of only tracking how the predicted label flips, we estimate per-example perturbation covariances of deep features and compare their distributions via a kernel on PSD matrices (\cref{sec:pcd}). This captures richer local geometry, yields a simple, scalable statistic that aligns with manifold-based accounts of clean vs. adversarial behavior.
\section{Permutation Test}\label{app:perm}

We detail the permutation test used throughout {\em statistical adversarial detection}~(SAD) to calibrate the decision threshold $t_\alpha$ for our proposed statistics (VD, PCD, and their aggregate USAD).

\paragraph{From SAD to Calibration.}
For completeness, we briefly review the problem setup of SAD and the reason for conducting the permutation test.
\cref{sec:sadd_revisit} formulates SAD as a two-sample hypothesis test problem: given a reference set of {\em clean examples}~(CEs) $X = \{ \x_i \}_{i=1}^{n} \stackrel{\text{i.i.d.}}{\sim} \bP$ and a query set $Y = \{ \y_j \}_{j=1}^{m} \stackrel{\text{i.i.d.}}{\sim} \bQ$, we test 
\begin{equation*}
    \H_0 : \bP = \bQ
\quad \text{vs.} \quad
    \H_1 : \bP \neq \bQ.
\end{equation*}
At significance level $\alpha$, we want a {\em test statistic} $\cT (X, Y)$ and a {\em testing threshold} $t_\alpha$ such that $\Pr( \cT(X, Y) > t_\alpha \mid \H_0 ) \leq \alpha$, so that rejecting $\H_0$ (declaring $Y$ adversarial) has a controlled false alarm rate.
That is, the probability of falsely identifying the query data $Y$ as adversarial is at most $\alpha$.
In this paper, $\cT(X, Y)$ can be instantiated by VD (\cref{sec:vd} and Eq.\eqref{eq:vd_hat}), PCD (\cref{sec:pcd} and Eq.\eqref{eq:PCD_hat}), or their aggregate USAD (\cref{sec: dual} and Eq.\eqref{eq:usad_statistic}).

Here, the missing piece is $t_\alpha$, which is defined as the $(1-\alpha)$-quantile of the null distribution of $\cT(X, Y)$ under $\H_0$.
For most SAD methods and modern statistics (including MMD, VD, PCD, and USAD), this distribution is not available in closed form.
We therefore exploit the {\em exchangeability} property of the SAD setup: {\em under $\H_0: \bP = \bQ$, samples from $X$ and $Y$ are i.i.d., so the labels ``reference'' and ``query'' are arbitrary} and construct an empirical null distribution using {\em permutation test} with $X$ and $Y$.
We will finally determine the {\em testing threshold} $t_\alpha$ calibrated on this empirical null distribution.

\paragraph{Statistics Under Test.}
For notational convenience, we write
\[
\mathrm{STAT} \in 
\Big\{\widehat{\mathrm{VD}}(\cdot,\cdot;\ell),\;
      \widehat{\mathrm{PCD}}(\cdot,\cdot;\rho),\;
      \cT^{\rm A}(\cdot,\cdot;\ell,\rho)\Big\}
\]
and use $\phi$ to denote the corresponding kernel choice
($\phi=\ell$ for VD, $\phi=\rho$ for PCD, and
$\phi=\{\ell,\rho\}$ for USAD).

\paragraph{Permutation Null Construction.}
We have previously mentioned that $X$ and $Y$ are exchangeable under $\H_0$ and will construct the empirical null distribution with pooled samples
\begin{equation*}
    Z \triangleq X \cup Y = \{\x_1,\dots,\x_n,\y_1,\dots,\y_m\} = \{ \z_1, \dots, \z_{n+m} \}.
\end{equation*}
Let $\mPi_{n+m}$ be the set of permutations of
$\{1,\dots,n+m\}$.
For each permutation $\mpi = (\pi_1, \dots, \pi_{n+m}) \in \mPi_{n+m}$, we form
\begin{equation*}
    X_{\mpi}=\{\z_{\pi_i}\}_{i=1}^{n} \quad \text{and} \quad Y_{\mpi}=\{\z_{\pi_i}\}_{i=n+1}^{n+m},
\end{equation*}
and define 
\begin{equation*}
    \mathrm{STAT}(\mpi Z;\phi) \triangleq \mathrm{STAT}(X_{\mpi},Y_{\mpi};\phi),
\end{equation*}
where $\mpi Z=\{\z_{\pi_1},\dots,\z_{\pi_n},\z_{\pi_{n+1}},\dots,\z_{\pi_{n+m}}\}$ denotes the shuffled pooled samples.

Since under $\H_0$ the entries of $Z$ are i.i.d., $Z$ is exchangeable: permuting the indices does not change its distribution. 
Hence, under $\H_0$, $\mathrm{STAT}(X,Y;\phi)$ has the same distribution as $\mathrm{STAT}(\mpi Z; \phi)$ for a uniformly random $\mpi \in \mPi_{n+m}$ independent of $Z$.

\paragraph{Monte-Carlo Approximation.}
Enumerating all $(n+m)!$ permutations in $\mPi_{n+m}$ (or, equivalently, all $\binom{n+m}{n}$ distinct splits $(X_{\mpi}, Y_{\mpi})$, on which $\mathrm{STAT}(\mpi Z;\phi)$ depends) is usually impossible, so we approximate the permutation distribution by sampling $R$ permutations. 
For each $r=1,\dots,R$, we
\begin{enumerate}
  \item draw $\mpi_{(r)}$ uniformly from $\mPi_{n+m}$;
  \item construct $X_{(r)} = X_{\boldsymbol{\pi}_{(r)}}$ and
      $Y_{(r)} = Y_{\boldsymbol{\pi}_{(r)}}$;
  \item compute
    \[
      \mathrm{STAT}^{(r)} = \mathrm{STAT}(\mpi_{(r)} Z;\phi) = \mathrm{STAT}(X_{(r)},Y_{(r)};\phi),
    \]
\end{enumerate}
leading to $R$ such statistics, i.e., $\mathrm{STAT}^{(r)}$ for $r = 1, \dots, R$, in total.
Then, the empirical upper-$\alpha$ {\em testing threshold} is defined as



\[
t_\alpha(Z,\{\mpi_{(r)}\}_{r=1}^R) = \underset{t}{\arg\min} \left\{\sum^{R}_{r=1}\frac{\bI[\mathrm{STAT}(\mpi_{(r)} Z;\phi)\leq t]}{R+1}\geq1-\alpha\right\}\ .
\]
This is the empirical $(1 - \alpha)$-quantile $t_\alpha$ as discussed in~\cref{sec:sadd_revisit}, where the denominator $R+1$ accounts for the observed statistic $\mathrm{STAT}(X,Y;\phi)$, which is included among the permutation statistics as the one obtained with the identity permutation~\cite{Hemerik:Goeman2018}. Equivalently, $\H_0$ is rejected if and only if $\big(1+\sum_{r=1}^{R}\bI[\mathrm{STAT}(\mpi_{(r)}Z;\phi)\geq\mathrm{STAT}(X,Y;\phi)]\big)/(R+1)\leq\alpha$. This guarantees Type-I error control at level $\alpha$ for any number of permutations $R$ (see Appendix~\ref{app:perm_proof}); we assume $\alpha(R+1)\geq1$, since otherwise $\H_0$ is never rejected.

\paragraph{Decision Rule.}
Finally, we have the following test with the {\em testing threshold} $t_\alpha$ with observed statistic $\mathrm{STAT}(X,Y;\phi)$ on samples $X$ and $Y$ as follows
\begin{equation}\label{eq:test_pro}
\delta_{\alpha}=\bI\left[\mathrm{STAT}(X,Y;\phi)> t_\alpha(Z,\{\mpi_{(r)}\}_{r=1}^R)\right]\ .
\end{equation}


\section{Maximum Mean Discrepancy}\label{app:mmd}

\cref{sec:sadd_revisit} frames SAD as a two-sample testing~(TST) between the reference clean distribution $\bP$ and a query distribution $\bQ$ of potentially {\em adversarial examples}~(AEs).
The {\em maximum mean discrepancy} (MMD) is a popular kernel two-sample statistic used by existing SAD methods and by the SAMMD baseline evaluated in this paper. 
This section provides a self-contained account of kernel-based TST and of the population and empirical forms of MMD that underlie Eq.\eqref{eq: mmd}.

\smallskip
\noindent
\textbf{Kernel-based Two-sample Testing~(TST)} asks whether two independent samples are generated from the same distribution \cite{Gretton:Borgwardt:Rasch:Scholkopf:Smola2012}.
Formally, let $\cX$ be the input space and $X = \{ \x_i \}_{i=1}^{n} \stackrel{\text{i.i.d.}}{\sim} \bP$ and a query set $Y = \{ \y_j \}_{j=1}^{m} \stackrel{\text{i.i.d.}}{\sim} \bQ$.
The goal of TST is to decide whether $\bP = \bQ$ ({\it cf.} Eq.\eqref{eq:sadd_as_ht}).
When inputs are in high-dimensional spaces (e.g., images), directly estimating the densities of $\bP$ and $\bQ$ is fragile.
Kernel methods emerge as a widely used class of non-parametric approaches to sidestep this fragility.
These methods map each distribution into a {\em reproducing kernel Hilbert space}~(RKHS) where distributions are represented by their mean embeddings, and differences between distributions become distances between points in that RKHS~\cite{Gretton:Borgwardt:Rasch:Scholkopf:Smola2012,Muandet:Fukumizu:Sriperumbudur:Scholkopf2017}.
By definition, let $\kappa:\cX\times\cX\rightarrow\bR$ be a positive definite kernel with RKHS $\cH_\kappa$ and associated map $\kappa(\cdot, \x) \in \cH_\kappa$.
For any probability measure $\bW$ on $\cX$ with $E_{\z \sim \bW} \left[ \sqrt{\kappa(\z, \z)} \right] < \infty$, its kernel mean embedding is 
\begin{equation*}
    \mmu_{\bW} = E_{\z \sim \bW} \left[ \kappa(\cdot, \z) \right] \in \cH_{\kappa}.
\end{equation*}
Specializing to $\bP$ and $\bQ$ gives
\begin{equation*}
    \mmu_\bP = E_{\x \sim \bP} \left[ \kappa(\cdot, \x) \right] \quad \text{and} \quad \mmu_\bQ = E_{\y \sim \bQ} \left[ \kappa(\cdot, \y) \right].
\end{equation*}
If $\kappa$ is {\em characteristic}, this embedding is injective~\cite{Biggs:Schrab:Gretton2023}:
\begin{equation*}
    \mmu_{\bP} = \mmu_{\bQ} \;\; \textit{ if and only if } \;\; \bP = \bQ.
\end{equation*}
Kernel-based TST, therefore, rephrases the hypotheses as 
\begin{equation}\label{eq:hype_tst}
\H^{\textnormal M}_0: \mmu_\bP=\mmu_\bQ\ \ \ \text{and}\ \ \ \H^{\textnormal M}_1:\mmu_\bP\neq\mmu_\bQ\ ,
\end{equation}
and bases the test on a norm of $\mmu_{\bP} - \mmu_{\bQ}$.
Specifically, the {\em maximum mean discrepancy}~(MMD) statistic~\cite{Gretton:Borgwardt:Rasch:Scholkopf:Smola2012} uses the squared RKHS distance (i.e., $\ell_2$ norm) between these mean embeddings, which admits a simple expression in terms of kernel evaluations.

\smallskip
\noindent

\smallskip
\noindent\textbf{Maximum Mean Discrepancy (MMD)}~\cite{Gretton:Borgwardt:Rasch:Scholkopf:Smola2012} is 
a popular kernel-based metric that assesses the kernel mean embedding equality with kernel $\kappa$:
\begin{equation*}
    \mathrm{MMD}^2(\bP,\bQ,\kappa) = \|\mmu_\bP-\mmu_\bQ\|_{\cH_\kappa}^2.
\end{equation*}
Using the reproducing property $\langle \kappa(\cdot, \x), \kappa(\cdot, \x^\prime) \rangle_{\cH_\kappa} = \kappa(\x, \x^{\prime})$ and expanding the square gives us the following:
\begin{equation*}
    \mathrm{MMD}^2(\bP, \bQ, \kappa) = E_{\x, \x^\prime \stackrel{\text{i.i.d.}}{\sim} \bP} \left[ \kappa(\x, \x^\prime) \right] + E_{\y, \y^\prime \stackrel{\text{i.i.d.}}{\sim} \bQ} \left[ \kappa(\y, \y^\prime) \right] - 2 E_{\x \sim \bP, \y \sim \bQ} \left[ \kappa(\x, \y) \right].
\end{equation*}
For a characteristic kernel, $\mathrm{MMD}^2(\bP, \bQ, \kappa) = 0$ if and only if $\bP = \bQ$~\cite{Gretton:Borgwardt:Rasch:Scholkopf:Smola2012, Biggs:Schrab:Gretton2023}.

\smallskip
\noindent
\textbf{Empirical Estimator.}
In practice, the distributions $\bP$ and $\bQ$ are rarely available in closed form, as well as the population embeddings $\mmu_{\bP}$ and $\mmu_{\bQ}$.
For the samples $X$ and $Y$, we can construct an unbiased estimator of $\mathrm{MMD}^2(\bP, \bQ, \kappa)$ by the U-statistic:
\begin{equation*}
\widehat{\textnormal{MMD}}^2(X,Y;\kappa)= \sum_{i\neq j}^n\frac{\kappa(\x_i,\x_j)}{n(n-1)}-\sum_{i=1}^{n}\sum_{j=1}^{m}\frac{2\kappa(\x_i,\y_j)}{nm}+\sum_{i\neq j}^m\frac{\kappa(\y_i,\y_j)}{m(m-1)}\ .
\end{equation*}
Under the sampling $\x_i \stackrel{\text{i.i.d.}}{\sim} \bP$ and $\y_i \stackrel{\text{i.i.d.}}{\sim} \bQ$, this estimator satisfies $E\left[ \widehat{\textnormal{MMD}}^2(X,Y;\kappa)\right] = \mathrm{MMD}^2(\bP, \bQ, \kappa)$~\cite{Gretton:Borgwardt:Rasch:Scholkopf:Smola2012}.

In the main text, we denote the population quantity by $\mathrm{MMD}^2(\bP, \bQ, \kappa)$ (Eq.\eqref{eq: mmd}).
When MMD is used as a SAD {\em test statistic}, its empirical estimate $\widehat{\textnormal{MMD}}^2(X,Y;\kappa)$ plays the role of $\cT(X, Y)$ in~\cref{sec:sadd_revisit}.
Then, a permutation test~(detailed in \cref{app:perm}) is applied to $\widehat{\textnormal{MMD}}^2(X,Y;\kappa)$ to obtain a threshold $t_\alpha$ such that $\Pr(\cT(X, Y) > t_\alpha \mid \H_0) \approx \alpha$.

\smallskip
\noindent
\textbf{MMD on Semantic Features.}
In existing SAD practices~\cite{Gao:Liu:Zhang:Han:Liu:Niu:Sugiyama2021, zhang2025one}, MMD-based SAD methods do not work directly on raw pixels, but on semantic features $f(\x) \in \cF$ typically extracted from the penultimate layer of the threat classifier (i.e., the classifier under attack).
This is equivalent to composing the kernel with the feature map $\kappa (f(\x), f(\x^\prime))$
and computing $\widehat{\textnormal{MMD}}^2(X,Y;\kappa)$.
The SAMMD baseline used in our experiments applies this feature-space MMD, together with calibrated thresholds, to detect adversarial deviations; the mathematical definitions above remain unchanged under this composition.
\setlength{\parindent}{0pt}
\section{Proofs of Theoretical Analysis}\label{app:proof}
We provide here the proofs of the theoretical results presented in the main text.  
We begin by presenting several foundational concepts and concentration results for $U$-statistics, which serve as key tools in our analysis.
\begin{definition}\label{def:U_sta}\cite{Kim:Balakrishnan:Wasserman2022}
Let $g(\x,\y)$ be a bivariate function, which is symmetric in its arguments, i.e., $g(\x,\y)=g(\y,\x)$. We define a function for a two-sample $U$-statistic as $h(\x_1,\x_2;\y_1,\y_2)$, which is composed of evaluations of $g$ over the variables $\x_1,\x_2,\y_1$ and $\y_2$, and is symmetric, i.e., $h(\x_1,\x_2;\y_1,\y_2)=h(\x_2,\x_1;\y_2,\y_1)$.
Suppose we have two samples $X=\{\x_i\}_{i=1}^n\sim\bP^n$ and $Y=\{\y_j\}_{j=1}^m\sim\bQ^m$. The two-sample U-statistic is given by:
\[
U_{n,m}(X,Y;\phi) = \binom{n}{2}^{-1}\binom{m}{2}^{-1} \sum_{i<j}^n\sum_{k<l}^m h(\x_i,\x_j;\y_k,\y_l)\ .
\]
\end{definition}

Next, we present a standard large-deviation bound for bounded $U$-statistics.
\begin{theorem}\label{thm:larg_U}\cite{Hoeffding1994}
    If the function $h$ is bounded, that is, $\|h\|_\infty\triangleq\sup|h(\x_1,\x_2;\y_1,\y_2)|<\infty$, and assume $n \leq m$, then for any $t > 0$,
    \[
    \Pr(|U_{n,m}(X,Y;\phi)-\theta|\geq t)\leq2\exp{(-\lfloor n/2\rfloor t^2/(2\|h\|_\infty^2))}\ ,
    \]
    where $\theta=E[h(\x_1,\x_2;\y_1,\y_2)]$.
\end{theorem}

We now introduce the notation for permutation-based statistics.  
Let $\mPi_{n+m}$ denote the set of all possible permutations of $\{1,\dots,n,n+1,\dots,n+m\}$ over the pooled sample $Z=\{\x_1,\dots,\x_n,\y_1,\dots,\y_m\}=\{\z_1,\dots,\z_n,\z_{n+1},\dots,\z_{n+m}\}$.  
Given a permutation $\mpi=(\pi_1,\allowbreak\dots,\allowbreak\pi_n,\allowbreak\pi_{n+1},\allowbreak\dots,\allowbreak\pi_{n+m}) \in \mPi_{n+m}$, define $X_{\mpi}=\{\z_{\pi_i}\}_{i=1}^{n}$ and $Y_{\mpi}=\{\z_{\pi_i}\}_{i=n+1}^{n+m}$.

We next state a useful concentration inequality for permuted $U$-statistics, which will be repeatedly used in the proofs below.
\begin{theorem}\label{thm:6.1} \citep[Theorem 6.1]{Kim:Balakrishnan:Wasserman2022}
Consider the permuted two-sample $U$-statistic $U_{n,m}(X_{\mpi},Y_{\mpi};\phi)$ with $h(\x_1,\x_2;\y_1,\y_2)=g(\x_1,\x_2)+g(\y_1,\y_2)-g(\x_1,\y_2)-g(\x_2,\y_1)$ in Definition~\ref{def:U_sta}, where $\mpi$ is uniformly distributed on $\mPi_{n+m}$, and assume $n \leq m$. Define
\[
\Gamma_{n,m}^2:=\frac{1}{n^2\left(n-1\right)^2} \sup _{\mpi \in \mPi_{n+m}}\left\{\sum_{i<j}^ng^2\left(\z_{\pi_i}, \z_{\pi_j}\right)\right\} .
\]
Then, for every $t>0$ and some constant $C>0$, we have
\[
\Pr\left(U_{n,m}(X_{\mpi},Y_{\mpi};\phi) \geq t \mid Z\right) \leq \exp \left(-C \min \left(\frac{t^2}{\Gamma_{n,m}^2}, \frac{t}{\Gamma_{n,m}}\right)\right)\ .
\]
\end{theorem}

\subsection{Detailed Proofs of Theorem~\ref{thm:perm}}\label{app:perm_proof}
We begin with a useful definition as follows.
\begin{definition}\label{Def:invariant}\citep{Hemerik:Goeman2018} 
Let $Z$ be the pooled sample, and let $\mPi_{n+m}$ be the finite set of permutations $\mpi$ acting on $Z$, such that $\mPi_{n+m}$ is a group with respect to the operation of composition of transformations. Let $\H_0$ be any null hypothesis which implies that the joint distribution of the test statistics $\textnormal{STAT}(\mpi Z;\phi)$, $\mpi\in\mPi_{n+m}$, is invariant under all transformations in $\mPi_{n+m}$ of $Z$. We have, under $\H_0$,
\[
\big(\textnormal{STAT}(\mpi'Z;\phi)\big)_{\mpi'\in\mPi_{n+m}}\stackrel{d}{=}\big(\textnormal{STAT}(\mpi\cdot\mpi'Z;\phi)\big)_{\mpi'\in\mPi_{n+m}} \ \ \ \text{for all}\ \mpi\in\mPi_{n+m}\ ,
\]
where $\stackrel{d}{=}$ denotes equality in distribution.
\end{definition}
We present the proofs of Theorem~\ref{thm:perm} as follows

\begin{proof} It is easy to see that the permutation set $\mPi_{n+m}$ forms a group under the operation of composition of transformations. Moreover, when the test batch $Y$ consists of clean data, $X$ and $Y$ are drawn from the same distribution, hence $Z$ is exchangeable w.r.t.\ $\mPi_{n+m}$ and the joint law of $\{\textnormal{STAT}(\mpi Z;\phi):\mpi\in\mPi_{n+m}\}$ is invariant under permutation by any $\mpi\in\mPi_{n+m}$. Therefore Definition~\ref{Def:invariant} applies.
For USAD, $\widehat{\Sigma}_{\rm A}$ is estimated from CEs before the test and kept fixed across permutations (\cref{sec: dual}); as long as it is estimated independently of the tested samples $X$ and $Y$, the arguments below hold conditionally on $\widehat{\Sigma}_{\rm A}$, under which $\cT^{\rm A}(\cdot,\cdot;\ell,\rho)$ is a fixed function of the two samples.
Given the pooled sample $Z$, we write the multiset
\[
\{\textnormal{STAT}(\mpi Z;\phi):\mpi\in\mPi_{n+m}\}
\]
of $(n+m)!$ statistics, and note that the identity permutation is included, so $\textnormal{STAT}(Z;\phi)=\textnormal{STAT}(X,Y;\phi)$ is one element of this multiset. By Definition~\ref{Def:invariant}, for a uniformly distributed $\mpi\sim\mathrm{Unif}(\mPi_{n+m})$ independent of $Z$,
\[
\big(\textnormal{STAT}(\mpi'Z;\phi)\big)_{\mpi'\in\mPi_{n+m}}\ \stackrel{d}{=}\ \big(\textnormal{STAT}(\mpi\cdot\mpi'Z;\phi)\big)_{\mpi'\in\mPi_{n+m}}.
\]
Hence, the rank of $\textnormal{STAT}(Z;\phi)$ in this multiset has the same distribution as the rank of $\textnormal{STAT}(\mpi Z;\phi)$ in it, which is uniform on $\{1,\dots,(n+m)!\}$ if there are no ties; with ties, using the conservative convention (counting ties as exceedances) increases the right tail exceedance count, which makes rejection less likely, keeping the type I error conservative. In practical implementation (see Appendix~\ref{app:perm}), we approximate the permutation distribution by $R<(n+m)!$ random permutations $\mpi_{(1)},\dots,\mpi_{(R)}\overset{i.i.d.}{\sim}\mathrm{Unif}(\mPi_{n+m})$, drawn independently of $Z$, and use the testing threshold $t_\alpha(Z,\{\mpi_{(r)}\}_{r=1}^R)$ in Appendix~\ref{app:perm}.
By the same exchangeability argument, the $R+1$ statistics $\textnormal{STAT}(X,Y;\phi),\textnormal{STAT}(\mpi_{(1)}Z;\phi),\dots,\textnormal{STAT}(\mpi_{(R)}Z;\phi)$ are exchangeable under $\H_0$~\citep{Hemerik:Goeman2018}, so the rank of $\textnormal{STAT}(X,Y;\phi)$ among them is uniform on $\{1,\dots,R+1\}$ (conservative under ties).
By the definition of $t_\alpha(Z,\{\mpi_{(r)}\}_{r=1}^R)$, $\textnormal{STAT}(X,Y;\phi)>t_\alpha(Z,\{\mpi_{(r)}\}_{r=1}^R)$ holds only if at least $\lceil(1-\alpha)(R+1)\rceil$ of the $R$ permuted statistics are smaller than $\textnormal{STAT}(X,Y;\phi)$, i.e., only if its rank among the $R+1$ statistics exceeds $\lceil(1-\alpha)(R+1)\rceil$. Hence
\[
\Pr(\delta_\alpha=1)\equiv\Pr(\textnormal{STAT}(X,Y;\phi)>t_\alpha(Z,\{\mpi_{(r)}\}_{r=1}^R))\ \leq\ \frac{R+1-\lceil(1-\alpha)(R+1)\rceil}{R+1}\ \leq\ \alpha.
\]
\end{proof}

\subsection{Detailed Proofs of Theorem~\ref{thm:VD_test}}\label{app:VD_test}
We present the proofs of Theorem~\ref{thm:VD_test} as follows
\begin{proof}
Recall the definition of VD as
\[
\textnormal{VD}(\bP,\bQ;\ell) = \left(V_{\bP}-V_{\bQ}\right)^2\ .
\]
The corresponding estimator is
\begin{eqnarray*}
\widehat{\textnormal{VD}}(X,Y;\ell) &=& \left(\widehat{V}(X;\ell)-\widehat{V}(Y;\ell)\right)^2\ ,
\end{eqnarray*}
where
\begin{eqnarray*}
\widehat{V}(X;\ell) = \sum_{i}^n\frac{\ell(f(\x_i),f(\x_i))}{n}-\sum_{i\neq j}^n\frac{\ell(f(\x_i),f(\x_j))}{n(n-1)}\ ,\\
\widehat{V}(Y;\ell) = \sum_{i}^m\frac{\ell(f(\y_i),f(\y_i))}{m}-\sum_{i\neq j}^m\frac{\ell(f(\y_i),f(\y_j))}{m(m-1)}\ .
\end{eqnarray*}
These two estimators can be equivalently written as second-order $U$-statistics:
\begin{eqnarray*}
\widehat{V}(X;\ell) = \binom{n}{2}^{-1}\sum_{i<j}^n\frac{\ell(f(\x_i),f(\x_i))+\ell(f(\x_j),f(\x_j))}{2}-\ell(f(\x_i),f(\x_j))\ ,\\
\widehat{V}(Y;\ell) = \binom{m}{2}^{-1}\sum_{i<j}^m\frac{\ell(f(\y_i),f(\y_i))+\ell(f(\y_j),f(\y_j))}{2}-\ell(f(\y_i),f(\y_j))\ .
\end{eqnarray*}
Assume $\ell$ is bounded on the range of $f$, i.e., $\|\ell\|_\infty\triangleq\sup_{\x,\y\in\cX}|\ell(f(\x),f(\y))|<\infty$. Then
$|V_{\bP}|,|V_{\bQ}|,|\widehat V(X;\ell)|,|\widehat V(Y;\ell)|\le 2\|\ell\|_\infty$, and hence
\begin{eqnarray}\label{eq:split}
\lefteqn{\left|\textnormal{VD}(\bP,\bQ;\ell)-\widehat{\textnormal{VD}}(X,Y;\ell)\right|}\nonumber\\
&=&\left|\left(V_{\bP}-V_{\bQ}\right)^2-\left(\widehat{V}(X;\ell)-\widehat{V}(Y;\ell)\right)^2\right|\nonumber\\
&=&\left(\widehat{V}(X;\ell)-\widehat{V}(Y;\ell)-\left(V_{\bP}-V_{\bQ}\right)\right)\left(\widehat{V}(X;\ell)-\widehat{V}(Y;\ell)+\left(V_{\bP}-V_{\bQ}\right)\right)\nonumber\\
&\leq&\left|\widehat{V}(X;\ell)-\widehat{V}(Y;\ell)-\left(V_{\bP}-V_{\bQ}\right)\right| \cdot \left|\widehat{V}(X;\ell)-\widehat{V}(Y;\ell)+\left(V_{\bP}-V_{\bQ}\right)\right|\nonumber\\
&\leq&8\|\ell\|_\infty\cdot\left|\widehat{V}(X;\ell)-\widehat{V}(Y;\ell)-\left(V_{\bP}-V_{\bQ}\right)\right|\ .
\end{eqnarray}

Here, the estimator of $V_{\bP}-V_{\bQ}$ can be further presented as a two-sample $U$-statistic as in Definition~\ref{def:U_sta}
\[
\widehat{V}(X;\ell)-\widehat{V}(Y;\ell)=\binom{n}{2}^{-1}\binom{m}{2}^{-1} \sum_{i<j}^n\sum_{k<l}^m h_{\textnormal{VD}}(\x_i,\x_j;\y_k,\y_l)\ ,
\]
where 
\begin{eqnarray*}
h_{\textnormal{VD}}(\x_i,\x_j;\y_k,\y_l) &=& \frac{\ell(f(\x_i),f(\x_i))+\ell(f(\x_j),f(\x_j))}{2}-\ell(f(\x_i),f(\x_j))\\
&&\qquad\qquad-\frac{\ell(f(\y_k),f(\y_k))+\ell(f(\y_l),f(\y_l))}{2}+\ell(f(\y_k),f(\y_l))\ .
\end{eqnarray*}
Under the alternative hypothesis $\H^{\textnormal V}_1:V_{\bP}\neq V_{\bQ}$, we have that
\[
\widehat{V}(X;\ell)-\widehat{V}(Y;\ell)\xrightarrow{P} V_{\bP}-V_{\bQ} \neq 0\ ,
\]
by Theorem~\ref{thm:larg_U} (note that $\|h_{\textnormal{VD}}\|_\infty\leq4\|\ell\|_\infty$), where $\xrightarrow{P}$ denotes convergence in probability, and hence
\[
\widehat{\textnormal{VD}}(X,Y;\ell)\equiv(\widehat{V}(X;\ell)-\widehat{V}(Y;\ell))^2 \xrightarrow{P} (V_{\bP}-V_{\bQ})^2 > 0\ ,
\]
by Eq.~\eqref{eq:split}.

Under a uniformly random permutation $\mpi\in\mPi_{n+m}$, conditionally on $Z$, $X_{\mpi}$ is a sample of size $n$ drawn without replacement from $Z$, and $Y_{\mpi}$ consists of the remaining $m$ examples.
Writing $g(\z_i,\z_j)=\frac{\ell(f(\z_i),f(\z_i))+\ell(f(\z_j),f(\z_j))}{2}-\ell(f(\z_i),f(\z_j))$, so that $h_{\textnormal{VD}}(\x_i,\x_j;\y_k,\y_l)=g(\x_i,\x_j)-g(\y_k,\y_l)$ and $|g|\leq2\|\ell\|_\infty$, both $\widehat{V}(X_{\mpi};\ell)$ and $\widehat{V}(Y_{\mpi};\ell)$ are averages of $g$ over the pairs of examples within $X_{\mpi}$ and within $Y_{\mpi}$, respectively, and thus have the same conditional mean $\binom{n+m}{2}^{-1}\sum_{i<j}^{n+m}g(\z_i,\z_j)$.
Moreover, among the pairs of summands of $\widehat{V}(X_{\mpi};\ell)$, the $O(n^3)$ pairs that share an example have covariances bounded by $4\|\ell\|_\infty^2$, whereas the remaining pairs involve four distinct examples drawn without replacement from $Z$ and have covariances of order $O(\|\ell\|_\infty^2/(n+m))$; hence the conditional variance of $\widehat{V}(X_{\mpi};\ell)$ is $O(\|\ell\|_\infty^2/n)$, and likewise that of $\widehat{V}(Y_{\mpi};\ell)$ is $O(\|\ell\|_\infty^2/m)$.
By Chebyshev’s inequality, the two-sample $U$-statistic with permuted samples $X_{\mpi}$ and $Y_{\mpi}$ exhibits the convergence that 
\[
\widehat{V}(X_{\mpi};\ell)-\widehat{V}(Y_{\mpi};\ell)\xrightarrow{P}0\ ,
\]
as $n,m\rightarrow\infty$ with $n/m\rightarrow c\in(0,\infty)$. Consequently, the permutation distribution of $\widehat{V}(X_{\mpi};\ell)-\widehat{V}(Y_{\mpi};\ell)$ concentrates at zero and 
\[
\widehat{\textnormal{VD}}(X_{\mpi},Y_{\mpi};\ell)=(\widehat{V}(X_{\mpi};\ell)-\widehat{V}(Y_{\mpi};\ell))^2\xrightarrow{P}0\ .
\]
Since $R$ is fixed and $\alpha(R+1)\geq1$, the testing threshold $t_\alpha(Z,\{\mpi_{(r)}\}_{r=1}^R)$ (see Appendix~\ref{app:perm} with $\textnormal{STAT}(X,Y;\phi) = \widehat{\textnormal{VD}}(X,Y;\ell)$) lies between the smallest and the largest of $\widehat{\textnormal{VD}}(X_{(r)},Y_{(r)};\ell)$, $r=1,\dots,R$, and hence also converges to $0$ in probability,
which implies, by Slutsky’s theorem~\cite{Papoulis:Pillai2001},
\[ 
\Pr(\delta_\alpha=1)\equiv\Pr(\widehat{\textnormal{VD}}(X,Y;\ell)>t_\alpha(Z,\{\mpi_{(r)}\}_{r=1}^R))\ \longrightarrow\ 1\ , 
\]
where the test $\delta_\alpha$ is the permutation procedure (see Eq.~\eqref{eq:test_pro} in Appendix~\ref{app:perm} with $\textnormal{STAT}(X,Y;\phi) = \widehat{\textnormal{VD}}(X,Y;\ell)$).
\end{proof}

\subsection{Detailed Proofs of Theorem~\ref{thm:PCD_test}}\label{app:PCD_test}
We present the proofs of Theorem~\ref{thm:PCD_test} as follows
\begin{proof}
Recall the definition of PCD as
\[
\textnormal{PCD}(\bP,\bQ,\rho)= \|\mmu_{\Sigma_\bP}-\mmu_{\Sigma_\bQ}\|_{\cH_\rho}^2=E[\rho(\Sigma_{\x},\Sigma_{\x'})+\rho(\Sigma_{\y},\Sigma_{\y'})-2\rho(\Sigma_{\x},\Sigma_{\y})]\ .
\]
The corresponding estimator is
\[
\widehat{\textnormal{PCD}}(X,Y;\rho)=\sum_{i\neq j}^n\frac{\rho(\Sigma_{\x_i},\Sigma_{\x_j})}{n(n-1)}-\sum_{i}^n\sum_{j}^m\frac{2\rho(\Sigma_{\x_i},\Sigma_{\y_j})}{nm}+\sum_{i\neq j}^m\frac{\rho(\Sigma_{\y_i},\Sigma_{\y_j})}{m(m-1)}\ .
\]
which can be further presented as a two-sample $U$-statistic as in Definition~\ref{def:U_sta}
\[
\widehat{\textnormal{PCD}}(X,Y;\rho)=\binom{n}{2}^{-1}\binom{m}{2}^{-1} \sum_{i<j}^n\sum_{k<l}^m h_{\textnormal{PCD}}(\x_i,\x_j;\y_k,\y_l)\ ,
\]
where 
\[
h_{\textnormal{PCD}}(\x_i,\x_j;\y_k,\y_l) = \rho(\Sigma_{\x_i},\Sigma_{\x_j})+\rho(\Sigma_{\y_k},\Sigma_{\y_l})-\rho(\Sigma_{\x_i},\Sigma_{\y_l})-\rho(\Sigma_{\x_j},\Sigma_{\y_k})\ .
\]
Under the alternative hypothesis $\H^{\textnormal C}_1: \mmu_{\Sigma_\bP}\neq\mmu_{\Sigma_\bQ}$, we have that
\[
\widehat{\textnormal{PCD}}(X,Y;\rho)\xrightarrow{P} \textnormal{PCD}(\bP,\bQ,\rho) > 0\ ,
\]
by Theorem~\ref{thm:larg_U}, since $\|h_{\textnormal{PCD}}\|_\infty\leq4\|\rho\|_\infty$ with $\|\rho\|_\infty\triangleq\sup|\rho(\cdot,\cdot)|<\infty$.

Under a uniformly random permutation $\mpi\in\mPi_{n+m}$, the two-sample $U$-statistic with permuted samples $X_{\mpi}$ and $Y_{\mpi}$ exhibits the convergence that 
\[
\widehat{\textnormal{PCD}}(X_{\mpi},Y_{\mpi};\rho)\xrightarrow{P}0\ ,
\]
based on Theorem~\ref{thm:6.1} as $n,m\rightarrow\infty$ with $n/m\rightarrow c\in(0,\infty)$: $h_{\textnormal{PCD}}$ has the form required there with $g(\z_i,\z_j)=\rho(\Sigma_{\z_i},\Sigma_{\z_j})$, $\Gamma_{n,m}^2\leq\|\rho\|_\infty^2/(2n(n-1))\rightarrow0$ (assuming $n\leq m$ without loss of generality, since $\widehat{\textnormal{PCD}}$ is symmetric in its two samples), and applying Theorem~\ref{thm:6.1} to both $\rho$ and $-\rho$ controls both tails. Consequently, since $R$ is fixed and $\alpha(R+1)\geq1$, the testing threshold $t_\alpha(Z,\{\mpi_{(r)}\}_{r=1}^R)$ (see Appendix~\ref{app:perm} with $\textnormal{STAT}(X,Y;\phi) = \widehat{\textnormal{PCD}}(X,Y;\rho)$), which lies between the smallest and the largest of $\widehat{\textnormal{PCD}}(X_{(r)},Y_{(r)};\rho)$, $r=1,\dots,R$, also converges to $0$ in probability. This implies, by Slutsky’s theorem~\cite{Papoulis:Pillai2001},
\[ 
\Pr(\delta_\alpha=1)\equiv\Pr(\widehat{\textnormal{PCD}}(X,Y;\rho)>t_\alpha(Z,\{\mpi_{(r)}\}_{r=1}^R))\ \longrightarrow\ 1\ , 
\]
where the test $\delta_\alpha$ is the permutation procedure (see Eq.~\eqref{eq:test_pro} in Appendix~\ref{app:perm} with $\textnormal{STAT}(X,Y;\phi) = \widehat{\textnormal{PCD}}(X,Y;\rho)$).
\end{proof}

\subsection{Detailed Proofs of Corollary~\ref{cor:USAD_test}}\label{app:USAD_test}
We present the proofs of Corollary~\ref{cor:USAD_test} as follows
\begin{proof}
As shown in proofs of Theorems~\ref{thm:VD_test} and~\ref{thm:PCD_test}, under a uniformly random permutation $\mpi\in\mPi_{n+m}$, we have 
\[
\widehat{\textnormal{VD}}(X_{\mpi},Y_{\mpi};\ell)=(\widehat{V}(X_{\mpi};\ell)-\widehat{V}(Y_{\mpi};\ell))^2\xrightarrow{P}0\ ,
\]
and
\[
\widehat{\textnormal{PCD}}(X_{\mpi},Y_{\mpi};\rho)\xrightarrow{P}0\ ,
\]
which imply that
\[
\cT(X_{\mpi},Y_{\mpi};\ell,\rho)\equiv(\widehat{\textnormal{VD}}(X_{\mpi},Y_{\mpi};\ell),\widehat{\textnormal{PCD}}(X_{\mpi},Y_{\mpi};\rho))^{\top}\xrightarrow{P}(0,0)^{\top}\
\]
by Slutsky’s theorem~\cite{Papoulis:Pillai2001}.

Then, we have 
\[
\cT^{\rm A}(X_{\mpi},Y_{\mpi};\ell,\rho) = \cT(X_{\mpi},Y_{\mpi};\ell,\rho)^{\top}\widehat{\Sigma}^{-1}_{\textnormal{A}}\cT(X_{\mpi},Y_{\mpi};\ell,\rho)\xrightarrow{P}0\ ,
\]
where $\widehat{\Sigma}_{\rm A}$ is estimated once and kept fixed across permutations (\cref{sec: dual}). Since rescaling $\widehat{\Sigma}_{\rm A}$ by a positive constant rescales the observed and all permuted USAD statistics by the same factor, the test $\delta_\alpha$ does not depend on the scale of $\widehat{\Sigma}_{\rm A}$, and we normalize $\widehat{\Sigma}_{\rm A}$ to have unit trace without loss of generality. As the condition number of $\widehat{\Sigma}_{\rm A}$ remains bounded (Corollary~\ref{cor:USAD_test}), the eigenvalues of $\widehat{\Sigma}_{\rm A}^{-1}$ are then at least $1$ and bounded from above, which gives the convergence above.

Consequently, the permutation distribution of $\cT^{\rm A}(X_{\mpi},Y_{\mpi};\ell,\rho)$ concentrates at zero, and the testing threshold
$t_\alpha\big(Z,\{\mpi_{(r)}\}_{r=1}^R\big)$ (see Appendix~\ref{app:perm}, with $\textnormal{STAT}(X,Y;\phi)=\cT^{\rm A}(X,Y;\ell,\rho)$), which is at most $\max_{1\leq r\leq R}\cT^{\rm A}(X_{(r)},Y_{(r)};\ell,\rho)$, converges to $0$ in probability. Similarly, if \textbf{either or both} of the alternatives $\H^{\textnormal V}_1$ and $\H^{\textnormal C}_1$ hold, then at least one of VD or PCD converges in probability to a positive constant, and thus
$\cT(X,Y;\ell,\rho)$ converges in probability to a nonzero vector. It follows that the observed USAD value
$\cT^{\rm A}(X,Y;\ell,\rho)\geq\|\cT(X,Y;\ell,\rho)\|_2^2$ is bounded away from zero with probability tending to one.

Therefore, by Slutsky’s theorem~\cite{Papoulis:Pillai2001},
\[ 
\Pr(\delta_\alpha=1)\equiv\Pr(\cT^{\rm A}(X,Y;\ell,\rho)>t_\alpha(Z,\{\mpi_{(r)}\}_{r=1}^R))\ \longrightarrow\ 1\ , 
\]
where the test $\delta_\alpha$ is the permutation procedure (see Eq.~\eqref{eq:test_pro} in Appendix~\ref{app:perm} with $\textnormal{STAT}(X,Y;\phi) = \cT^{\rm A}(X,Y;\ell,\rho)$).
\end{proof}
\section{Experimental Details} \label{appendix:exp}
\subsection{Aggregation Strategy}
DUAL~\cite{zhou2025dual} can be extended to aggregate heterogeneous statistics because it operates on their joint correlation structure, without requiring the components to share the same U-statistic form, scaling, or null distribution. In contrast, MMDAgg~\cite{Schrab:Kim:Albert:Laurent:Guedj:Gretton2023} relies on all components having identical U-statistic structures and comparable bootstrap quantiles, making it unsuitable for mixing statistics with different functional forms or null behaviors. MMD-FUSE~\cite{Biggs:Schrab:Gretton2023} restricts aggregation by requiring all inputs to follow the same kernel-based MMD formulation, so that its log-sum-exp fusion remains theoretically calibrated; as a result, it cannot accommodate statistics that differ in scale or structural definition. 
By standardizing each statistic through its empirical covariance under the null and combining them via a correlation-aware quadratic form, DUAL enables principled fusion of heterogeneous statistics that capture complementary aspects of distributional discrepancy.

\subsection{Inference Time Complexity}
For clarity, we first introduce the notations used in the time–complexity analysis.
Let $n$ denote the batch size of clean examples, i.e., $X = \{\x_i\}_{i=1}^n$, and let $m$ denote the batch size of test examples, i.e., $Y = \{\y_j\}_{j=1}^m$.
We use $K$ to denote the number of random perturbations used in PCD, and $q$ to denote the dimension of the semantic feature, as in the main text.
We denote by $C_\ell$ and $C_\rho$ the computational cost of evaluating the kernel functions $\ell$ and $\rho$ once, respectively, and by $C_f$ the cost of extracting semantic features from the target classifier $f$. We denote by $B$ the number of bootstrap resampling repeats used to estimate the correlation $\widehat{\Sigma}_{\mathrm{A}}$, and by $R$ the number of repeats used in the permutation test.

\textbf{Time Complexity of VD Estimator.} For the clean batch $X$, evaluating the kernel over all diagonal and off-diagonal pairs of semantic features requires $\mathcal{O}(n^{2})$ operations. The test batch $Y$ contributes an additional $\mathcal{O}(m^{2})$ operations. With a per-evaluation cost of $C_\ell$, the total kernel computation cost is $\mathcal{O}((n^{2}+m^{2})C_\ell)$. Extracting semantic features from the classifier $f$ adds $\mathcal{O}((n+m)C_f)$, which is typically negligible compared with the quadratic kernel terms.
Overall, the time complexity of VD is
\[
\mathcal{O}\big((n^{2}+m^{2})C_\ell + (n+m)C_f\big).
\]

\textbf{Time Complexity of PCD Estimator.} For two batches $X$ and $Y$, each with $K$ random perturbations per example, extracting semantic features through the classifier $f$ requires $\mathcal{O}((n+m)K C_f).$ Computing the covariance matrix for each example involves $K$ outer-product operations in a $q$-dimensional feature space, resulting in a total cost of $\mathcal{O}((n+m)K q^2).$ After obtaining the covariance matrix for every example, the empirical PCD estimator in Eqn.~\eqref{eq:PCD_hat} requires evaluating the kernel $\rho$ over all pairs of covariances within and across the two batches. This involves $n(n-1)$, $m(m-1)$, and $nm$ kernel computations, yielding a total complexity of $\mathcal{O}((n+m)^{2} C_\rho).$ Overall, the time complexity of PCD is
\[
\mathcal{O}\big((n+m)(KC_f+Kq^2)+(n+m)^{2} C_\rho\big).
\]

\textbf{Time Complexity of USAD Estimator.} For estimating the correlation $\widehat{\Sigma}_{\mathrm{A}}$ (i.e., Eqn.~\eqref{eq:sigma_aggregate}) from $B$ bootstrap replicates incurs an offline cost of 
\[
\mathcal{O}\left(B\left((n^2+m^2)C_\ell + (n+m)^2 C_\rho + (n+m)(K C_f + Kq^2)\right)\right)\ ,
\]
which is dominated by computing VD and PCD $B$ times. Notably, this is a one-off preprocessing cost: once $\widehat{\Sigma}_{\mathrm{A}}$ is estimated, it is fixed and reused in all subsequent permutation tests. 

Given the precomputed $\widehat{\Sigma}_{\mathrm{A}}$, evaluating the aggregated test statistic requires computing VD and PCD only once for the test batch $Y$. Therefore, the inference-time complexity of USAD is
\[
\mathcal{O}\left((n^2+m^2)C_\ell + (n+m)^2 C_\rho + (n+m)(K C_f + Kq^2)\right),
\]
which corresponds exactly to the cost of evaluating VD and PCD on the observed data.

\textbf{Time Complexity of Permutation Test.} During the permutation test, the kernel matrices of $\ell$ and $\rho$ over the pooled examples $X\cup Y$ are computed only once: for PCD, they are already computed when evaluating the statistic on the original batches, while for VD, the permuted statistics additionally require the $nm$ cross-batch kernel values $\ell(f(\x_i),f(\y_j))$, an extra one-off cost of $\mathcal{O}(nmC_\ell)$. 
For each permutation, we simply rearrange the precomputed kernel matrices according to the permuted indices, and then aggregate the corresponding entries to obtain the permuted VD and PCD values. 
No kernel evaluations or covariance computations are repeated. 
Since each permutation only involves indexing and summation over existing matrices, the per-permutation cost is linear in the number of matrix entries.

Therefore, for $R$ permutations, the time complexity of the permutation test is 
$\mathcal{O}\bigl(R(n^2 + m^2)\bigr)$ for VD, 
$\mathcal{O}\bigl(R(n + m)^2\bigr)$ for PCD, 
and $\mathcal{O}\bigl(R(n + m)^2\bigr)$ for USAD, which is dominated by the cost of repeatedly aggregating entries from the precomputed kernel and covariance matrices.

In addition to the complexity analysis, we also report the actual elapsed time in \cref{app:computation_efficiency}.

\subsection{Implementation Details.}
\textbf{Computational Resources.} \label{resources}
We implement all methods using Python 3.9.18 with PyTorch 2.7 and CUDA 12.4 on two platforms. One platform is an NVIDIA RTX 4090 GPU PC with PyTorch framework. Another platform is a High-performance Computer cluster with several NVIDIA A100 GPUs with Pytorch framework. The memory of two platforms are both 64 GB. The storage of disk of two platforms are both over 4 TB. The \textit{CIFAR-10} dataset and \textit{ImageNet-1K} dataset can be downloaded via torchvision or from their official websites. Since our proposed methods do not require training, we directly utilize the full CIFAR-10 test set (10,000 images) and the ImageNet-1K validation set (50,000 images) for evaluation. We employ pre-trained ResNet-18, WideResNet-28, WideResNet-70, ResNet-50, and ViT-B-16 models with weights obtained from torchvision or downloaded checkpoints (provided in our GitHub repository; see the link in the abstract) for feature extraction and classification. 

\textbf{Experimental Setup for Datasets.} We employ two primary datasets: \textit{CIFAR-10} \citep{cifar} and \textit{ImageNet-1K} \citep{deng2009imagenet}. For clean examples, we independently and randomly sample reference examples from each dataset. For adversarial examples in CIFAR-10, we apply each adversarial attack to the entire CIFAR-10 test set (10,000 images) to generate 10,000 adversarial examples. For ImageNet-1K, following \citet{Gao:Liu:Zhang:Han:Liu:Niu:Sugiyama2021}, we independently and randomly select 500 clean examples from the full ImageNet-1K validation set (50,000 images), then apply each adversarial attack to these selected images to generate 500 adversarial examples. For statistical power experiments, we independently and randomly select $|Y|=50$ adversarial examples for testing (except when explicitly varying the number of examples). To assess type-I error rates, we partition both the CIFAR-10 test set and ImageNet-1K validation set into two equal subsets. We designate one subset as the reference set and the other as the input set, then independently and randomly select equal numbers of examples from both subsets to evaluate whether detection methods incorrectly classify clean examples as adversarial.  

\textbf{Experimental Setup for Baselines.} For the single-statistic SAD method SAMMD, we adopt all default experimental settings from \citet{Gao:Liu:Zhang:Han:Liu:Niu:Sugiyama2021}. For the aggregated-statistic two-sample testing methods MMDAgg, MMD-FUSE, and MMD-DUAL, we follow the default experimental settings from \citet{zhou2025dual}, with one modification: we set the number of aggregations to 2 to ensure a fair comparison with our method. While our method can be extended to kernel aggregation, such extensions are beyond the scope of this work on adversarial detection.

\textbf{Experimental Setup for Our Methods.} For PCD, we first apply principal component analysis (PCA) for dimensionality reduction, where the target dimension is automatically determined based on the original feature size. We set the Gaussian perturbation mean to $0$ and standard deviation to $1/255$, with the number of perturbations fixed at $200$ across all experiments. For both PCD and VD, we set the bandwidth parameters as $b_{\text{PCD}} = b_{\text{VD}} = 4\sqrt{2}$. For USAD, we use bandwidth $b_{\text{USAD}} = 13\sqrt{2}$. In the permutation testing procedure, we set the significance level at $\alpha = 0.05$ and the number of permutations to $100$. All reported results are computed over $10$ independent experimental trials with different random seeds, where each trial conducts $100$ sampling iterations to estimate the mean test power. Thus, every result is averaged over $1{,}000$ repetitions in total.
A sensitivity analysis with respect to the number of perturbations $K$ and the kernel bandwidth is reported in Appendix~\ref{app:sensitivity}.

\textbf{Illustrations of  Kernel Types.} 
For VD, we use the Gaussian kernel 
\[
\ell(x, y) = \exp\!\left(-\frac{\|x-y\|^{2}}{b^{2}}\right),
\]
where the bandwidth $b$ is selected using the median heuristic~\cite{Schrab:Kim:Guedj:Gretton2022, Biggs:Schrab:Gretton2023} computed from the clean feature set.

For PCD, we use the log-RBF kernel applied to covariance matrices,
\[
\rho_{\mathrm{log}}(\Sigma_x, \Sigma_y)
= \exp\!\left(-\frac{\|\log \Sigma_x - \log \Sigma_y\|_{F}^{2}}{b^{2}}\right),
\]
or the Gaussian kernel defined directly on covariance matrices,
\[
\rho_{\mathrm{gau}}(\Sigma_x, \Sigma_y)
= \exp\!\left(-\frac{\|\Sigma_x - \Sigma_y\|_{F}^{2}}{b^{2}}\right),
\]
where the bandwidth $b$ is selected using the median heuristic computed from the clean feature set.

\textbf{Implementation Details of Adversarial Attacks (Including Adaptive Attack).} In our experiments, we detect AEs that \emph{successfully} fool the threat model, and all implementation details are designed around this principle. Specifically, we use a diverse set of attack methods under both $\ell_\infty$-norm and $\ell_2$-norm to comprehensively evaluate the detection ability of USAD, including \emph{AutoAttack} (AA) \cite{croce2020reliable}, \emph{basic iterative method} (BIM) \cite{kurakin2017adversarial}, \emph{Carlini \& Wagner} (CW) attack \cite{carlini2017towards}, \emph{fast gradient sign method} (FGSM) \cite{goodfellow2015explaining} and \emph{projected gradient descent} (PGD) \cite{madry2018towards}.
For $\ell_\infty$-norm attacks, following \citet{Gao:Liu:Zhang:Han:Liu:Niu:Sugiyama2021}, we set the maximum allowed perturbation budget $\epsilon$ to range from $1/255$ to $8/255$.
For $\ell_2$-norm attacks, to ensure successful attacks against the threat model, we increase $\epsilon$ over a set of discrete values from $48/255$ to $104/255$, with successive values differing by $8/255$ (i.e., $48/255$, $56/255$, $64/255$, $72/255$, $80/255$, $88/255$, $96/255$ and $104/255$).
For both $\ell_\infty$-norm and $\ell_2$-norm attacks, we set the step size to be $\epsilon/5$.
For all iterative attacks (i.e., except for FGSM), we set the iteration number to 5.

\textbf{Adaptive Attack.} 
To evaluate robustness under adaptive attacks, we follow the strategy of \citet{Gao:Liu:Zhang:Han:Liu:Niu:Sugiyama2021} and generate adversarial examples on CIFAR-10 and ImageNet using a PGD-based white-box attack, where both the classifier and our detector are fully accessible to the attacker. Given the trained classifier $f: \mathcal{X} \to \mathcal{C}$ and a dataset 
$D = \{(\x_i, c_i)\}_{i=1}^{N}$ with $\x_i \in \mathcal{X} \subseteq \mathbb{R}^d$ and ground-truth label $c_i \in \mathcal{C}$, the adaptive adversarial examples
$\tilde{D} = \{\x_i+\mzeta_i^\star\}_{i=1}^{N}$ are produced by solving
\begin{equation}
\begin{aligned}
\{\mzeta_i^\star\}_{i=1}^{N}
= \operatorname*{arg\,max}_{\{\mzeta_i\}_{i=1}^{N}}
\Bigg[
\sum_{i=1}^{N}\mathcal{L}\bigl(f(\x_i+\mzeta_i),c_i\bigr)
-\mathrm{STAT}\bigl(D,\{\x_i+\mzeta_i\}_{i=1}^{N};\phi\bigr)
\Bigg],\\
\textnormal{s.t.}\quad
\|\mzeta_i\|_{\infty}\le \epsilon,\quad \forall i\in[N].
\end{aligned}
\nonumber
\label{eq:adaptive_attack}
\end{equation}
where $\mathcal{L}$ is the attack loss function (e.g., cross-entropy), and 
$\mathrm{STAT}(\cdot)$ denotes the detection statistic used by our method, which can be VD, PCD, or USAD with the corresponding kernel function $\phi$. The associated kernels $\phi \in \{\ell, \rho, \{\ell, \rho\}\}$ correspond respectively to VD, PCD, and USAD.

\section{Additional Experiments}\label{app:more_exp}

This section complements the results reported in~\cref{sec: experiment} with thirteen additional groups of experiments.
In summary, we (i) evaluate the proposed detector under $\ell_2$ norm attacks and compare it against baselines (\cref{app:l2_attack}); (ii) report detailed numerical results across perturbation budgets and additional architectures (\cref{app:detailed_exp}); (iii) evaluate robustness to transfer attacks from surrogate target models (\cref{app:transfer_attack}); (iv) empirically verify Type-I error control (\cref{app:type_1_control}); (v) measure inference-time computational efficiency against competing detectors (Section~\ref{app:computation_efficiency}); (vi) evaluate the proposed detector against sparse $\ell_0$ (Sparse-PGD) and $\ell_1$ adversarial attacks (Section~\ref{app:sparse_attacks}); (vii) extend our evaluation to adversarial video detection on UCF101 (Section~\ref{app:video_detection}); (viii) ablate the correlation-aware aggregation of USAD and incorporate MMD into the framework (Section~\ref{app:ablation}); (ix) summarize the arguments from prior work for the two uncertainty factors (\cref{factor:global_uncertainty} and \cref{factor:local_uncertainty}) and directly analyze them on synthetic data and real AEs (Section~\ref{app:factor_analysis}); (x) analyze how the AE fraction affects test power and study low AE fractions in the query batch together with larger batch sizes (Section~\ref{app:low_fraction}); (xi) evaluate USAD under a benign distribution shift (Section~\ref{app:benign_shift}); (xii) compare VD with MVD and compare with a per-example detector under an adaptive attack (Section~\ref{app:mvd_pfd}); and (xiii) analyze the sensitivity to the kernel bandwidth and the number of perturbations $K$ (Section~\ref{app:sensitivity}). Unless stated otherwise, we follow the hypothesis-testing protocol with the main text: significance level $\alpha = 0.05$ and adversarial batch size $|Y| = 50$.


\begin{figure*}[!t]
\vspace{-0.8em}
\begin{center}
\centerline{\includegraphics[width=\columnwidth]{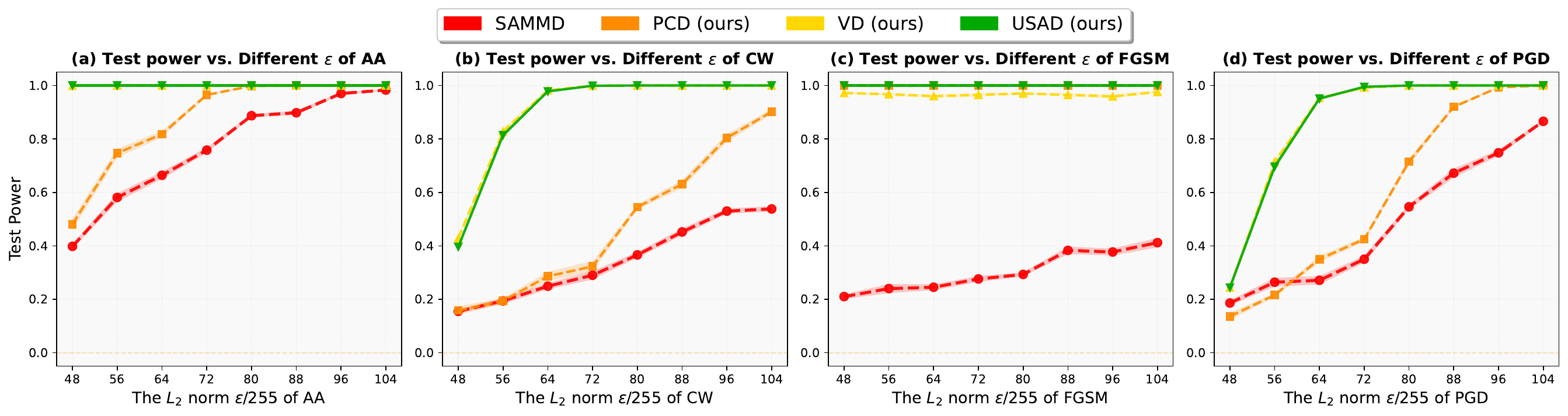}}
\vspace{-1em}
\caption{Results $(a\!-\!d)$ are test power (detection rate) under different adversarial attacks with different $\epsilon$ under $L_2$ norm, given adversarial samples all share the same sample size $|Y| = 50$. The results are averaged over $1,000$ repetitions and the ideal test power is $1$ (the same as $100\%$ detection rate). The target model is ResNet-18 trained on \textit{CIFAR-10} dataset.} 
\label{exp:l2}
\end{center}
\end{figure*}

\subsection{\texorpdfstring{Detection under $L_2$ Norm Attacks}{Detection under L2 Norm Attacks}}\label{app:l2_attack}

The main text (\cref{sec: experiment}) primarily reports the detection performance on $\ell_\infty$-bounded adversarial perturbations, following the previous studies~\cite{Gao:Liu:Zhang:Han:Liu:Niu:Sugiyama2021}.
To check that our detectors are not tied to this particular norm, we also evaluate our method under $\ell_2$-bounded attacks.

Figure~\ref{exp:l2} reports test power for PCD, VD, and USAD, together with the SAMMD baseline, against AA, CW, FGSM, and PGD attacks targeting a ResNet-18 trained on \emph{CIFAR-10}, constrained in $\ell_2$ norm. 
We fix the adversarial batch size to $|Y|=50$ and vary the perturbation budget $\epsilon$, averaging over $1{,}000$ repetitions. 

Across all four attacks, the uncertainty-aware statistics maintain high detection power. 
USAD and PCD achieve near-perfect detection once $\epsilon \ge 64$, while SAMMD remains noticeably weaker at the same budgets. 
For AA and FGSM (panels (a) and (c)), USAD reaches test power essentially equal to one already around $\epsilon \approx 48$, substantially outperforming SAMMD at small perturbation levels. 
For PGD (panel (d)), both USAD and VD exceed $0.95$ test power at $\epsilon \approx 64$, whereas SAMMD requires considerably larger perturbations to approach this regime. 
These results indicate that the statistics underlying USAD capture distributional shifts that manifest under both $\ell_\infty$- and $\ell_2$-bounded attacks.


\begin{figure*}[!t]
\begin{center}
\centerline{\includegraphics[width=\columnwidth]{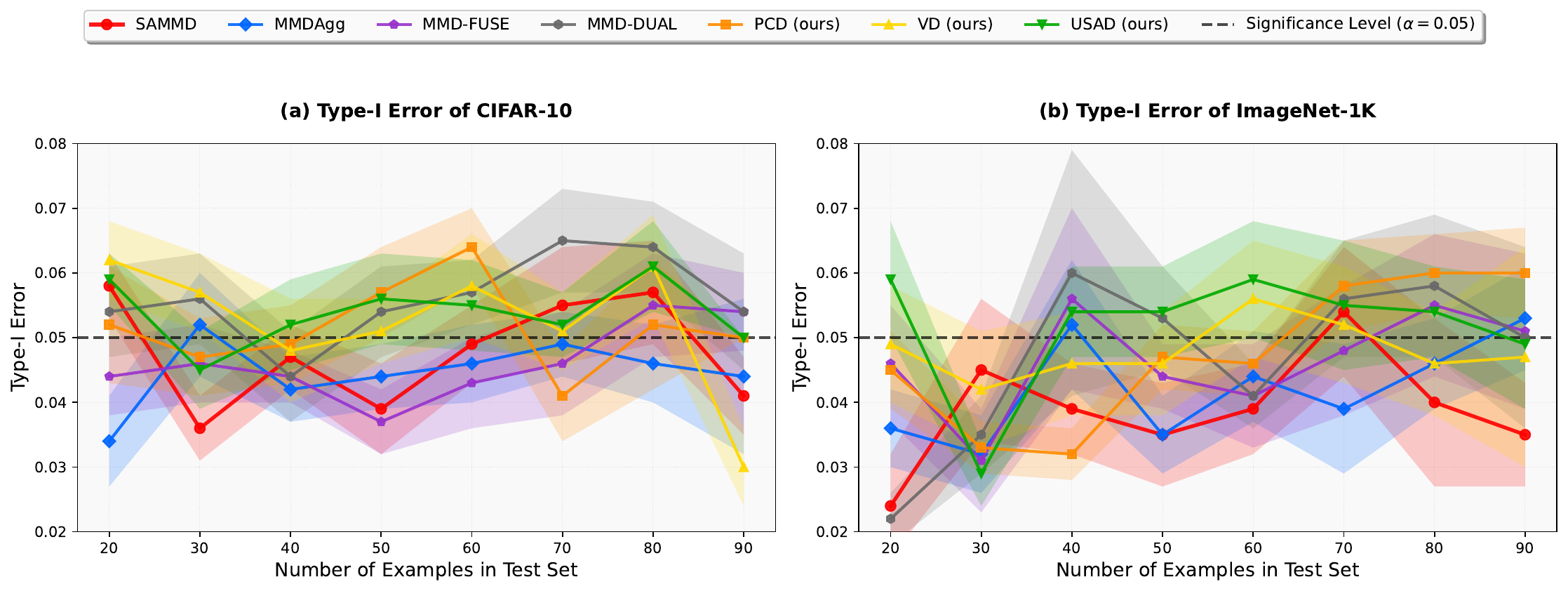}}
\vspace{-1em}
\caption{Results $(a\!-\!b)$ are type I error (false alarm rate) control check where given examples are actually drawn from clean examples under difference clean example sizes. The results are averaged over $1,000$ repetitions and the ideal type I error is around the significance level $\alpha = 0.05$ (only $\alpha$\% of the chance to reject clean examples). The target model is ResNet-18 trained on \textit{CIFAR-10} dataset in $(a)$, and the target model is ResNet-50 trained on \textit{ImageNet-1K} dataset in $(b)$.} 
\label{exp:typeI}
\end{center}
\end{figure*}


\subsection{Detailed Numerical Results} \label{app:detailed_exp} Tables \ref{tab:cifar} and \ref{tab:imagenet} present the detailed numerical results corresponding to Figure \ref{exp:main} in the main paper, providing exact test power values with standard deviations for all evaluated methods across different perturbation budgets $\epsilon \in \{1, 2, \ldots, 8\}$. Table \ref{tab:cifar} corresponds to Figure \ref{exp:main} $(a\!-\!d)$ and reports the detection performance of our proposed methods (PCD, VD, and USAD) against the baseline SAMMD on adversarial examples generated against a ResNet-18 model trained on the \textit{CIFAR-10} dataset. 
Table \ref{tab:imagenet} corresponds to Figure \ref{exp:main} $(e\!-\!h)$ and presents results on adversarial examples generated against a ResNet-50 model trained on the \textit{ImageNet-1K} dataset. The superior performance of our methods is even more pronounced on this larger-scale, more complex dataset.
While the main paper visualizes a subset of attacks for clarity—specifically AA, CW, FGSM, and PGD for CIFAR-10 (Figure \ref{exp:main} $a\!-\!d$), and AA, BIM, CW, and PGD for ImageNet-1K (Figure \ref{exp:main} $e\!-\!h$)—these tables include all five attacks (AA, BIM, CW, FGSM, and PGD) for comprehensive evaluation with exact values. These results indicate that our methods maintain their superiority over the baseline, consistent with the observed trends across all other attacks and datasets.
The standard deviations (computed over 1,000 repetitions with $|Y| = 50$ adversarial examples per trial) are consistently small, particularly for larger perturbation budgets where test power approaches 1.0, indicating the stability and reliability of our detection methods. The boldface values in the tables highlight the best-performing method(s) for each attack and perturbation budget, clearly demonstrating that our proposed methods achieve state-of-the-art detection performance across the vast majority of experimental configurations.

\begin{table}[ht]
\centering
\caption{Test power of dataset \textit{CIFAR-10} under five different attacks with different perturbation budget $\epsilon$. The given adversarial examples all share the same number of examples $|Y|=50$. The results are averaged over $1,000$ repetitions. The target model is ResNet-18 trained on \textit{CIFAR-10} dataset.}
\label{tab:cifar}
\resizebox{\textwidth}{!}{%
\begin{tabular}{llcccccccc}
\toprule
\textbf{Attack} & \textbf{Method} & \multicolumn{8}{c}{\textbf{$\epsilon$}} \\
\cmidrule(lr){3-10}
& & 1 & 2 & 3 & 4 & 5 & 6 & 7 & 8 \\
\midrule
\multirow{4}{*}{AA} 
& SAMMD & 0.160$\pm$0.012 & 0.554$\pm$0.011 & 0.871$\pm$0.005 & 0.986$\pm$0.003 & 0.998$\pm$0.001 & 0.996$\pm$0.002 & \textbf{1.000$\pm$0.000} & \textbf{1.000$\pm$0.000} \\
& PCD (ours) & 0.222$\pm$0.016 & 0.752$\pm$0.012 & 0.997$\pm$0.001 & \textbf{1.000$\pm$0.000} & \textbf{1.000$\pm$0.000} & \textbf{1.000$\pm$0.000} & \textbf{1.000$\pm$0.000} & \textbf{1.000$\pm$0.000} \\
& VD (ours) & 0.261$\pm$0.010 & \textbf{1.000$\pm$0.000} & \textbf{1.000$\pm$0.000} & \textbf{1.000$\pm$0.000} & \textbf{1.000$\pm$0.000} & \textbf{1.000$\pm$0.000} & \textbf{1.000$\pm$0.000} & \textbf{1.000$\pm$0.000} \\
& USAD (ours) & \textbf{0.298}$\pm$0.016 & \textbf{1.000$\pm$0.000} & \textbf{1.000$\pm$0.000} & \textbf{1.000$\pm$0.000} & \textbf{1.000$\pm$0.000} & \textbf{1.000$\pm$0.000} & \textbf{1.000$\pm$0.000} & \textbf{1.000$\pm$0.000} \\
\midrule
\multirow{4}{*}{BIM} 
& SAMMD & 0.148$\pm$0.010 & 0.250$\pm$0.013 & 0.577$\pm$0.011 & 0.792$\pm$0.011 & 0.894$\pm$0.008 & 0.926$\pm$0.008 & 0.973$\pm$0.005 & 0.980$\pm$0.004 \\
& PCD (ours) & 0.125$\pm$0.008 & 0.274$\pm$0.011 & 0.523$\pm$0.019 & 0.971$\pm$0.005 & \textbf{1.000$\pm$0.000} & \textbf{1.000$\pm$0.000} & \textbf{1.000$\pm$0.000} & \textbf{1.000$\pm$0.000} \\
& VD (ours) & 0.211$\pm$0.015 & 0.984$\pm$0.003 & \textbf{1.000$\pm$0.000} & \textbf{1.000$\pm$0.000} & \textbf{1.000$\pm$0.000} & \textbf{1.000$\pm$0.000} & \textbf{1.000$\pm$0.000} & \textbf{1.000$\pm$0.000} \\
& USAD (ours) & \textbf{0.219}$\pm$0.010 & \textbf{0.985}$\pm$0.004 & \textbf{1.000$\pm$0.000} & \textbf{1.000$\pm$0.000} & \textbf{1.000$\pm$0.000} & \textbf{1.000$\pm$0.000} & \textbf{1.000$\pm$0.000} & \textbf{1.000$\pm$0.000} \\
\midrule
\multirow{4}{*}{CW} 
& SAMMD & 0.147$\pm$0.011 & 0.139$\pm$0.006 & 0.300$\pm$0.019 & 0.478$\pm$0.013 & 0.697$\pm$0.013 & 0.849$\pm$0.011 & 0.902$\pm$0.009 & 0.948$\pm$0.005 \\
& PCD (ours) & 0.234$\pm$0.012 & 0.192$\pm$0.013 & 0.862$\pm$0.011 & \textbf{1.000$\pm$0.000} & 0.994$\pm$0.003 & \textbf{1.000$\pm$0.000} & \textbf{1.000$\pm$0.000} & \textbf{1.000$\pm$0.000} \\
& VD (ours) & 0.507$\pm$0.010 & 0.756$\pm$0.010 & \textbf{1.000$\pm$0.000} & \textbf{1.000$\pm$0.000} & \textbf{1.000$\pm$0.000} & \textbf{1.000$\pm$0.000} & \textbf{1.000$\pm$0.000} & \textbf{1.000$\pm$0.000} \\
& USAD (ours) & \textbf{0.528}$\pm$0.015 & \textbf{0.850}$\pm$0.010 & \textbf{1.000$\pm$0.000} & \textbf{1.000$\pm$0.000} & \textbf{1.000$\pm$0.000} & \textbf{1.000$\pm$0.000} & \textbf{1.000$\pm$0.000} & \textbf{1.000$\pm$0.000} \\
\midrule
\multirow{4}{*}{FGSM} 
& SAMMD & 0.190$\pm$0.012 & 0.226$\pm$0.015 & 0.408$\pm$0.019 & 0.682$\pm$0.009 & 0.899$\pm$0.012 & 0.971$\pm$0.005 & 0.997$\pm$0.002 & \textbf{1.000$\pm$0.000} \\
& PCD (ours) & 0.311$\pm$0.014 & 0.303$\pm$0.009 & 0.800$\pm$0.008 & 0.950$\pm$0.003 & \textbf{1.000$\pm$0.000} & \textbf{1.000$\pm$0.000} & \textbf{1.000$\pm$0.000} & \textbf{1.000$\pm$0.000} \\
& VD (ours) & 0.794$\pm$0.016 & 0.835$\pm$0.008 & 0.851$\pm$0.011 & 0.945$\pm$0.006 & 0.992$\pm$0.004 & 0.998$\pm$0.001 & 0.999$\pm$0.001 & \textbf{1.000$\pm$0.000} \\
& USAD (ours) & \textbf{0.810}$\pm$0.013 & \textbf{0.842}$\pm$0.011 & \textbf{0.858}$\pm$0.011 & \textbf{0.987}$\pm$0.002 & \textbf{1.000$\pm$0.000} & \textbf{1.000$\pm$0.000} & \textbf{1.000$\pm$0.000} & \textbf{1.000$\pm$0.000} \\
\midrule
\multirow{4}{*}{PGD} 
& SAMMD & 0.192$\pm$0.012 & 0.200$\pm$0.009 & 0.381$\pm$0.011 & 0.832$\pm$0.009 & 0.829$\pm$0.006 & 0.933$\pm$0.007 & 0.971$\pm$0.006 & 0.986$\pm$0.003 \\
& PCD (ours) & 0.251$\pm$0.012 & 0.257$\pm$0.015 & 0.924$\pm$0.008 & \textbf{1.000$\pm$0.000} & \textbf{1.000$\pm$0.000} & \textbf{1.000$\pm$0.000} & \textbf{1.000$\pm$0.000} & \textbf{1.000$\pm$0.000} \\
& VD (ours) & 0.653$\pm$0.015 & 0.687$\pm$0.018 & \textbf{1.000$\pm$0.000} & \textbf{1.000$\pm$0.000} & \textbf{1.000$\pm$0.000} & \textbf{1.000$\pm$0.000} & \textbf{1.000$\pm$0.000} & \textbf{1.000$\pm$0.000} \\
& USAD (ours) & \textbf{0.663}$\pm$0.014 & \textbf{0.925}$\pm$0.004 & \textbf{1.000$\pm$0.000} & \textbf{1.000$\pm$0.000} & \textbf{1.000$\pm$0.000} & \textbf{1.000$\pm$0.000} & \textbf{1.000$\pm$0.000} & \textbf{1.000$\pm$0.000} \\
\bottomrule
\end{tabular}%
}
\end{table}

\begin{table}[ht]
\centering
\caption{Test power of dataset \textit{ImageNet-1K} under five different attacks with different perturbation budget $\epsilon$. The given adversarial examples all share the same number of examples $|Y|=50$. The results are averaged over $1,000$ repetitions. The target model is ResNet-50 trained on \textit{ImageNet-1K} dataset.}
\label{tab:imagenet}
\resizebox{\textwidth}{!}{%
\begin{tabular}{llcccccccc}
\toprule
\textbf{Attack} & \textbf{Method} & \multicolumn{8}{c}{\textbf{$\epsilon$}} \\
\cmidrule(lr){3-10}
& & 1 & 2 & 3 & 4 & 5 & 6 & 7 & 8 \\
\midrule
\multirow{4}{*}{AA} 
& SAMMD & 0.107$\pm$0.010 & 0.310$\pm$0.017 & 0.765$\pm$0.028 & 0.953$\pm$0.007 & 0.998$\pm$0.001 & \textbf{1.000$\pm$0.000} & \textbf{1.000$\pm$0.000} & \textbf{1.000$\pm$0.000} \\
& PCD (ours) & 0.569$\pm$0.010 & 0.980$\pm$0.004 & \textbf{1.000$\pm$0.000} & \textbf{1.000$\pm$0.000} & \textbf{1.000$\pm$0.000} & \textbf{1.000$\pm$0.000} & \textbf{1.000$\pm$0.000} & \textbf{1.000$\pm$0.000} \\
& VD (ours) & 0.577$\pm$0.035 & 0.984$\pm$0.005 & \textbf{1.000$\pm$0.000} & \textbf{1.000$\pm$0.000} & \textbf{1.000$\pm$0.000} & \textbf{1.000$\pm$0.000} & \textbf{1.000$\pm$0.000} & \textbf{1.000$\pm$0.000} \\
& USAD (ours) & \textbf{0.748}$\pm$0.019 & \textbf{0.998}$\pm$0.002 & \textbf{1.000$\pm$0.000} & \textbf{1.000$\pm$0.000} & \textbf{1.000$\pm$0.000} & \textbf{1.000$\pm$0.000} & \textbf{1.000$\pm$0.000} & \textbf{1.000$\pm$0.000} \\
\midrule
\multirow{4}{*}{BIM} 
& SAMMD & 0.102$\pm$0.012 & 0.158$\pm$0.019 & 0.305$\pm$0.028 & 0.467$\pm$0.029 & 0.674$\pm$0.025 & 0.820$\pm$0.013 & 0.887$\pm$0.016 & 0.970$\pm$0.005 \\
& PCD (ours) & 0.189$\pm$0.016 & 0.505$\pm$0.014 & 0.851$\pm$0.008 & 0.953$\pm$0.006 & 0.987$\pm$0.003 & 0.984$\pm$0.003 & 0.998$\pm$0.001 & 0.993$\pm$0.002 \\
& VD (ours) & 0.194$\pm$0.014 & 0.467$\pm$0.040 & 0.854$\pm$0.016 & 0.940$\pm$0.011 & 0.969$\pm$0.011 & 0.986$\pm$0.005 & 0.995$\pm$0.003 & 0.999$\pm$0.001 \\
& USAD (ours) & \textbf{0.196}$\pm$0.010 & \textbf{0.634}$\pm$0.030 & \textbf{0.961}$\pm$0.008 & \textbf{0.989}$\pm$0.004 & \textbf{0.998}$\pm$0.001 & \textbf{1.000$\pm$0.000} & \textbf{1.000$\pm$0.000} & \textbf{1.000$\pm$0.000} \\
\midrule
\multirow{4}{*}{CW} 
& SAMMD & 0.055$\pm$0.007 & 0.094$\pm$0.007 & 0.262$\pm$0.018 & 0.491$\pm$0.028 & 0.586$\pm$0.025 & 0.775$\pm$0.024 & 0.825$\pm$0.015 & 0.952$\pm$0.011 \\
& PCD (ours) & 0.206$\pm$0.012 & 0.822$\pm$0.009 & 0.841$\pm$0.011 & 0.989$\pm$0.003 & 0.986$\pm$0.003 & 0.988$\pm$0.003 & 0.990$\pm$0.002 & 0.980$\pm$0.003 \\
& VD (ours) & 0.596$\pm$0.022 & 0.972$\pm$0.007 & \textbf{1.000$\pm$0.000} & \textbf{1.000$\pm$0.000} & \textbf{1.000$\pm$0.000} & \textbf{1.000$\pm$0.000} & \textbf{1.000$\pm$0.000} & \textbf{1.000$\pm$0.000} \\
& USAD (ours) & \textbf{0.640}$\pm$0.024 & \textbf{0.980}$\pm$0.003 & \textbf{1.000$\pm$0.000} & \textbf{1.000$\pm$0.000} & \textbf{1.000$\pm$0.000} & \textbf{1.000$\pm$0.000} & \textbf{1.000$\pm$0.000} & \textbf{1.000$\pm$0.000} \\
\midrule
\multirow{4}{*}{FGSM} 
& SAMMD & 0.051$\pm$0.012 & 0.075$\pm$0.014 & 0.187$\pm$0.025 & 0.402$\pm$0.033 & 0.727$\pm$0.025 & 0.940$\pm$0.011 & 0.997$\pm$0.002 & \textbf{1.000}$\pm$0.000 \\
& PCD (ours) & 0.092$\pm$0.008 & 0.210$\pm$0.011 & 0.525$\pm$0.035 & 0.931$\pm$0.011 & 0.999$\pm$0.001 & \textbf{1.000}$\pm$0.000 & \textbf{1.000}$\pm$0.000 & \textbf{1.000}$\pm$0.000 \\
& VD (ours) & 0.043$\pm$0.005 & 0.127$\pm$0.012 & 0.254$\pm$0.015 & 0.516$\pm$0.021 & 0.726$\pm$0.017 & 0.924$\pm$0.011 & 0.967$\pm$0.006 & 0.997$\pm$0.001 \\
& USAD (ours) & \textbf{0.135}$\pm$0.014 & \textbf{0.312}$\pm$0.014 & \textbf{0.620}$\pm$0.025 & \textbf{0.966}$\pm$0.008 & \textbf{1.000}$\pm$0.000 & \textbf{1.000}$\pm$0.000 & \textbf{1.000}$\pm$0.000 & \textbf{1.000}$\pm$0.000 \\
\midrule
\multirow{4}{*}{PGD} 
& SAMMD & 0.090$\pm$0.015 & 0.169$\pm$0.012 & 0.274$\pm$0.031 & 0.450$\pm$0.027 & 0.674$\pm$0.024 & 0.875$\pm$0.010 & 0.957$\pm$0.008 & 0.975$\pm$0.005 \\
& PCD (ours) & 0.152$\pm$0.010 & 0.755$\pm$0.017 & 0.826$\pm$0.017 & 0.939$\pm$0.006 & 0.986$\pm$0.002 & 0.961$\pm$0.004 & 0.992$\pm$0.003 & 0.995$\pm$0.002 \\
& VD (ours) & 0.074$\pm$0.012 & 0.335$\pm$0.020 & 0.799$\pm$0.013 & 0.919$\pm$0.013 & 0.972$\pm$0.007 & 0.987$\pm$0.003 & 0.997$\pm$0.001 & \textbf{1.000$\pm$0.000} \\
& USAD (ours) & \textbf{0.217$\pm$0.013} & \textbf{0.813$\pm$0.016} & \textbf{0.933$\pm$0.008} & \textbf{0.994$\pm$0.003} & \textbf{1.000$\pm$0.000} & \textbf{0.999$\pm$0.001} & \textbf{1.000$\pm$0.000} & \textbf{1.000$\pm$0.000} \\
\bottomrule
\end{tabular}%
}
\end{table}

\begin{table}[ht]
\centering
\caption{Test power of dataset \textit{ImageNet-1K} under five different attacks with different perturbation budget $\epsilon$. The given adversarial examples all share the same number of examples $|Y|=50$. The results are averaged over $1,000$ repetitions. The target model is ViT-B-16 trained on \textit{ImageNet-1K} dataset.}
\label{tab:imagenet-vit}
\resizebox{\textwidth}{!}{%
\begin{tabular}{llcccccccc}
\toprule
\textbf{Attack} & \textbf{Method} & \multicolumn{8}{c}{\textbf{$\epsilon$}} \\
\cmidrule(lr){3-10}
& & 1 & 2 & 3 & 4 & 5 & 6 & 7 & 8 \\
\midrule
\multirow{4}{*}{AA} 
& SAMMD & 0.691$\pm$0.052 & \textbf{1.000$\pm$0.000} & \textbf{1.000$\pm$0.000} & \textbf{1.000$\pm$0.000} & \textbf{1.000$\pm$0.000} & \textbf{1.000$\pm$0.000} & \textbf{1.000$\pm$0.000} & \textbf{1.000$\pm$0.000} \\
& PCD (ours) & \textbf{1.000$\pm$0.000} & \textbf{1.000$\pm$0.000} & \textbf{1.000$\pm$0.000} & \textbf{1.000$\pm$0.000} & \textbf{1.000$\pm$0.000} & \textbf{1.000$\pm$0.000} & \textbf{1.000$\pm$0.000} & \textbf{1.000$\pm$0.000} \\
& VD (ours) & \textbf{1.000$\pm$0.000} & \textbf{1.000$\pm$0.000} & \textbf{1.000$\pm$0.000} & \textbf{1.000$\pm$0.000} & \textbf{1.000$\pm$0.000} & \textbf{1.000$\pm$0.000} & \textbf{1.000$\pm$0.000} & \textbf{1.000$\pm$0.000} \\
& USAD (ours) & \textbf{1.000$\pm$0.000} & \textbf{1.000$\pm$0.000} & \textbf{1.000$\pm$0.000} & \textbf{1.000$\pm$0.000} & \textbf{1.000$\pm$0.000} & \textbf{1.000$\pm$0.000} & \textbf{1.000$\pm$0.000} & \textbf{1.000$\pm$0.000} \\
\midrule
\multirow{4}{*}{BIM} 
& SAMMD & 0.758$\pm$0.110 & 0.846$\pm$0.025 & 0.973$\pm$0.012 & \textbf{1.000$\pm$0.000} & \textbf{1.000$\pm$0.000} & \textbf{1.000$\pm$0.000} & \textbf{1.000$\pm$0.000} & \textbf{1.000$\pm$0.000} \\
& PCD (ours) & \textbf{1.000$\pm$0.000} & \textbf{1.000$\pm$0.000} & \textbf{1.000$\pm$0.000} & \textbf{1.000$\pm$0.000} & \textbf{1.000$\pm$0.000} & \textbf{1.000$\pm$0.000} & \textbf{1.000$\pm$0.000} & \textbf{1.000$\pm$0.000} \\
& VD (ours) & 0.593$\pm$0.030 & \textbf{1.000$\pm$0.000} & \textbf{1.000$\pm$0.000} & \textbf{1.000$\pm$0.000} & \textbf{1.000$\pm$0.000} & \textbf{1.000$\pm$0.000} & \textbf{1.000$\pm$0.000} & \textbf{1.000$\pm$0.000} \\
& USAD (ours) & \textbf{1.000$\pm$0.000} & \textbf{1.000$\pm$0.000} & \textbf{1.000$\pm$0.000} & \textbf{1.000$\pm$0.000} & \textbf{1.000$\pm$0.000} & \textbf{1.000$\pm$0.000} & \textbf{1.000$\pm$0.000} & \textbf{1.000$\pm$0.000} \\
\midrule
\multirow{4}{*}{CW} 
& SAMMD & 0.472$\pm$0.066 & 0.997$\pm$0.003 & \textbf{1.000$\pm$0.000} & \textbf{1.000$\pm$0.000} & \textbf{1.000$\pm$0.000} & \textbf{1.000$\pm$0.000} & \textbf{1.000$\pm$0.000} & \textbf{1.000$\pm$0.000} \\
& PCD (ours) & \textbf{1.000$\pm$0.000} & \textbf{1.000$\pm$0.000} & \textbf{1.000$\pm$0.000} & \textbf{1.000$\pm$0.000} & \textbf{1.000$\pm$0.000} & \textbf{1.000$\pm$0.000} & \textbf{1.000$\pm$0.000} & \textbf{1.000$\pm$0.000} \\
& VD (ours) & 0.999$\pm$0.001 & \textbf{1.000$\pm$0.000} & \textbf{1.000$\pm$0.000} & \textbf{1.000$\pm$0.000} & \textbf{1.000$\pm$0.000} & \textbf{1.000$\pm$0.000} & \textbf{1.000$\pm$0.000} & \textbf{1.000$\pm$0.000} \\
& USAD (ours) & \textbf{1.000$\pm$0.000} & \textbf{1.000$\pm$0.000} & \textbf{1.000$\pm$0.000} & \textbf{1.000$\pm$0.000} & \textbf{1.000$\pm$0.000} & \textbf{1.000$\pm$0.000} & \textbf{1.000$\pm$0.000} & \textbf{1.000$\pm$0.000} \\
\midrule
\multirow{4}{*}{FGSM} 
& SAMMD & 0.732$\pm$0.065 & \textbf{1.000$\pm$0.000} & \textbf{1.000$\pm$0.000} & \textbf{1.000$\pm$0.000} & \textbf{1.000$\pm$0.000} & \textbf{1.000$\pm$0.000} & \textbf{1.000$\pm$0.000} & \textbf{1.000$\pm$0.000} \\
& PCD (ours) & \textbf{1.000$\pm$0.000} & \textbf{1.000$\pm$0.000} & \textbf{1.000$\pm$0.000} & \textbf{1.000$\pm$0.000} & \textbf{1.000$\pm$0.000} & \textbf{1.000$\pm$0.000} & \textbf{1.000$\pm$0.000} & \textbf{1.000$\pm$0.000} \\
& VD (ours) & 0.533$\pm$0.039 & 0.765$\pm$0.040 & 0.811$\pm$0.019 & 0.766$\pm$0.026 & 0.721$\pm$0.033 & 0.655$\pm$0.037 & 0.704$\pm$0.023 & 0.643$\pm$0.019 \\
& USAD (ours) & \textbf{1.000$\pm$0.000} & \textbf{1.000$\pm$0.000} & \textbf{1.000$\pm$0.000} & \textbf{1.000$\pm$0.000} & \textbf{1.000$\pm$0.000} & \textbf{1.000$\pm$0.000} & \textbf{1.000$\pm$0.000} & \textbf{1.000$\pm$0.000} \\
\midrule
\multirow{4}{*}{PGD} 
& SAMMD & 0.545$\pm$0.074 & 0.606$\pm$0.039 & 0.819$\pm$0.023 & 0.949$\pm$0.014 & 0.991$\pm$0.003 & \textbf{1.000$\pm$0.000} & \textbf{1.000$\pm$0.000} & \textbf{1.000$\pm$0.000} \\
& PCD (ours) & \textbf{1.000$\pm$0.000} & \textbf{1.000$\pm$0.000} & \textbf{1.000$\pm$0.000} & \textbf{1.000$\pm$0.000} & \textbf{1.000$\pm$0.000} & \textbf{1.000$\pm$0.000} & \textbf{1.000$\pm$0.000} & \textbf{1.000$\pm$0.000} \\
& VD (ours) & 0.495$\pm$0.042 & 0.988$\pm$0.004 & \textbf{1.000$\pm$0.000} & \textbf{1.000$\pm$0.000} & \textbf{1.000$\pm$0.000} & \textbf{1.000$\pm$0.000} & \textbf{1.000$\pm$0.000} & \textbf{1.000$\pm$0.000} \\
& USAD (ours) & \textbf{1.000$\pm$0.000} & \textbf{1.000$\pm$0.000} & \textbf{1.000$\pm$0.000} & \textbf{1.000$\pm$0.000} & \textbf{1.000$\pm$0.000} & \textbf{1.000$\pm$0.000} & \textbf{1.000$\pm$0.000} & \textbf{1.000$\pm$0.000} \\
\bottomrule
\end{tabular}%
}
\end{table}



\begin{table}[htbp]
\centering
\caption{Test power of dataset \textit{CIFAR-10} under different attacks with perturbation budget $\epsilon=4/255$ across different target models WideResNet-28 (WRN28) or WideResNet-70 (WRN70), when the feature extractor are the same as target model. Both of WRN28 and WRN70 are trained on \textit{CIFAR-10} dataset. The given adversarial examples all share the same number of examples $|Y|=50$. The results are averaged over $1,000$ repetitions.}
\label{tab:threat}
\resizebox{\textwidth}{!}{%
\begin{tabular}{llccccc}
\toprule
\textbf{Target Model} & \textbf{Detection Method} & \textbf{AA} & \textbf{BIM} & \textbf{CW} & \textbf{FGSM} & \textbf{PGD} \\
\midrule
\multirow{7}{*}{WRN28} 
& SAMMD     & 0.748 $\pm$ 0.029 & 0.287 $\pm$ 0.023 & 0.356 $\pm$ 0.013 & 0.497 $\pm$ 0.026 & 0.367 $\pm$ 0.024 \\
&  \cellcolor{orL}{PCD (ours)}    &  \cellcolor{orL}{0.997 $\pm$ 0.002} &  \cellcolor{orL}{0.613 $\pm$ 0.016} &  \cellcolor{orL}{\textbf{0.910 $\pm$ 0.017}} &  \cellcolor{orL}{\textbf{0.999 $\pm$ 0.001}} &  \cellcolor{orL}{\textbf{0.937 $\pm$ 0.014}} \\
& \cellcolor{goL}{VD (ours)}     & \cellcolor{goL}{\textbf{1.000 $\pm$ 0.000}} & \cellcolor{goL}{\textbf{0.737 $\pm$ 0.055}} & \cellcolor{goL}{0.484 $\pm$ 0.105} & \cellcolor{goL}{\textbf{0.999 $\pm$ 0.001}} & \cellcolor{goL}{0.321 $\pm$ 0.094} \\
\cmidrule{2-7}
& MMDAgg    & 0.992 $\pm$ 0.002 & 0.628 $\pm$ 0.015 & 0.610 $\pm$ 0.010 & 0.735 $\pm$ 0.015 & 0.682 $\pm$ 0.012 \\
& MMD-FUSE  & \textbf{1.000 $\pm$ 0.000} & 0.789 $\pm$ 0.013 & 0.742 $\pm$ 0.013 & 0.967 $\pm$ 0.006 & 0.748 $\pm$ 0.012 \\
& MMD-DUAL  & \textbf{1.000 $\pm$ 0.000} & 0.523 $\pm$ 0.018 & 0.938 $\pm$ 0.007 & 0.978 $\pm$ 0.005 & 0.882 $\pm$ 0.010 \\
& \cellcolor{grL}{USAD (ours)}   & \cellcolor{grL}{\textbf{1.000 $\pm$ 0.000}} & \cellcolor{grL}{\textbf{0.860 $\pm$ 0.028}} & \cellcolor{grL}{\textbf{0.953 $\pm$ 0.011}} & \cellcolor{grL}{\textbf{1.000 $\pm$ 0.000}} & \cellcolor{grL}{\textbf{0.959 $\pm$ 0.005}} \\
\midrule
\multirow{7}{*}{WRN70} 
& SAMMD     & 0.943 $\pm$ 0.011 & 0.416 $\pm$ 0.012 & 0.486 $\pm$ 0.010 & 0.500 $\pm$ 0.011 & 0.485 $\pm$ 0.020 \\
&  \cellcolor{orL}{PCD (ours)}    &  \cellcolor{orL}{\textbf{1.000 $\pm$ 0.000}} &  \cellcolor{orL}{0.773 $\pm$ 0.016} &  \cellcolor{orL}{0.930 $\pm$ 0.006} &  \cellcolor{orL}{\textbf{1.000 $\pm$ 0.000}} &  \cellcolor{orL}{\textbf{0.953 $\pm$ 0.007}} \\
& \cellcolor{goL}{VD (ours)}     & \cellcolor{goL}{\textbf{1.000 $\pm$ 0.000}} & \cellcolor{goL}{\textbf{0.938 $\pm$ 0.010}} & \cellcolor{goL}{\textbf{0.983 $\pm$ 0.002}} & \cellcolor{goL}{0.994 $\pm$ 0.002} & \cellcolor{goL}{0.844 $\pm$ 0.011} \\
\cmidrule{2-7}
& MMDAgg    & 0.976 $\pm$ 0.003 & 0.577 $\pm$ 0.010 & 0.623 $\pm$ 0.012 & 0.686 $\pm$ 0.010 & 0.634 $\pm$ 0.012 \\
& MMD-FUSE  & \textbf{1.000 $\pm$ 0.000} & 0.700 $\pm$ 0.019 & 0.778 $\pm$ 0.013 & 0.915 $\pm$ 0.011 & 0.698 $\pm$ 0.014 \\
& MMD-DUAL  & \textbf{1.000 $\pm$ 0.000} & 0.749 $\pm$ 0.016 & 0.837 $\pm$ 0.011 & 0.857 $\pm$ 0.009 & 0.735 $\pm$ 0.016 \\
& \cellcolor{grL}{USAD (ours)}   & \cellcolor{grL}{\textbf{1.000 $\pm$ 0.000}} & \cellcolor{grL}{\textbf{0.956 $\pm$ 0.006}} & \cellcolor{grL}{\textbf{0.992 $\pm$ 0.002}} & \cellcolor{grL}{\textbf{1.000 $\pm$ 0.000}} & \cellcolor{grL}{\textbf{0.982 $\pm$ 0.005}} \\
\bottomrule
\end{tabular}%
}
\end{table}

\subsection{Robustness to Transfer Attacks and Additional Architectures} \label{app:transfer_attack}

We next investigate robustness to transfer attacks, where adversarial examples are generated on a surrogate target model but detected using features from a distinct threat classifier, e.g., the AEs are generated with respect to a Swin Transformer, while the threat model is a ResNet-18.
This setting reflects a realistic threat model in which the defender is unaware of the attacker’s architecture.

\textbf{Transfer Attacks.} Table \ref{tab:transfer} evaluates the detection performance under transfer attack scenarios, where adversarial examples are generated by attacking different target models (ResNet-18 (Res18), WideResNet-28 (WRN28), WideResNet-70 (WRN70), and Shifted window (Swin) Transformer) but detected using features extracted from Res18. Our proposed methods demonstrate exceptional robustness to transfer attacks across all target models and attack types. Notably, both PCD and USAD achieve perfect detection ($1.000 \pm 0.000$) across all five attacks (AA, BIM, CW, FGSM, PGD) and all four target models (Res18, WRN28, WRN70, Swin), indicating near-universal detection capability regardless of the architectural characteristics of the model used to generate adversarial examples. VD similarly achieves perfect or near-perfect detection in the vast majority of scenarios, with only minor degradation on CW attacks against WRN28 and WRN70 ($0.999 \pm 0.001$) and BIM attacks against WRN28 ($0.993 \pm 0.001$), which remains substantially superior to baseline methods. In contrast, the baseline SAMMD exhibits significant vulnerability to transfer attacks, with detection rates often dropping below $0.4$ (AA on WRN28, WRN70 and SWIN) when it performs $0.995$ on Res18 originally. This performance gap suggests that single-statistic method based on maximum mean discrepancy alone may be overly sensitive to the specific characteristics of the attack generation process and fail to capture the fundamental distributional shifts induced by adversarial perturbations across different model architectures.
The aggregated-statistic baselines (MMDAgg, MMD-FUSE, MMD-DUAL) show improved robustness compared to SAMMD, which aligns with the theoretical motivation for aggregation—combining multiple test statistics can provide more comprehensive distributional characterization and reduce sensitivity to specific attack characteristics. However, these methods still fall short of our proposed approaches. For example, on WRN28, MMDAgg achieves only $0.459$ for AA attacks and $0.245$ for BIM attacks, while MMD-FUSE and MMD-DUAL perform better but still exhibit substantial degradation on certain attack-model combinations (e.g., MMD-FUSE: $0.794$ for CW on WRN28, MMD-DUAL: $0.764$ for BIM on WRN28). Our USAD method, which also employs aggregation but with our novel uncertainty-aware test statistics (i.e., VD and PCD), achieves perfect detection across all scenarios, demonstrating that the choice of base statistics is as critical as the aggregation strategy itself.

\textbf{Additional Architectures.} To demonstrate the broader applicability of our detection framework beyond ResNet architectures, we evaluate our methods on vision transformers and wide residual networks. Table \ref{tab:imagenet-vit} presents results on the ViT-B-16 model trained on ImageNet-1K, where our methods achieve perfect or near-perfect detection across all attacks and perturbation budgets. Notably, all three of our proposed methods (PCD, VD, USAD) attain $1.000 \pm 0.000$ test power starting from $\epsilon = 1$ for AA, BIM, CW, and PGD attacks, while SAMMD exhibits substantially lower detection rates at small perturbation budgets (e.g., 0.691 for AA at $\epsilon = 1$, 0.758 for BIM at $\epsilon = 1$). The only exception is VD on FGSM attacks, where detection rates remain moderate across different $\epsilon$ values, suggesting potential architectural-specific sensitivities of certain base statistics to particular attack types on transformer models. Table \ref{tab:threat} further validates our framework on WRN28 and WRN70 architectures trained on CIFAR-10, where adversarial examples are generated by attacking the same model used for feature extraction. Our methods maintain strong performance with USAD achieving detection rates exceeding 0.860 across all attacks on both architectures, while SAMMD struggles with rates often below 0.500. Interestingly, the performance patterns differ across architectures: VD achieves perfect detection on WRN28 for AA attacks but exhibits more variable performance on other attacks, whereas on WRN70, both PCD and VD achieve perfect detection for AA while showing complementary strengths on other attacks. This architectural diversity in the results underscores the importance of aggregation in USAD, which consistently delivers robust detection by combining multiple complementary statistics that capture different aspects of adversarial perturbations across varied model architectures, from convolutional networks to transformers.

\begin{table}[t]
\centering
\caption{Test power of dataset \textit{CIFAR-10} under different transfer attacks with perturbation budget $\epsilon=4/255$ across different target models, while the feature extractor ResNet-18 keeps the same, which is trained on \textit{CIFAR-10} dataset. The given adversarial examples all share the same number of examples $|Y|=50$. The results are averaged over $1,000$ repetitions.}
\resizebox{\textwidth}{!}{%
\begin{tabular}{llccccc}
\toprule
\textbf{Target Model} & \textbf{Detection Method} & \textbf{AA} & \textbf{BIM} & \textbf{CW} & \textbf{FGSM} & \textbf{PGD} \\
\midrule
\multirow{7}{*}{Res18} 
& SAMMD     & $0.995 \pm 0.002$ & $0.689 \pm 0.017$ & $0.434 \pm 0.012$ & $0.610 \pm 0.011$ & $0.562 \pm 0.019$ \\
& \cellcolor{orL}{PCD (ours)}   & \cellcolor{orL}{$\mathbf{1.000 \pm 0.000}$} & \cellcolor{orL}{$\mathbf{1.000 \pm 0.000}$} & \cellcolor{orL}{$\mathbf{1.000 \pm 0.000}$} & \cellcolor{orL}{$\mathbf{1.000 \pm 0.000}$} & \cellcolor{orL}{$\mathbf{1.000 \pm 0.000}$} \\
& \cellcolor{goL}{VD (ours)}    & \cellcolor{goL}{$\mathbf{1.000 \pm 0.000}$} & \cellcolor{goL}{$\mathbf{1.000 \pm 0.000}$} & \cellcolor{goL}{$\mathbf{1.000 \pm 0.000}$} & \cellcolor{goL}{$0.975 \pm 0.003$} & \cellcolor{goL}{$\mathbf{1.000 \pm 0.000}$} \\
\cmidrule(lr){2-7}
& MMDAgg    & $\mathbf{1.000 \pm 0.000}$ & $0.978 \pm 0.003$ & $0.788 \pm 0.012$ & $0.582 \pm 0.021$ & $0.871 \pm 0.010$ \\
& MMD-FUSE  & $\mathbf{1.000 \pm 0.000}$ & $\mathbf{1.000 \pm 0.000}$ & $\mathbf{1.000 \pm 0.000}$ & $0.894 \pm 0.011$ & $\mathbf{1.000 \pm 0.000}$ \\
& MMD-DUAL  & $\mathbf{1.000 \pm 0.000}$ & $\mathbf{1.000 \pm 0.000}$ & $\mathbf{1.000 \pm 0.000}$ & $0.952 \pm 0.005$ & $\mathbf{1.000 \pm 0.000}$ \\
& \cellcolor{grL}{USAD (ours)}  & \cellcolor{grL}{$\mathbf{1.000 \pm 0.000}$} & \cellcolor{grL}{$\mathbf{1.000 \pm 0.000}$} & \cellcolor{grL}{$\mathbf{1.000 \pm 0.000}$} & \cellcolor{grL}{$\mathbf{1.000 \pm 0.000}$} & \cellcolor{grL}{$\mathbf{1.000 \pm 0.000}$} \\
\midrule
\multirow{7}{*}{WRN28} 
& SAMMD     & $0.387 \pm 0.018$ & $0.193 \pm 0.010$ & $0.227 \pm 0.015$ & $0.769 \pm 0.010$ & $0.319 \pm 0.015$ \\
& \cellcolor{orL}{PCD (ours)}   & \cellcolor{orL}{$\mathbf{1.000 \pm 0.000}$} & \cellcolor{orL}{$\mathbf{1.000 \pm 0.000}$} & \cellcolor{orL}{$\mathbf{1.000 \pm 0.000}$} & \cellcolor{orL}{$\mathbf{1.000 \pm 0.000}$} & \cellcolor{orL}{$\mathbf{1.000 \pm 0.000}$} \\
& \cellcolor{goL}{VD (ours)}    & \cellcolor{goL}{$\mathbf{1.000 \pm 0.000}$} & \cellcolor{goL}{$0.993 \pm 0.001$} & \cellcolor{goL}{$0.999 \pm 0.001$} & \cellcolor{goL}{$\mathbf{1.000 \pm 0.000}$} & \cellcolor{goL}{$\mathbf{1.000 \pm 0.000}$} \\
\cmidrule(lr){2-7}
& MMDAgg    & $0.459 \pm 0.022$ & $0.245 \pm 0.011$ & $0.295 \pm 0.017$ & $0.812 \pm 0.020$ & $0.522 \pm 0.014$ \\
& MMD-FUSE  & $0.844 \pm 0.011$ & $0.625 \pm 0.007$ & $0.794 \pm 0.012$ & $0.992 \pm 0.004$ & $0.993 \pm 0.002$ \\
& MMD-DUAL  & $0.937 \pm 0.006$ & $0.764 \pm 0.007$ & $0.900 \pm 0.009$ & $0.998 \pm 0.002$ & $0.998 \pm 0.002$ \\
& \cellcolor{grL}{USAD (ours)} & \cellcolor{grL}{$\mathbf{1.000 \pm 0.000}$} & \cellcolor{grL}{$\mathbf{1.000 \pm 0.000}$} & \cellcolor{grL}{$\mathbf{1.000 \pm 0.000}$} & \cellcolor{grL}{$\mathbf{1.000 \pm 0.000}$} & \cellcolor{grL}{$\mathbf{1.000 \pm 0.000}$} \\
\midrule
\multirow{7}{*}{WRN70} 
& SAMMD     & $0.398 \pm 0.009$ & $0.207 \pm 0.015$ & $0.227 \pm 0.010$ & $0.792 \pm 0.015$ & $0.199 \pm 0.013$ \\
& \cellcolor{orL}{PCD (ours)}   & \cellcolor{orL}{$\mathbf{1.000 \pm 0.000}$} & \cellcolor{orL}{$\mathbf{1.000 \pm 0.000}$} & \cellcolor{orL}{$\mathbf{1.000 \pm 0.000}$} & \cellcolor{orL}{$\mathbf{1.000 \pm 0.000}$} & \cellcolor{orL}{$\mathbf{1.000 \pm 0.000}$} \\
& \cellcolor{goL}{VD (ours)}    & \cellcolor{goL}{$\mathbf{1.000 \pm 0.000}$} & \cellcolor{goL}{$\mathbf{1.000 \pm 0.000}$} & \cellcolor{goL}{$0.999 \pm 0.001$} & \cellcolor{goL}{$\mathbf{1.000 \pm 0.000}$} & \cellcolor{goL}{$0.999 \pm 0.001$} \\
\cmidrule(lr){2-7}
& MMDAgg    & $0.524 \pm 0.016$ & $0.289 \pm 0.010$ & $0.352 \pm 0.014$ & $0.764 \pm 0.015$ & $0.324 \pm 0.015$ \\
& MMD-FUSE  & $0.970 \pm 0.004$ & $0.810 \pm 0.011$ & $0.847 \pm 0.008$ & $0.982 \pm 0.005$ & $0.795 \pm 0.014$ \\
& MMD-DUAL  & $0.988 \pm 0.003$ & $0.912 \pm 0.007$ & $0.911 \pm 0.009$ & $0.996 \pm 0.002$ & $0.882 \pm 0.008$ \\
& \cellcolor{grL}{USAD (ours)}  & \cellcolor{grL}{$\mathbf{1.000 \pm 0.000}$} & \cellcolor{grL}{$\mathbf{1.000 \pm 0.000}$} & \cellcolor{grL}{$\mathbf{1.000 \pm 0.000}$} & \cellcolor{grL}{$\mathbf{1.000 \pm 0.000}$} & \cellcolor{grL}{$\mathbf{1.000 \pm 0.000}$} \\
\midrule
\multirow{7}{*}{SWIN} 
& SAMMD     & $0.365 \pm 0.015$ & $0.306 \pm 0.017$ & $0.279 \pm 0.016$ & $0.600 \pm 0.015$ & $0.403 \pm 0.015$ \\
& \cellcolor{orL}{PCD (ours)}   & \cellcolor{orL}{$\mathbf{1.000 \pm 0.000}$} & \cellcolor{orL}{$\mathbf{1.000 \pm 0.000}$} & \cellcolor{orL}{$\mathbf{1.000 \pm 0.000}$} & \cellcolor{orL}{$\mathbf{1.000 \pm 0.000}$} & \cellcolor{orL}{$\mathbf{1.000 \pm 0.000}$} \\
& \cellcolor{goL}{VD (ours)}    & \cellcolor{goL}{$\mathbf{1.000 \pm 0.000}$} & \cellcolor{goL}{$\mathbf{1.000 \pm 0.000}$} & \cellcolor{goL}{$\mathbf{1.000 \pm 0.000}$} & \cellcolor{goL}{$\mathbf{1.000 \pm 0.000}$} & \cellcolor{goL}{$\mathbf{1.000 \pm 0.000}$} \\
\cmidrule(lr){2-7}
& MMDAgg    & $0.630 \pm 0.010$ & $0.535 \pm 0.021$ & $0.515 \pm 0.007$ & $0.730 \pm 0.015$ & $0.420 \pm 0.010$ \\
& MMD-FUSE  & $0.999 \pm 0.001$ & $0.992 \pm 0.003$ & $0.987 \pm 0.002$ & $\mathbf{1.000 \pm 0.000}$ & $0.894 \pm 0.010$ \\
& MMD-DUAL  & $\mathbf{1.000 \pm 0.000}$ & $0.995 \pm 0.002$ & $0.994 \pm 0.003$ & $\mathbf{1.000 \pm 0.000}$ & $0.946 \pm 0.007$ \\
& \cellcolor{grL}{USAD (ours)}  & \cellcolor{grL}{$\mathbf{1.000 \pm 0.000}$} & \cellcolor{grL}{$\mathbf{1.000 \pm 0.000}$} & \cellcolor{grL}{$\mathbf{1.000 \pm 0.000}$} & \cellcolor{grL}{$\mathbf{1.000 \pm 0.000}$} & \cellcolor{grL}{$\mathbf{1.000 \pm 0.000}$} \\
\bottomrule
\end{tabular}%
}
\label{tab:transfer}
\end{table}

\subsection{Type-I Error Control Check} \label{app:type_1_control}
Controlling type-I error (false alarm rate) is crucial for practical deployment. Figure~\ref{exp:typeI} shows that all SAD and MMD testing methods maintain Type-I error rates close to the theoretical significance level of $\alpha = 0.05$ across different number of examples in the test set. Our methods (PCD, VD, and USAD) also exhibit stable type-I error rates, confirming the significance of all the statistical adversarial detection methods.




\subsection{Computational Efficiency.} \label{app:computation_efficiency}
Table \ref{tab:running_times} reports the inference time per test trial averaged over 100 runs, with all experiments conducted using reference example size $|X|$ equal to input example size $|Y| = 50$ on the platform described in Appendix~\ref{resources}. Our VD method demonstrates superior computational efficiency across both datasets, achieving $0.016\text{s}$ on CIFAR-10 and $0.015\text{s}$ on ImageNet-1K—representing 33\% and 85\% speedups over the baseline SAMMD, respectively. Notably, VD's inference time remains nearly constant across datasets, indicating excellent scalability to higher-dimensional data. While SAMMD shows moderate computational cost on CIFAR-10 $(0.024\text{s})$, it increases substantially to $0.097\text{s}$ on ImageNet-1K, suggesting quadratic or higher complexity with respect to feature dimensionality. In contrast, PCD and USAD exhibit significantly longer inference times ($0.117$s and $0.132$s on CIFAR-10; $0.759$s and $0.772$s on ImageNet-1K), due to their reliance on computationally intensive distance matrix operations or optimization procedures. The minimal overhead of VD, combined with its strong detection performance, makes it particularly attractive for real-time or large-scale deployment scenarios where computational budget is constrained.
Although PCD and USAD are slower than SAMMD, their inference time remains below one second per query batch, and, even in large-scale deployment, USAD only needs to be applied to small query windows rather than the entire data stream.
The extra cost is useful because PCD can detect attacks where VD is weak: for example, under the adaptive attack with $m = 10$ in Table~\ref{tab:adaptive}, VD achieves $0.186$ power, while USAD reaches $0.818$.
Thus, VD can be used when latency is critical, while USAD can be used when reliable detection is more important; the trade-off between the number of perturbations $K$ in PCD and the running time is further studied in Appendix~\ref{app:sensitivity}.
\begin{table}[t]
\centering
\caption{Running Time (seconds) per test trial. We record the total inference time for four methods across 100 trials, then we divide each total time by 100. For PCD and USAD, it is also possible to pre-compute the distance matrix between each covariance matrix generated by CEs offline, which can make the computation inference time at least twice faster than current implementation.}
\label{tab:running_times}
\resizebox{\textwidth}{!}{%
\begin{tabular}{l|ccccccc}
\toprule
Dataset & MMDAgg & MMD-FUSE & MMD-DUAL & SAMMD & PCD & VD & USAD \\
\midrule
\textit{CIFAR-10}    & 0.072 & 0.047 & 0.047 & 0.024 & 0.117 & 0.016 & 0.132 \\
\textit{ImageNet-1K} & 0.051 & 0.044 & 0.037 & 0.097 & 0.759 & 0.015 & 0.772 \\
\bottomrule
\end{tabular}%
}
\end{table}

\subsection{\texorpdfstring{$\ell_0$ (Sparse-PGD) and $\ell_1$ Attacks.}{L0 (Sparse-PGD) and L1 Attacks.}}
\label{app:sparse_attacks}

While the main paper focuses on dense $\ell_\infty$ and $\ell_2$ adversarial perturbations, sparse attacks induce qualitatively different geometric distortions and thus warrant separate empirical verification. To assess the robustness of our detection framework against this regime, we additionally evaluate against $\ell_0$-bounded perturbations generated by Sparse-PGD (SPGD)~\citep{zhong2025sparse} and $\ell_1$-bounded perturbations generated by $\ell_1$-APGD. For SPGD, we vary the sparsity budget $k \in \{10, 20, 50, 100\}$ (i.e., the number of perturbed pixels), and for $\ell_1$-APGD we sweep the perturbation norm $N \in \{10, 20, 30\}$. All experiments follow the same evaluation protocol as the main paper and are verified to maintain a valid type-I error control.

Table~\ref{tab:sparse_attacks} reports the detection power of PCD, VD, and USAD across these settings. Our methods, particularly VD and USAD, remain highly effective against sparse adversarial perturbations: both attain detection power of $1.000$ across nearly all configurations, with only marginal degradation under the most restrictive $\ell_1$ budget ($\ell_1$-N10). PCD shows a sharper sensitivity to the perturbation budget under $\ell_1$ attacks, but recovers to near-perfect detection as the budget grows. These results confirm that the proposed framework generalizes beyond dense $\ell_\infty$/$\ell_2$ threat models and is robust to the geometrically distinct distortions induced by sparse adversarial attacks.

\begin{table}[ht]
\centering
\caption{Detection power against $\ell_1$-APGD (varying norm $N$) and Sparse-PGD (varying sparsity $k$) on CIFAR-10. All results are averaged over multiple trials and reported as mean $\pm$ standard deviation. Type-I error is controlled in all settings.}
\label{tab:sparse_attacks}
\small
\resizebox{\textwidth}{!}{%
\begin{tabular}{l|ccc|cccc}
\toprule
Method & $\ell_1$-N10 & $\ell_1$-N20 & $\ell_1$-N30 & SPGD-$k$10 & SPGD-$k$20 & SPGD-$k$50 & SPGD-$k$100 \\
\midrule
PCD  & $0.231 \pm 0.009$ & $0.476 \pm 0.009$ & $0.719 \pm 0.011$ & $0.991 \pm 0.003$ & $1.000 \pm 0.000$ & $1.000 \pm 0.000$ & $1.000 \pm 0.000$ \\
VD   & $0.977 \pm 0.005$ & $1.000 \pm 0.000$ & $1.000 \pm 0.000$ & $1.000 \pm 0.000$ & $1.000 \pm 0.000$ & $1.000 \pm 0.000$ & $1.000 \pm 0.000$ \\
USAD & $0.981 \pm 0.004$ & $1.000 \pm 0.000$ & $1.000 \pm 0.000$ & $1.000 \pm 0.000$ & $1.000 \pm 0.000$ & $1.000 \pm 0.000$ & $1.000 \pm 0.000$ \\
\bottomrule
\end{tabular}%
}
\end{table}

\subsection{Adversarial Video Detection on UCF101.}
\label{app:video_detection}

A natural extension of our framework is to the video domain, where adversaries typically inject perturbations across multiple frames and video models routinely process frames in batches. This setting is particularly well-suited to multi-sample detectors such as ours, since the temporal structure of a video provides a natural batch of correlated samples for two-sample testing. To verify that our method transfers to this regime, we conduct an additional evaluation on the UCF101 action recognition benchmark using the sparse video attack proposed by~\cite{wei2019sparse}.

\textbf{Setup.} We attack 65 sample videos drawn from the first two classes in alphabetical order (\textit{ApplyEyeMakeup} and \textit{Archery}). For each successfully attacked video, we extract a 40-frame clip from both its clean and adversarial versions, treating each clip as a sample set for two-sample testing. Detection is performed by running a permutation test between the clean and adversarial frame sets drawn from the same class, yielding a binary detection decision per video. The reported detection rate (test power) is averaged over all successfully attacked videos, and type-I error is computed under the null using clean-vs-clean clip pairs.

\paragraph{Results.} Table~\ref{tab:video_detection} summarizes the outcome. Our method achieves a detection rate of $81.5\%$ while maintaining $0.0\%$ type-I error, demonstrating that it transfers effectively to the video domain without modification. This result is consistent with our intuition: because adversarial video attacks distribute perturbations across many frames, the per-frame distributional shift accumulates into a strong multi-sample signal that our framework is well-equipped to detect.

\begin{table}[ht]
\centering
\caption{Adversarial video detection on UCF101 against the sparse video attack of~\cite{wei2019sparse}. We evaluate on the first two classes (alphabetical), comparing 40-frame clean and adversarial clips from the same class via a permutation test. Type-I error is controlled.}
\label{tab:video_detection}
\small
\begin{tabular}{lccc}
\toprule
Class & Total $N$ & Avg.\ Detection & Type-I Error \\
\midrule
\textit{ApplyEyeMakeup} + \textit{Archery} & 65 & $81.5\%$ & $0.0\%$ \\
\bottomrule
\end{tabular}
\end{table}

\subsection{Ablation on the Correlation-aware Aggregation and Incorporating MMD}
\label{app:ablation}

\textbf{Equal-weight Aggregation.}
To isolate the effect of the correlation-aware aggregation in \cref{sec: dual}, we compare USAD with a simple equal-weight sum of standardized VD and PCD, where VD and PCD are first standardized by their null standard deviations (i.e., the square roots of the diagonal entries of $\widehat{\Sigma}_{\mathrm{A}}$ in Eq.~\eqref{eq:sigma_aggregate}) and then added with equal weights, with the testing threshold calibrated by the same permutation test. The experiments are conducted on CIFAR-10 using ResNet-18 as the target model under a five-step PGD $\ell_\infty$ attack, with $|X|=100$ and $|Y|=50$.
All methods control the Type-I error at $\alpha = 0.05$, and the results are averaged over $10$ experiments with $100$ tests each.
Table~\ref{tab:equal_weight} shows that USAD matches or outperforms the equal-weight sum across all settings, improving power by $0.045$ and $0.024$ at $\epsilon = 1/255$ and $2/255$, respectively.
This supports the benefit of accounting for the correlation between VD and PCD.

\begin{table}[ht]
\centering
\caption{Test power of PCD, VD, an equal-weight sum of standardized VD and PCD, and USAD under different perturbation budgets $\epsilon$. All methods control the Type-I error at $\alpha=0.05$, and the results are averaged over $10$ experiments with $100$ tests each. The experiments are conducted on CIFAR-10 using ResNet-18 under a five-step PGD $\ell_\infty$ attack, with $|X|=100$ and $|Y|=50$.}
\label{tab:equal_weight}
\begin{tabular}{c|cccc}
\toprule
$\epsilon$ & PCD & VD & Equal-weight sum & USAD \\
\midrule
$1/255$ & 0.344 & 0.769 & 0.752 & $\mathbf{0.797}$ \\
$2/255$ & 0.175 & 0.777 & 0.755 & $\mathbf{0.779}$ \\
$3/255$ & 0.313 & 1.000 & 1.000 & $\mathbf{1.000}$ \\
\bottomrule
\end{tabular}
\end{table}

\textbf{Incorporating MMD.}
As discussed in \cref{sec: theo}, the aggregation framework is general and can incorporate additional statistics.
To demonstrate this flexibility, we add MMD to the USAD framework, i.e., we aggregate VD, PCD, and $\widehat{\textnormal{MMD}}^2$ with the same correlation-aware weighting as in Eq.~\eqref{eq:usad_statistic}.
The experiments are conducted on CIFAR-10 using ResNet-18 as the target model under a five-step PGD $\ell_\infty$ attack, with $|X|=100$ and $|Y|=50$.
The results are averaged over $10$ experiments with $150$ tests each, with standard error at most $0.016$.
Table~\ref{tab:usad_mmd} confirms that USAD can incorporate MMD.
However, MMD provides limited additional signal in the weaker-attack settings and reduces the finite-sample power.
This shows that adding more statistics does not necessarily improve performance; the statistics should be selected according to the target alternatives.

\begin{table}[ht]
\centering
\caption{Test power of USAD (VD+PCD) and its variant incorporating MMD (VD+PCD+MMD) under different perturbation budgets $\epsilon$. The results are averaged over $10$ experiments with $150$ tests each, with standard error at most $0.016$. The experiments are conducted on CIFAR-10 using ResNet-18 under a five-step PGD $\ell_\infty$ attack, with $|X|=100$ and $|Y|=50$.}
\label{tab:usad_mmd}
\resizebox{\textwidth}{!}{%
\begin{tabular}{l|ccccc}
\toprule
Method & $\epsilon=1/255$ & $\epsilon=2/255$ & $\epsilon=3/255$ & $\epsilon=4/255$ & $\epsilon=5/255$--$8/255$ \\
\midrule
USAD (VD+PCD) & $\mathbf{0.769}$ & $\mathbf{0.775}$ & $\mathbf{1.000}$ & $\mathbf{1.000}$ & $\mathbf{1.000}$ \\
USAD (VD+PCD+MMD) & 0.611 & 0.697 & $\mathbf{1.000}$ & $\mathbf{1.000}$ & $\mathbf{1.000}$ \\
\bottomrule
\end{tabular}%
}
\end{table}

\subsection{Analysis of the Two Uncertainty Factors}
\label{app:factor_analysis}

\textbf{Arguments from Prior Work.}
\cref{factor:global_uncertainty} builds on the observation that clean and adversarial data reside on different manifolds in the semantic feature space, with AEs typically lying slightly off the clean data submanifold and exhibiting altered neighborhood structures.
In particular, \citet{Feinman:Curtin:Shintre:Gardner2017} detect AEs through kernel density estimates in the feature space of the last hidden layer together with Bayesian uncertainty estimates, and \citet{ma2018characterizing} show that adversarial regions exhibit higher local intrinsic dimensionality (LID) than normal data regions; both indicate that AEs alter how semantic features are spread, which VD measures through the within-distribution variance.
\cref{factor:local_uncertainty} builds on two complementary observations.
On the one hand, under isotropic Gaussian perturbations, randomized smoothing shows that predictions can be certified to be locally stable, providing evidence that clean inputs tend to exhibit greater stability to small noise~\cite{Cohen:Rosenfeld:Kolter2019,Li:Chen:Wang:Carin2019}.
On the other hand, norm-bounded AEs are generated by small perturbations that cross local decision boundaries~\cite{Moosavi:Fawzi:Frossard2016,fawzi2016robustness} and often exploit brittle non-robust features~\cite{Ilyas:Santurkar:Tsipras:Engstrom:Tran:Madry2019}, indicating increased sensitivity to certain perturbations.
Accordingly, random-perturbation-based detectors~\cite{huang2019model,roth2019odds,zhang2023detecting} distinguish AEs from CEs through their atypical local response patterns under added perturbations, which motivates PCD to compare the covariances of semantic features under Gaussian perturbations.

VD and PCD are designed to capture these two factors, respectively.
While the test power of VD and PCD across real attacks (e.g., Table~\ref{tab:cifar}) supports this design, it does not directly show how strongly the two factors are present in each setting.
We therefore provide a more direct analysis, first in a controlled synthetic setting and then on real AEs.

\textbf{Controlled Synthetic Experiment.}
We use a two-dimensional classifier whose decision boundary is the line on which the first input coordinate equals $0$, and repeat each setting $200$ times with $|X| = |Y| = 50$.
For \cref{factor:global_uncertainty}, we keep the clean and query samples at the same distance from the decision boundary, so that their local sensitivity is unchanged, and only increase the global spread of the query distribution.
For \cref{factor:local_uncertainty}, we keep the global spread fixed and move the query samples closer to the decision boundary by decreasing their boundary distance.
Table~\ref{tab:factor_synthetic} reports the values of $\widehat{\textnormal{VD}}$ and $\widehat{\textnormal{PCD}}$ in both cases.
VD increases consistently as the global spread increases, while PCD remains unchanged; conversely, as the samples approach the decision boundary, PCD increases substantially, while VD remains close to zero.
These controlled results support that VD captures changes in the global feature spread, whereas PCD captures the increased local sensitivity near the decision boundary.

\begin{table}[ht]
\centering
\caption{Controlled synthetic analysis of \cref{factor:global_uncertainty} (left: increasing the global spread of the query distribution with the boundary distance fixed) and \cref{factor:local_uncertainty} (right: decreasing the boundary distance of the query samples with the global spread fixed). We report $\widehat{\textnormal{VD}}$ (mean $\pm$ standard deviation) and $\widehat{\textnormal{PCD}}$ over $200$ repetitions with $|X| = |Y| = 50$, using a two-dimensional classifier whose decision boundary is the line on which the first input coordinate equals $0$.}
\label{tab:factor_synthetic}
\begin{tabular}{ccc}
\toprule
Query spread & VD & PCD \\
\midrule
0.05 & $0.0035 \pm 0.0050$ & 0.0000 \\
0.10 & $0.0617 \pm 0.0248$ & 0.0000 \\
0.20 & $0.1502 \pm 0.0335$ & 0.0000 \\
0.30 & $0.1982 \pm 0.0386$ & 0.0000 \\
0.40 & $0.2178 \pm 0.0335$ & 0.0000 \\
\bottomrule
\end{tabular}
\hspace{2em}
\begin{tabular}{ccc}
\toprule
Boundary distance & VD & PCD \\
\midrule
0.50 & $0.0030 \pm 0.0044$ & 0.0000 \\
0.20 & $0.0035 \pm 0.0044$ & 0.0001 \\
0.10 & $0.0030 \pm 0.0043$ & 0.1206 \\
0.05 & $0.0030 \pm 0.0047$ & 1.0694 \\
0.02 & $0.0027 \pm 0.0036$ & 1.6369 \\
\bottomrule
\end{tabular}
\end{table}

\textbf{Analysis on Real AEs.}
Following the setup of the synthetic study, we further measure the two factors on real AEs generated by different attacks.
We consider ResNet-18 trained on \textit{CIFAR-10} with $\epsilon = 1/255$ and $|Y| = 50$, i.e., the setting of Table~\ref{tab:cifar}, from which the VD and PCD power below are taken.
Both quantities are computed using the classifier's 64-dimensional penultimate features.
For \cref{factor:global_uncertainty}, we measure the spread change as the absolute difference in empirical RKHS variance between the query batch and the clean reference, $|\widehat V(Y;\ell)-\widehat V(X;\ell)|$, using the Gaussian kernel $\ell(f(\x),f(\x'))=\exp(-\|f(\x)-f(\x')\|_2^2/32)$.
For \cref{factor:local_uncertainty}, we measure the Euclidean distance from the semantic feature $f(\y)$ of each AE $\y$ to the nearest boundary of its predicted-class region under the final linear classification layer:
\[
\mathrm{dist}(\y)=\min_{c\neq\hat{c}}\frac{s_{\hat{c}}(\y)-s_c(\y)}{\|w_{\hat{c}}-w_c\|_2},
\]
where $s_c(\y)$ is the logit for class $c\in\mathcal{C}$, which is affine in $f(\y)$ with weight vector $w_c$ in the final linear classification layer, and $\hat{c}=\arg\max_{c\in\mathcal{C}}s_c(\y)$ is the predicted class.
Boundary distances are first averaged over the AEs in each query batch, and both quantities are then averaged over $1{,}000$ batches with $|X|=|Y|=50$.
Table~\ref{tab:factor_real} shows that a larger spread change is associated with higher VD power across these attacks; in particular, FGSM has both the largest spread change and the highest VD power.
Similarly, a smaller boundary distance tends to be associated with higher PCD power, and FGSM has both the smallest boundary distance and the highest PCD power.
Overall, these trends are consistent with our interpretation that VD captures global feature-spread changes, while PCD captures local perturbation sensitivity.

\begin{table}[ht]
\centering
\caption{Analysis of \cref{factor:global_uncertainty} (left: change in the feature spread of the query batch and VD power) and \cref{factor:local_uncertainty} (right: average boundary distance of the AEs and PCD power) on real AEs generated by different attacks against ResNet-18 trained on \textit{CIFAR-10}, with $\epsilon=1/255$ and $|Y|=50$.}
\label{tab:factor_real}
\begin{tabular}{ccc}
\toprule
Attack & Spread change & VD power \\
\midrule
AA   & 0.0464 & $0.261 \pm 0.010$ \\
CW   & 0.0699 & $0.507 \pm 0.010$ \\
PGD  & 0.0800 & $0.653 \pm 0.015$ \\
FGSM & $\mathbf{0.1402}$ & $\mathbf{0.794 \pm 0.016}$ \\
\bottomrule
\end{tabular}
\hspace{2em}
\begin{tabular}{ccc}
\toprule
Attack & Boundary distance & PCD power \\
\midrule
AA   & 2.112 & $0.222 \pm 0.016$ \\
CW   & 1.326 & $0.234 \pm 0.012$ \\
PGD  & 1.246 & $0.251 \pm 0.012$ \\
FGSM & $\mathbf{0.954}$ & $\mathbf{0.311 \pm 0.014}$ \\
\bottomrule
\end{tabular}
\end{table}

\subsection{Detection under Low AE Fractions in the Query Batch}
\label{app:low_fraction}

In practical monitoring scenarios, AEs may constitute only a small fraction of the query batch.
Our consistency results (Theorems~\ref{thm:VD_test}--\ref{thm:PCD_test} and Corollary~\ref{cor:USAD_test}) show that, for any fixed non-zero AE fraction, the detection power approaches $1$ as the batch size increases, provided the AEs change the uncertainty captured by VD or PCD.
Therefore, there is no universal minimum AE fraction: reliable finite-sample detection depends jointly on the AE fraction, the batch size, and the attack signal, and a smaller AE fraction must be compensated by a larger batch.

\textbf{Effect of the AE Fraction on Test Power.}
We make this dependence explicit as follows.
Let $\eta \in (0, 1]$ denote the fraction of AEs in the query batch, and write the query distribution as the mixture $\bQ_\eta = (1-\eta)\bP + \eta\,\mathbb{A}$, where $\mathbb{A}$ denotes the distribution of AEs.
We define
\begin{equation*}
    \Delta_V = \left|V_\bP - V_{\mathbb{A}} - (1-\eta)\,\textnormal{MMD}^2(\bP,\mathbb{A},\ell)\right|
    \quad\text{and}\quad
    \Delta_C = \|\mmu_{\Sigma_\bP}-\mmu_{\Sigma_{\mathbb{A}}}\|_{\cH_\rho},
\end{equation*}
where $V_{\mathbb{A}}$ and $\mmu_{\Sigma_{\mathbb{A}}}$ are defined as in \cref{sec:vd} and \cref{sec:pcd}, and $\textnormal{MMD}^2(\bP,\mathbb{A},\ell)$ denotes the MMD in Eq.~\eqref{eq: mmd} computed on semantic features with kernel $\ell$ (Appendix~\ref{app:mmd}).
Here, $\Delta_V$ and $\Delta_C$ measure how different the clean and adversarial distributions are in the global variance captured by VD and in the perturbation-induced covariance captured by PCD, respectively.
We also write $N_{\rm eff} = nm/(n+m)$ for the effective sample size.
\begin{remark}[Informal]\label{rem:ae_fraction}
Let $\ell$ and $\rho$ be bounded kernels, and let $\Delta$ be either $\Delta_V$ (for VD) or $\Delta_C$ (for PCD).
Up to kernel-dependent constants and lower-order terms, whenever $\eta\Delta \gtrsim \sqrt{\log(1/\alpha)/N_{\rm eff}}$, the probability that the corresponding test fails to reject $\H_0$ when $Y$ is drawn from $\bQ_\eta$ (i.e., its Type-II error) satisfies
\begin{equation*}
    \Pr\left(\cT(X,Y) \leq t_\alpha \mid Y \sim \bQ_\eta\right) \lesssim \exp\left\{-C N_{\rm eff}\,\eta^2\Delta^2\right\}
\end{equation*}
for some constant $C > 0$.
Consequently, $N_{\rm eff}\,\eta^2\Delta^2 \gtrsim \log(1/\alpha) + \log(1/\beta)$ suffices for a Type-II error of at most $\beta$.
\end{remark}
\begin{proof}[Proof sketch]
(1) \emph{Population discrepancies.}
For VD, expanding the RKHS variance under the mixture gives $V_{\bQ_\eta} = (1-\eta)V_\bP + \eta V_{\mathbb{A}} + \eta(1-\eta)\,\textnormal{MMD}^2(\bP,\mathbb{A},\ell)$, and hence $|V_\bP - V_{\bQ_\eta}| = \eta\Delta_V$.
For PCD, drawing from $\bQ_\eta$ produces the covariance matrix of a CE with probability $1-\eta$ and that of an AE with probability $\eta$; by linearity of kernel mean embeddings, $\mmu_{\Sigma_{\bQ_\eta}} = (1-\eta)\mmu_{\Sigma_\bP} + \eta\,\mmu_{\Sigma_{\mathbb{A}}}$, and hence $\|\mmu_{\Sigma_\bP} - \mmu_{\Sigma_{\bQ_\eta}}\|_{\cH_\rho} = \eta\Delta_C$.
Thus, the population discrepancies underlying both VD and PCD decrease linearly with $\eta$.
(2) \emph{Concentration and permutation thresholds.}
The empirical VD gap $\widehat{V}(X;\ell)-\widehat{V}(Y;\ell)$ is a difference of bounded second-order $U$-statistics (Appendix~\ref{app:VD_test}), and Hoeffding's inequality~\cite{Hoeffding1994} gives $\Pr\big(|(\widehat{V}(X;\ell)-\widehat{V}(Y;\ell)) - (V_\bP - V_{\bQ_\eta})| \geq t\big) \lesssim \exp(-C N_{\rm eff}\,t^2)$ for any $t > 0$.
For PCD, the RKHS distance between the empirical mean embeddings of $\{\Sigma_{\x_i}\}_{i=1}^{n}$ and $\{\Sigma_{\y_j}\}_{j=1}^{m}$ deviates from $\eta\Delta_C$ by more than $O(n^{-1/2}+m^{-1/2}) + t$ with probability $\lesssim \exp(-C N_{\rm eff}\,t^2)$~\cite{Gretton:Borgwardt:Rasch:Scholkopf:Smola2012}, and, for a bounded kernel $\rho$, the unbiased estimator in Eq.~\eqref{eq:PCD_hat} differs from the squared distance between these empirical embeddings by at most $O(1/n+1/m)$; these are lower-order terms once $\eta\Delta_C$ is above the sampling noise.
Under permutation, the two groups are formed from the same pooled sample, and concentration bounds for permuted $U$-statistics~\cite{Kim:Balakrishnan:Wasserman2022} (Theorem~\ref{thm:6.1} for PCD; for VD, the permuted gap is a difference of $U$-statistics of samples drawn without replacement, see Appendix~\ref{app:VD_test}) imply that the critical radius of the VD gap or the PCD embedding distance is of order $\sqrt{\log(1/\alpha)/N_{\rm eff}}$, or, equivalently, that the critical value of the corresponding squared statistic is of order $\log(1/\alpha)/N_{\rm eff}$.
(3) \emph{Type-II error.}
If $\eta\Delta$ is larger than this critical radius and the lower-order terms above, then a failure to reject $\H_0$ requires the empirical discrepancy to differ from its population value by at least a constant fraction of $\eta\Delta$; substituting $t \asymp \eta\Delta$ into the concentration bounds above yields the bound in Remark~\ref{rem:ae_fraction}.
\end{proof}
Remark~\ref{rem:ae_fraction} gives a quadratic sample-size dependence on the AE fraction: when $\Delta$ remains of the same order, halving the AE fraction requires roughly four times as many effective samples to maintain the same Type-II error.
Moreover, if $n$ is fixed, increasing only $m$ has diminishing benefit because $N_{\rm eff} \to n$, which explains why very low AE fractions remain difficult even with large query batches.
The remark also requires the relevant discrepancy $\Delta$ to remain non-negligible: if VD or PCD captures little difference between the clean and adversarial distributions, increasing the sample size cannot compensate for an absent signal.

\textbf{Setup.}
We evaluate USAD and the MMD-based methods under different AE fractions and query batch sizes on CIFAR-10 with ResNet-18 under PGD-$\ell_\infty$ with $\epsilon = 8/255$, using a fixed reference size $|X| = 100$ and $|Y| \in \{50, 100, 200, 500, 1000\}$, with $5$ independent experiments, $100$ tests per experiment, and $100$ permutations.
Since the number of AEs in a query batch must be an integer, the implementation rounds the product of the AE fraction and $|Y|$ down to an integer; thus, for $|Y| = 50$ and a $1\%$ AE fraction, the query batch contains no AE, and this entry measures the Type-I error rather than the power (for $|Y| = 50$, the smallest non-zero AE fraction is $2\%$, corresponding to one AE).

\textbf{Results.}
Table~\ref{tab:low_fraction} shows that larger batches can accumulate weak evidence from diluted AEs: at a $10\%$ AE fraction, for example, the power of USAD increases from $0.650$ to $0.961$ as the batch size grows from $50$ to $1000$.
This is an advantage of batch-level detection, which can combine information across multiple queries.
USAD achieves a clear advantage over the MMD-based methods at $10\%$ contamination and remains consistently stronger at $5\%$.
At $1\%$--$2\%$, however, all methods have limited power.
Extremely low contamination, such as $0.1\%$ or $1\%$, is challenging for adversarial detection in general, including both batch-level and per-example methods.
For example, among $1{,}000$ queries, a $5\%$ per-example false-alarm rate produces about $50$ false alarms, while $1\%$ contamination contains only $10$ AEs and $0.1\%$ contamination contains only one AE.
Thus, false alarms can easily outnumber true AEs, and reliable detection in this regime requires a false-alarm rate substantially below the AE prevalence and may require much larger batches.

\begin{table}[ht]
\centering
\caption{Rejection rates of USAD and MMD-based methods for different AE fractions in the query batch and different query batch sizes $|Y|$, on CIFAR-10 with ResNet-18 under PGD-$\ell_\infty$ with $\epsilon=8/255$ and a fixed reference size $|X|=100$. The results are averaged over $5$ independent experiments with $100$ tests each, using $100$ permutations. The number of AEs is the product of the AE fraction and $|Y|$ rounded down, so the entries with a $1\%$ AE fraction and $|Y|=50$ contain no AE and measure the Type-I error.}
\label{tab:low_fraction}
\begin{tabular}{c|l|ccccc}
\toprule
AE fraction & Method & $|Y| = 50$ & $|Y| = 100$ & $|Y| = 200$ & $|Y| = 500$ & $|Y| = 1000$ \\
\midrule
\multirow{5}{*}{$10\%$}
& USAD     & $\mathbf{0.650}$ & $\mathbf{0.839}$ & $\mathbf{0.889}$ & $\mathbf{0.928}$ & $\mathbf{0.961}$ \\
& MMD-DUAL & 0.189 & 0.311 & 0.378 & 0.511 & 0.578 \\
& MMDAgg   & 0.054 & 0.056 & 0.070 & 0.072 & 0.038 \\
& MMD-FUSE & 0.052 & 0.044 & 0.054 & 0.054 & 0.044 \\
& SAMMD    & 0.044 & 0.056 & 0.067 & 0.033 & 0.056 \\
\midrule
\multirow{5}{*}{$5\%$}
& USAD     & $\mathbf{0.150}$ & $\mathbf{0.273}$ & $\mathbf{0.319}$ & $\mathbf{0.366}$ & $\mathbf{0.396}$ \\
& MMD-DUAL & 0.056 & 0.067 & 0.144 & 0.122 & 0.056 \\
& MMDAgg   & 0.038 & 0.044 & 0.040 & 0.048 & 0.026 \\
& MMD-FUSE & 0.040 & 0.044 & 0.028 & 0.044 & 0.030 \\
& SAMMD    & 0.033 & 0.067 & 0.056 & 0.022 & 0.056 \\
\midrule
\multirow{5}{*}{$2\%$}
& USAD     & 0.050 & 0.111 & 0.094 & 0.094 & 0.111 \\
& MMD-DUAL & 0.067 & 0.011 & 0.033 & 0.100 & 0.033 \\
& MMDAgg   & 0.030 & 0.038 & 0.036 & 0.034 & 0.024 \\
& MMD-FUSE & 0.038 & 0.044 & 0.022 & 0.036 & 0.024 \\
& SAMMD    & 0.022 & 0.022 & 0.011 & 0.078 & 0.067 \\
\midrule
\multirow{5}{*}{$1\%$}
& USAD     & 0.044 & 0.089 & 0.061 & 0.061 & 0.067 \\
& MMD-DUAL & 0.044 & 0.033 & 0.000 & 0.011 & 0.022 \\
& MMDAgg   & 0.042 & 0.040 & 0.030 & 0.036 & 0.022 \\
& MMD-FUSE & 0.042 & 0.042 & 0.024 & 0.032 & 0.024 \\
& SAMMD    & 0.011 & 0.044 & 0.011 & 0.000 & 0.033 \\
\bottomrule
\end{tabular}
\end{table}

\subsection{Behavior under a Benign Distribution Shift}
\label{app:benign_shift}

As discussed in \cref{sec:sadd_revisit} and Appendix~\ref{app:discussion}, VD and PCD are not formally specific to adversarial attacks, and benign shifts that alter the global feature spread or the perturbation-induced local covariance may also be detected.
We therefore evaluate how USAD responds to a benign distribution shift.

\textbf{Setup.}
We use a ResNet-18 classifier trained on \textit{CIFAR-10} as the target model.
We use clean CIFAR-10.1~\citep{recht2018cifar10} as a benign shift and CIFAR-10 PGD with $\epsilon = 4/255$ as an adversarial shift.
We use $30$ reference sets with $|X| = 100$ and $|Y| = 50$, and $100$ query batches per reference set, and report the results for each component averaged over different random seeds.
We report the benign false-alarm rate (BFAR), i.e., the proportion of CIFAR-10.1 batches incorrectly flagged as adversarial, the true-positive rate (TPR), i.e., the proportion of PGD batches correctly detected, and the AUROC, which measures how well the statistics separate benign-shift and adversarial batches across all thresholds; lower BFAR and higher TPR and AUROC are better.

\textbf{Results.}
Table~\ref{tab:benign_shift} shows that both PCD and VD remain reasonably robust to this benign shift, with BFAR below $10\%$ while maintaining $100\%$ PGD detection.
More importantly, the near-perfect AUROC shows that benign-shift and adversarial batches remain highly separable according to the VD and PCD statistics.
Thus, the main issue is not a lack of adversarial information, but that a threshold calibrated under the strict equality null may be sensitive to small but acceptable natural shifts.
To illustrate this more clearly, we vary the decision threshold and measure the PGD detection rate when the false-alarm rate on CIFAR-10.1 is limited to $1\%$, $3\%$, or $5\%$ (Table~\ref{tab:benign_shift_tpr}).
Even at the stringent $1\%$ BFAR operating point, PCD and VD retain $99.57\%$ and $99.93\%$ PGD detection, respectively, showing that a threshold that tolerates small benign shifts can still reliably distinguish the larger adversarial shift.

\begin{table}[ht]
\centering
\caption{Behavior of PCD and VD under a benign distribution shift (clean CIFAR-10.1) versus an adversarial shift (CIFAR-10 PGD with $\epsilon=4/255$), using $30$ reference sets with $|X|=100$, $|Y|=50$, and $100$ query batches per reference set. The target model is a ResNet-18 trained on \textit{CIFAR-10}. BFAR is the proportion of CIFAR-10.1 batches incorrectly flagged as adversarial, TPR is the proportion of PGD batches correctly detected, and AUROC measures the separability of benign-shift and adversarial batches across all thresholds.}
\label{tab:benign_shift}
\begin{tabular}{l|ccc}
\toprule
Component & CIFAR-10.1 BFAR $\downarrow$ & PGD TPR $\uparrow$ & Benign-vs-adversarial AUROC $\uparrow$ \\
\midrule
PCD & $7.77\%$ & $\mathbf{100.00\%}$ & $\mathbf{0.9996}$ \\
VD  & $9.75\%$ & $\mathbf{100.00\%}$ & $\mathbf{1.0000}$ \\
\bottomrule
\end{tabular}
\end{table}

\begin{table}[ht]
\centering
\caption{PGD detection rate (TPR) of PCD and VD when the decision threshold is chosen such that the false-alarm rate on CIFAR-10.1 batches (BFAR) is limited to $1\%$, $3\%$, or $5\%$, under the same setting as Table~\ref{tab:benign_shift}.}
\label{tab:benign_shift_tpr}
\begin{tabular}{l|ccc}
\toprule
Component & TPR at $1\%$ BFAR $\uparrow$ & TPR at $3\%$ BFAR $\uparrow$ & TPR at $5\%$ BFAR $\uparrow$ \\
\midrule
PCD & $\mathbf{99.57\%}$ & $\mathbf{99.90\%}$ & $\mathbf{99.93\%}$ \\
VD  & $\mathbf{99.93\%}$ & $\mathbf{100.00\%}$ & $\mathbf{100.00\%}$ \\
\bottomrule
\end{tabular}
\end{table}

\textbf{Discussion.}
Following \citet{zhou2026closeness}, these observations motivate viewing the problem as distribution closeness testing rather than strict equality testing, where the null hypothesis allows the discrepancy between $\bP$ and $\bQ$ to lie within an acceptable range of benign shift and is rejected only when the discrepancy exceeds this range.
For example, the threshold can be chosen so that only $1\%$, $3\%$, or $5\%$ of CIFAR-10.1 batches are flagged; a new batch is then classified as adversarial only when its VD or PCD discrepancy exceeds this acceptable benign-shift range.
As shown above, even at $1\%$ BFAR, almost all PGD batches are still detected, which means that small natural shifts can be tolerated while larger adversarial shifts remain detectable.
Building on this interpretation, our VD and PCD statistics could be extended to distribution closeness testing, and developing a formal closeness-testing procedure with theoretical guarantees is an interesting direction for future work.

\subsection{Comparison with MVD and with a Per-example Detector under an Adaptive C\&W Attack}
\label{app:mvd_pfd}

\textbf{VD versus MVD.}
Following the discussion in \cref{sec:vd}, we examine whether the simpler trace-level VD offers a practical advantage over MVD~\cite{makigusa2024}, which compares the full RKHS covariance operators.
We directly compare the two methods under identical experimental conditions, using the same AEs, semantic features, Gaussian kernel, sampled batches, and permutation assignments, with $|X| = |Y| = 50$. The experiments are conducted on CIFAR-10 using ResNet-18 as the target model, under the corresponding perturbation budgets used for each adversarial attack.
Table~\ref{tab:vd_mvd} shows that, across the four attacks, VD achieves higher test power than MVD, with the largest improvements on FGSM and PGD, and that MVD requires approximately $1.8\times$ the runtime of VD ($0.714$ ms versus $0.397$ ms).
These results support using the simpler trace-level VD for small-query SAD, where both test power and computational cost matter.

\begin{table}[ht]
\centering
\caption{Test power (left) and runtime (right) of VD and MVD under identical experimental conditions (the same AEs, semantic features, Gaussian kernel, sampled batches, and permutation assignments). The experiments are conducted on CIFAR-10 using ResNet-18 as the target model, with $|X|=|Y|=50$.}
\label{tab:vd_mvd}
\begin{tabular}{l|cc}
\toprule
Attack & VD & MVD \\
\midrule
CW   & $\mathbf{0.453}$ & 0.432 \\
FGSM & $\mathbf{0.957}$ & 0.522 \\
PGD  & $\mathbf{0.597}$ & 0.455 \\
AA   & $\mathbf{1.000}$ & 0.947 \\
\bottomrule
\end{tabular}
\hspace{2em}
\begin{tabular}{l|cc}
\toprule
Runtime & VD & MVD \\
\midrule
Total runtime (ms) & $\mathbf{0.397}$ & 0.714 \\
Relative total runtime & $\mathbf{1.00\times}$ & $1.80\times$ \\
\bottomrule
\end{tabular}
\end{table}

\textbf{PCD versus Existing Perturbation-based Detectors.}
Random perturbations have been used by prior per-example detectors to reveal differences between clean and adversarial inputs~\cite{huang2019model,roth2019odds,zhang2023detecting,wang2025detecting}.
To the best of our knowledge, PCD is the first batch-level detector to represent each input through the covariance of its semantic-feature responses under Gaussian perturbations and to test whether these covariance distributions differ between clean and query batches.
PCD uses a permutation-calibrated two-sample test with Type-I error control, whereas prior per-example detectors classify each input separately and do not provide explicit Type-I error control.

\textbf{Comparison with PFD under an Adaptive C\&W Attack.}
Since per-example detectors and SAD methods make decisions at different granularities, a direct comparison requires a common setting.
Following \citet{Carlini:Wagner2017}, we evaluate both our methods and PFD~\cite{wang2025detecting}, a recent perturbation-based per-example detector, under an adaptive C\&W attack, in which the attacker has access to the detector and optimizes each adversarial input to simultaneously mislead the classifier, evade the detector, and maintain a small $\ell_2$ distortion.
Specifically, the attack minimizes the $\ell_2$ distortion plus the C\&W classifier loss and a detector-evasion loss constructed for either PFD or our detectors, over an unconstrained reparameterization of the input through the hyperbolic tangent, which is solved with Adam.
For a fair comparison, we keep all other attack settings fixed and change only the detector-specific loss: we optimize each adversarial example for $200$ steps, with the standard C\&W trade-off coefficient set to $0.01$, the confidence parameter set to $0$, and the coefficient of the detector-evasion loss set to $1.0$.
The experiments are conducted on CIFAR-10 using ResNet-18 as the target model. For PCD and USAD, we use query batches of size $|Y|=50$. Their detection rate is the proportion of adversarial query batches for which the permutation test rejects the null hypothesis at $\alpha=0.05$, whereas PFD reports the proportion of individual adversarial examples detected as adversarial.
Table~\ref{tab:pfd} shows that, under the same adaptive attack, PFD does not detect the generated AEs, whereas PCD and USAD retain detection rates of $41.5\%$ and $44.8\%$, respectively.
This provides a direct comparison between our detectors and a strong perturbation-based per-example detector under an adaptive attack, and demonstrates that USAD is substantially more robust than PFD against the adaptive C\&W attack.

\begin{table}[ht]
\centering
\caption{Detection rates of PFD, PCD, and USAD under an adaptive C\&W attack~\cite{Carlini:Wagner2017} that optimizes each adversarial example for $200$ steps, changing only the detector-specific loss. The experiments are conducted on CIFAR-10 using ResNet-18, with $|Y|=50$ for the batch-level detectors PCD and USAD.}
\label{tab:pfd}
\begin{tabular}{l|c|ccc}
\toprule
Setting & Mean $\ell_2$ & PFD detection rate & PCD detection rate & USAD detection rate \\
\midrule
Adaptive C\&W & 0.072 & $0.0\%$ & $41.5\%$ & $\mathbf{44.8\%}$ \\
\bottomrule
\end{tabular}
\end{table}

\subsection{Sensitivity to the Kernel Bandwidth and the Number of Perturbations}
\label{app:sensitivity}

\textbf{Kernel Bandwidth.}
The bandwidth parameters of our kernels are set using the median heuristic and fixed values (see Appendix~\ref{appendix:exp}).
To examine the sensitivity to this choice, we vary the bandwidth from $0.25\times$ to $4\times$ the median-heuristic value.
The experiments are conducted on CIFAR-10 using ResNet-18 as the target model under a five-step PGD $\ell_\infty$ attack with $\epsilon=2/255$, $|X|=100$, and $|Y|=50$.
Table~\ref{tab:bandwidth} shows that USAD, VD, and PCD remain stable across this $16\times$ bandwidth range, with maximum power variations of only $0.008$, $0.002$, and $0.002$, respectively, compared with $0.038$ for SAMMD.
This shows that our methods are not sensitive to the bandwidth choice.

\begin{table}[ht]
\centering
\caption{Test power when the kernel bandwidth is varied from $0.25\times$ to $4\times$ the median-heuristic value. The last column (spread) is the difference between the maximum and minimum power across the five bandwidths. The experiments are conducted on CIFAR-10 using ResNet-18 under a five-step PGD $\ell_\infty$ attack with $\epsilon=2/255$, $|X|=100$, and $|Y|=50$.}
\label{tab:bandwidth}
\begin{tabular}{l|ccccc|c}
\toprule
Method & $0.25\times$ & $0.5\times$ & $1\times$ & $2\times$ & $4\times$ & Spread \\
\midrule
USAD (VD+PCD) & 0.766 & 0.770 & 0.764 & 0.766 & 0.762 & 0.008 \\
VD            & 0.776 & 0.778 & 0.776 & 0.776 & 0.776 & 0.002 \\
PCD           & 0.172 & 0.174 & 0.174 & 0.174 & 0.174 & 0.002 \\
SAMMD         & 0.318 & 0.290 & 0.280 & 0.280 & 0.280 & 0.038 \\
\bottomrule
\end{tabular}
\end{table}

\textbf{Number of Perturbations.}
Since the inference cost of PCD (and hence USAD) grows with the number of perturbations $K$ (see the time complexity analysis in Appendix~\ref{appendix:exp}), we study whether $K$ can be reduced while maintaining good detection.
We vary $K$ on CIFAR-10 with ResNet-18 under PGD with $\epsilon = 4/255$, using $|X| = 100$ and $|Y| = 10$.
We report the test power and Type-I error of PCD, together with the end-to-end runtime.
As shown in Table~\ref{tab:num_perturbations}, reducing $K$ from $200$ to $100$ decreases the end-to-end runtime from $0.643$s to $0.374$s while maintaining comparable power and valid Type-I error, whereas further reducing $K$ to $50$ causes a clear power drop.
Thus, $K = 100$ provides a good performance--runtime trade-off; all other results in this paper use $K = 200$ (Appendix~\ref{appendix:exp}).

\begin{table}[ht]
\centering
\caption{Test power, Type-I error, and end-to-end runtime for different numbers of perturbations $K$ in PCD, on CIFAR-10 with ResNet-18 under PGD with $\epsilon=4/255$, $|X|=100$, and $|Y|=10$.}
\label{tab:num_perturbations}
\begin{tabular}{c|ccc}
\toprule
$K$ & Power & Type-I error & End-to-end runtime \\
\midrule
50  & 0.348 & 0.061 & $0.239$s \\
$\mathbf{100}$ & $\mathbf{0.500}$ & $\mathbf{0.050}$ & $\mathbf{0.374}$s \\
200 & 0.498 & 0.054 & $0.643$s \\
\bottomrule
\end{tabular}
\end{table}
\newpage
\section*{NeurIPS Paper Checklist}

\begin{enumerate}

\item {\bf Claims}
    \item[] Question: Do the main claims made in the abstract and introduction accurately reflect the paper's contributions and scope?
    \item[] Answer: \answerYes{} 
    \item[] Justification: The abstract and introduction state the main contributions of USAD, including the proposed uncertainty-aware statistics, their aggregation, theoretical Type-I error control, and empirical evaluation against standard and adaptive attacks.

\item {\bf Limitations}
    \item[] Question: Does the paper discuss the limitations of the work performed by the authors?
    \item[] Answer: \answerYes{} 
    \item[] Justification: The paper discusses limitations regarding the scope of SAD, the reliance on batch-level statistical testing, and the dependence of detection power on the query size, attack type, and feature representation; Appendix~\ref{app:discussion} further discusses the limitations of USAD, including its response to benign distribution shifts, low fractions of AEs in the query batch, and the alternatives targeted by VD and PCD.

\item {\bf Theory assumptions and proofs}
    \item[] Question: For each theoretical result, does the paper provide the full set of assumptions and a complete (and correct) proof?
    \item[] Answer: \answerYes{} 
    \item[] Justification: The theoretical results explicitly state the required assumptions, and complete proofs are provided in the appendix. The main text also explains how these results justify the validity of the proposed statistical tests. The informal analysis of the effect of the AE fraction (Remark~\ref{rem:ae_fraction} in Appendix~\ref{app:low_fraction}) is accompanied by a proof sketch.

    \item {\bf Experimental result reproducibility}
    \item[] Question: Does the paper fully disclose all the information needed to reproduce the main experimental results of the paper to the extent that it affects the main claims and/or conclusions of the paper (regardless of whether the code and data are provided or not)?
    \item[] Answer:  \answerYes{} 
    \item[] Justification: The paper reports the datasets, attacks, baselines, evaluation metrics, significance level, query sizes, and implementation details needed to reproduce the main experimental findings.

\item {\bf Open access to data and code}
    \item[] Question: Does the paper provide open access to the data and code, with sufficient instructions to faithfully reproduce the main experimental results, as described in supplemental material?
    \item[] Answer: \answerYes{} 
    \item[] Justification: The paper uses publicly available datasets and releases the code with instructions for reproducing the experiments (see the link in the abstract).

\item {\bf Experimental setting/details}
    \item[] Question: Does the paper specify all the training and test details (e.g., data splits, hyperparameters, how they were chosen, type of optimizer) necessary to understand the results?
    \item[] Answer: \answerYes{} 
    \item[] Justification: The experimental section specifies the datasets, target classifiers, attack configurations, baselines, query sizes, testing thresholds, and hyperparameters used for VD, PCD, and USAD.

\item {\bf Experiment statistical significance}
    \item[] Question: Does the paper report error bars suitably and correctly defined or other appropriate information about the statistical significance of the experiments?
    \item[] Answer: \answerYes{} 
    \item[] Justification: The paper reports detection power and Type-I error across repeated trials, and the statistical testing procedure controls the false alarm rate at the chosen significance level.

\item {\bf Experiments compute resources}
    \item[] Question: For each experiment, does the paper provide sufficient information on the computer resources (type of compute workers, memory, time of execution) needed to reproduce the experiments?
    \item[] Answer: \answerYes{} 
    \item[] Justification: The experiments of the work are conducted on two platforms. One platform is an NVIDIA RTX 4090 GPU PC with PyTorch framework. Another platform is a High-performance Computer cluster with several NVIDIA A100 GPUs with Pytorch framework. The memory of two platforms are both 64 GB. The storage of disk of two platforms are both over 4 TB.
    
\item {\bf Code of ethics}
    \item[] Question: Does the research conducted in the paper conform, in every respect, with the NeurIPS Code of Ethics \url{https://neurips.cc/public/EthicsGuidelines}?
    \item[] Answer: \answerYes{} 
    \item[] Justification: The research uses standard public datasets and adversarial robustness benchmarks, does not involve human subjects, and is intended to improve the reliability and security of machine learning systems.

\item {\bf Broader impacts}
    \item[] Question: Does the paper discuss both potential positive societal impacts and negative societal impacts of the work performed?
    \item[] Answer: \answerNA{} 
    \item[] Justification: The purpose of this paper is to improve adversarial detection and model reliability, without any negative societal impacts.
    
\item {\bf Safeguards}
    \item[] Question: Does the paper describe safeguards that have been put in place for responsible release of data or models that have a high risk for misuse (e.g., pre-trained language models, image generators, or scraped datasets)?
    \item[] Answer: \answerNA{} 
    \item[] Justification: The paper does not release high-risk models, scraped datasets, or generative systems with substantial misuse risk.

\item {\bf Licenses for existing assets}
    \item[] Question: Are the creators or original owners of assets (e.g., code, data, models), used in the paper, properly credited and are the license and terms of use explicitly mentioned and properly respected?
    \item[] Answer: \answerYes{} 
    \item[] Justification: The paper uses standard public datasets, models, and attack methods, and cites the corresponding original works. The use of these assets follows their intended research usage.

\item {\bf New assets}
    \item[] Question: Are new assets introduced in the paper well documented and is the documentation provided alongside the assets?
    \item[] Answer: \answerNo{} 
    \item[] Justification: The paper cites the original datasets, models, and attack methods, but does not explicitly list all licenses and terms of use. We will add these details in the appendix.

\item {\bf Crowdsourcing and research with human subjects}
    \item[] Question: For crowdsourcing experiments and research with human subjects, does the paper include the full text of instructions given to participants and screenshots, if applicable, as well as details about compensation (if any)? 
    \item[] Answer: \answerNA{} 
    \item[] Justification: The paper does not involve crowdsourcing or research with human subjects.

\item {\bf Institutional review board (IRB) approvals or equivalent for research with human subjects}
    \item[] Question: Does the paper describe potential risks incurred by study participants, whether such risks were disclosed to the subjects, and whether Institutional Review Board (IRB) approvals (or an equivalent approval/review based on the requirements of your country or institution) were obtained?
    \item[] Answer: \answerNA{}
    \item[] Justification: The paper does not involve human subjects, so IRB approval is not applicable.

\item {\bf Declaration of LLM usage}
    \item[] Question: Does the paper describe the usage of LLMs if it is an important, original, or non-standard component of the core methods in this research? Note that if the LLM is used only for writing, editing, or formatting purposes and does \emph{not} impact the core methodology, scientific rigor, or originality of the research, declaration is not required.
    \item[] Answer: \answerNA{} 
    \item[] Justification: The core methodology does not use LLMs as an important, original, or non-standard component. Any use of LLMs for writing or editing assistance does not affect the scientific content of the paper.

\end{enumerate}

\end{document}